\documentclass[twoside,11pt]{article}

\usepackage[abbrvbib, preprint]{jmlr2e}
\usepackage{algorithm}
\usepackage{algorithmic}
\usepackage{amsmath} 
\usepackage{subfig}
\usepackage{hyperref}
\usepackage[nameinlink,capitalize,noabbrev]{cleveref}
\usepackage{bm}
\usepackage{amssymb}
\usepackage{mathtools}
\usepackage{enumitem}
\usepackage{xspace}
\usepackage{xcolor}
\hypersetup{
    colorlinks,
    linkcolor=black,
    urlcolor=pink,
    citecolor=blue}
\usepackage{booktabs}
\usepackage{multirow}
\usepackage{graphicx}
\usepackage{tablefootnote}
\usepackage{tikz}
\usepackage{float}
\usepackage{caption}

\newcommand{\qed}{\ensuremath{\square}}

\usepackage{circuitikz} 
\usetikzlibrary{shapes.geometric, positioning, calc, arrows, automata}

\usepackage{amsmath}

\DeclareMathOperator*{\regret}{Regret}

\DeclareMathOperator*{\var}{Var}
\DeclareMathOperator*{\argmax}{argmax}

\newcommand{\cS}{\mathcal{S}}
\newcommand{\cA}{\mathcal{A}}

\newcommand{\cP}{\mathcal{P}}
\newcommand{\cF}{\mathcal{F}}

\newcommand{\cH}{\mathcal{H}}
\newcommand{\cE}{\mathcal{E}}
\newcommand{\cO}{\mathcal{O}}
\newcommand{\tsh}{(s,h)}
\newcommand{\tshn}{(s',h+1)}

\newcommand{\bbE}{\mathbb{E}}

\newtheorem{assumption}{Assumption}

\usepackage{lastpage}
\jmlrheading{23}{2022}{1-\pageref{LastPage}}{1/21; Revised 5/22}{9/22}{21-0000}{Author One and Author Two}

\firstpageno{1}

\begin{document}

\title{Improved Regret Bound for Safe Reinforcement Learning via Tighter Cost Pessimism and Reward Optimism}

\author{\name Kihyun Yu$^{1}$  \email khyu99@kaist.ac.kr
       \AND
        \name Duksang Lee$^{1}$ \email duksang@kaist.ac.kr 
        \AND
        \name William Overman$^{2}$ \email wpo@stanford.edu
       \AND
       \name Dabeen Lee$^{1,\dagger}$ \email dabeenl@kaist.ac.kr
       \AND
       \addr 
$^1$Department of Industrial and Systems Engineering, KAIST, Daejeon 34141, South Korea\\
$^2$Graduate School of Business, Stanford University, Stanford, CA 94305, United States\\
$^\dagger$ Corresponding author\\  
       }

\maketitle

\begin{abstract}
    This paper studies the safe reinforcement learning problem formulated as an episodic finite-horizon tabular constrained Markov decision process with an unknown transition kernel and stochastic reward and cost functions. We propose a model-based algorithm based on novel cost and reward function estimators that provide tighter cost pessimism and reward optimism. While guaranteeing no constraint violation in every episode, our algorithm achieves a regret upper bound of $\widetilde{\mathcal{O}}((\bar C - \bar C_b)^{-1}H^{2.5} S\sqrt{AK})$ where $\bar C$ is the cost budget for an episode, $\bar C_b$ is the expected cost under a safe baseline policy over an episode, $H$ is the horizon, and $S$, $A$ and $K$ are the number of states, actions, and episodes, respectively. This improves upon the best-known regret upper bound, and when $\bar C- \bar C_b=\Omega(H)$, it nearly matches the regret lower bound of $\Omega(H^{1.5}\sqrt{SAK})$. We deduce our cost and reward function estimators via a Bellman-type law of total variance to obtain tight bounds on the expected sum of the variances of value function estimates. This leads to a tighter dependence on the horizon in the function estimators. We also present numerical results to demonstrate the computational effectiveness of our proposed framework.
\end{abstract}

\section{Introduction}

Safe reinforcement learning (RL) aims to learn a policy that maximizes the cumulative reward and, at the same time, ensures that some safety requirements are satisfied during the learning process. Safe RL provides modeling frameworks for many practical scenarios where violating a safety constraint results in a critical situation. For example, it is crucial to enforce collision avoidance for autonomous driving~\citep{vehicle:isele2018safe, vehicle:krasowski2020safe} and robotics~\citep{robot:fisac2018general, robot:garcia2020teaching}. For financial planning, there exist legal and business regulations~\citep{abe-finance}. For healthcare systems, service providers consider restrictions due to patients' conditions~\citep{CORONATO2020101964}.

The standard approach is to formulate a safe RL problem as a constrained Markov decision process (CMDP), where the objective is to maximize the expected reward over a time horizon while there is a constraint that the expected cost should be under budget~\citep{Altman}. The presence of constraints, however, brings about challenges in developing solution methods for CMDPs. The Bellman optimality principle does not hold for CMDPs, and as a consequence, backward induction and the greedy operator cannot be directly applied to CMDPs~\citep{Altman}. This makes online learning of CMDPs difficult, and we need significantly different frameworks and algorithms compared to the unconstrained setting~\citep{efroni2020}.

The first direction for online reinforcement learning of CMDPs is to consider \emph{cumulative (or soft) constraint violation}, which sums up the constraint violations across episodes~\citep{efroni2020}. Here, the constraint violation in an episode is defined as the expected cost minus the budget. Then a policy can have a negative constraint violation, which means that a positive violation in one episode can be canceled out by a negative violation in another episode in the sum. This cancellation effect allows oscillating between such two cases, while still achieving zero cumulative constraint violation. This phenomenon can indeed be observed in practice~\citep{pmlr-v119-stooke20a,pmlr-v202-moskovitz23a}.

The second direction attempts to remedy the issue of error cancellation with the notion of \emph{hard constraint violation}~\citep{efroni2020}. It ignores episodes with a negative violation and takes the sum of only the positive constraint violations. \cite{efroni2020} developed OptCMDP and its efficient variant, OptCMDP-bonus, that attain a regret upper bound and a hard constraint violation of $\widetilde{\cO}(H^2\sqrt{S^2AK})$. %
Recently, \cite{ghosh2024towards} proposed a model-free algorithm with the same asymptotic guarantees.
However, as in the first setting, the algorithms cannot avoid episodes in which the constraint is violated. Thus, they are still not suitable for the aforementioned applications, where even a single incidence of violation can cause substantial problems.

The third approach seeks \emph{zero (hard) constraint violation}, requiring that the constraint is satisfied in every episode~\citep{simao2021always}. Satisfying constraints in the early stage is difficult when the model parameters, especially the transition kernel, are unknown. \cite{simao2021always} considered some abstraction of the transition model under which they showed an algorithm with no constraint violation, but no regret upper bound was presented. Then \cite{liu2021learning} came up with the first algorithm, OptPess-LP, that achieves a sublinear regret with no constraint violation, assuming the knowledge of a \emph{safe baseline policy}. Here, a safe baseline policy is a policy under which the expected cost is lower than the budget. OptPess-LP guarantees a regret upper bound of $\widetilde{\cO}((\bar C - \bar C_b)^{-1}H^3  \sqrt{S^{3}AK})$ where $\bar C$ is the budget, $\bar C_b$ is the expected cost under the safe baseline policy, $H$ is the length of the horizon, and $S$, $A$ and $K$ are the number of states, actions, and episodes, respectively. \cite{bura2022dope} developed Doubly Optimistic Pessimistic Exploration (DOPE) with an improved regret upper bound of $\widetilde{\cO}((\bar C - \bar C_b)^{-1}H^3  \sqrt{S^2AK})$. DOPE is based on designing tight optimistic reward function estimators (reward optimism) and conservative cost function estimators (cost pessimism).

While DOPE establishes a tight regret upper bound with no constraint violation, there is still room for improvement. The regret lower bound of $\Omega(H^{1.5}\sqrt{SAK})$ for the unconstrained case \citep{NEURIPS2018_d3b1fb02,pmlr-v132-domingues21a} also works as a lower bound for the constrained setting because we may take trivial cost functions. However, even when $\bar C - \bar C_b=\Omega(H)$, the regret upper bound of DOPE is as low as $\widetilde{\cO}(H^2  \sqrt{S^2AK})$ which has a gap of $\widetilde{\cO}(\sqrt{HS})$ from the lower bound. This naturally motivates the following question. 
\begin{center}
\emph{Is there an algorithm for learning CMDPs that guarantees no constraint violation during learning and achieves an improved regret upper bound?}
\end{center}

\paragraph{Our Contributions}
We answer this question affirmatively with an algorithm that improves upon DOPE via tighter reward optimism and cost pessimism. Our results are summarized in \Cref{table:hard} and as follows.

\begin{itemize}
    \item Our algorithm, DOPE+, achieves a regret upper bound of $\widetilde{\cO}((\bar C - \bar C_b)^{-1}H^{2.5} \sqrt{S^2AK})$ and ensures no constraint violation in every episode, with the knowledge of a safe baseline policy. This improves upon the best-known regret upper bound $\widetilde{\cO}((\bar C - \bar C_b)^{-1}H^{3}\sqrt{S^2AK})$ attained by DOPE. 

    \item When the gap $\bar C-\bar C_b$ between the budget and the expected cost under the safe baseline policy satisfies $\bar C-\bar C_b=\Omega(H)$, the regret upper bound becomes $\widetilde{\cO}(H^{1.5}\sqrt{S^2AK})$. This nearly matches the regret lower bound $\Omega(H^{1.5}  \sqrt{SAK})$, which shows that the regret upper bound achieves the optimal dependence on the horizon $H$.

    \item The improvement comes from our novel reward and cost function estimators with tighter reward optimism and cost pessimism. We deduce the function estimators by providing a tighter upper bound on the difference of value functions with respect to the true and estimated transition kernels. We analyze the difference based on the \emph{value difference lemma} due to \cite{dann2017pac}. The key step is to apply a Bellman-type law of total variance to control the expected sum of the variance of value function estimates, inspired by~\cite{azar2017, ssp-adversarial-unknown}.
\end{itemize}

\begin{table*}[ht]
\caption{Comparison of Safe RL algorithms for the Hard Constraint Violation Setting: OptCMDP, OptCMDP-bonus~\citep{efroni2020}, AlwaysSafe~\citep{simao2021always}, OptPess-LP~\citep{liu2021learning}, DOPE~\citep{bura2022dope}, and DOPE+ (\Cref{alg:hard}).}
\label{table:hard}
\begin{center}
\begin{tabular}{lll}
\toprule
{\small Algorithms}  &\small{Regret} &\small{Hard Constraint Violation}\\ 
\midrule
\small{OptCMDP, OptCMDP-bonus} %
& \small{$\widetilde{\cO}(H^2\sqrt{S^2 AK})$} & \small{$\widetilde{\cO}(H^2\sqrt{S^2 AK})$}\\
\small{AlwaysSafe} %
& \small{Unknown} & \small{0}\\
\small{OptPess-LP} & \small{$\widetilde{\cO}((\bar{C}-\bar{C}_b)^{-1} H^{3}\sqrt{S^3 A K})$} & \small{0}\\ 
\small{DOPE} & \small{$\widetilde{\cO}((\bar{C}-\bar{C}_b)^{-1} H^{3}\sqrt{S^2 A K})$} & \small{0}\\ 
\midrule
\small{\bf DOPE+} & \small{$\widetilde{\cO}((\bar{C}-\bar{C}_b)^{-1} H^{2.5}\sqrt{S^2 A K})$} & \small{0} \\
\bottomrule 
\end{tabular}
\end{center}
\end{table*}

DOPE by \cite{bura2022dope} provides a framework for CMDPs to consider the error term in inferring the value function of a policy caused by the difference between the true transition kernel and an estimated one. They also applied the value difference lemma to analyze the difference between the value function under the true transition kernel and that under an estimated transition kernel. Then they take a na\"ive upper bound of $H$ on any value function. 

Instead, we take one step further to refine the analysis by considering the variance terms of value functions. We decompose the difference term and apply a Bellman-type law of total variance~\citep{azar2017, ssp-adversarial-unknown} to bound the expected sum of the variances of value function estimates. We explain this and how it leads to an improved regret bound in detail in \Cref{sec:estimator}.

A more comprehensive literature review on online reinforcement learning of CMDPs is given in the appendix.

\section{Problem Setting}\label{sec:problem-setting}

A finite-horizon tabular  MDP is defined by a tuple $(\cS,\cA, H, \left\{P_h\right\}_{h=1}^{H-1},p)$ where $\cS$ is the finite state space with $|\cS|=S$, $\cA$ is the finite action space with $|\cA|=A$, $H$ is the finite-horizon, $P_h:\cS\times \cA\times \cS\to [0,1]$ is the transition kernel at step $h\in[H-1]$, and $p$ is the known initial distribution of the states. Here, $P_h(s'\mid s,a)$ is the probability of transitioning to state $s'$ from state $s$ when the chosen action is $a$ at step $h\in[H-1]$. Equivalently, we may define a single \emph{non-stationary} transition kernel $P:\cS\times \cA\times\cS\times [H]\to[0,1]$ with $P(s'\mid s,a,h)=P_h(s'\mid s,a)$  and $P(s'\mid s,a,H)=p(s')$ for $(s,a,s',h)\in\cS\times \cA\times \cS\times[H-1]$. We assume that $\{P_h\}_{h=1}^{H-1}$ and thus $P$ are \emph{unknown}.

Before an episode begins, the agent prepares a \emph{stochastic policy} $\pi:\cS\times [H]\times \cA\to[0,1]$ where $\pi(a\mid s,h)$ is the probability of taking action $a\in \cA$ in state $s\in \cS$ at step $h$. Here, $\pi$ can be viewed as a \emph{non-stationary policy} as it may change over the horizon, and this is due to the non-stationarity of $P$ over steps $h\in[H]$. Given a policy $\pi_k$ for episode $k\in[K]$, the MDP proceeds with trajectory $\{s_h^{P,\pi_k}, a_h^{P,\pi_k}\}_{h\in[H]}$ generated by $P$.

The reward and cost functions are given by $f,g:\cS\times\cA\times[H]\rightarrow [0,1]$, i.e., choosing action $a\in \cA$ at state $s \in \cS$ and step $h\in[H]$ generates a reward $f(s,a,h)$ and cost $g(s,a,h)$. Here, functions $f$ and $g$ are non-stationary over $h \in [H]$. However, the agent observes the noisy reward and cost. We denote the observed noisy reward and cost for episode $k\in [K]$ by $f_k(s,a,h)$ and $g_k(s,a,h)$, respectively. As in \cite{liu2021learning}, we assume that $f_k(s,a,h)$ and $g_k(s,a,h)$ are determined by independent\footnote{We may impose conditional independence.} noisy random variables $\zeta_k^f(s,a,h)$ and $ \zeta_k^g(s,a,h)$ following a zero-mean $1/2$-sub-Gaussian distribution, i.e., $f_k(s,a,h) = f(s,a,h) + \zeta_k^f(s,a,h)$ and $g_k(s,a,h) = g(s,a,h) + \zeta_k^g(s,a,h)$. We note that $1/2-$sub-Gaussian random variables $\zeta$ with zero mean satisfies $\mathbb{E}[\zeta]=0$ and $\mathbb{E}[\exp(\lambda \zeta)] \leq \exp(\lambda^2 / 4)$. Then Hoeffding's inequality implies the following.
\begin{lemma}\label{lem:bounded}
With probability at least $1-4\delta$, it holds that 
for any $(s,a,h)\in \cS\times\cA\times[H]$ and $k\in[K]$,
$$\left|f_k(s,a,h)\right|, \left|g_k(s,a,h)\right| \leq 1 + \sqrt{\ln(HSAK/\delta)}.$$ 
\end{lemma}

We define the value function $V^{\pi}_h(s; \ell,P)$ at state $s \in \cS$ and step $h\in[H]$ for a given policy $\pi$, function $\ell$, and transition kernel $P$ as \begin{align*}
V^{\pi}_h(s;\ell,P)
&=\bbE\left[\sum_{j=h}^H \ell(s_j^{P,\pi}, a_j^{P,\pi},j) \mid \ell, \pi, P, s_h^{P,\pi}=s\right].
\end{align*}
Moreover, let $V_1^{\pi}(\ell,P) = \bbE_{s \sim p}\left[ V_1^{\pi}(s;\ell,P) \mid \ell, \pi, P\right]$ where $p$ is the known distribution of the initial state.

The goal of the constrained Markov decision process is to learn an optimal policy $\pi^*$ defined as
\begin{equation*}
\begin{aligned}
\pi^*  \in \argmax_{\pi}&\quad  V^{\pi}_1(f,P)\quad \text{s.t.}\quad V^{\pi}_1(g,P)\leq \bar{C}
\end{aligned}
\end{equation*}
where $\bar C$ is the budget on the expected cost over the horizon. As the model parameters $f,g,P$ are unknown, we develop a learning algorithm that computes policies over multiple episodes. For $K$ episodes, we deduce policies $\pi_1, \ldots, \pi_K$ with the safety requirement that
$$V_1^{\pi_k}(g,P)\leq \bar C \quad \forall k\in[K]$$
holds with high probability. The safety requirement is equivalent to enforcing zero hard constraint violation where the {hard constraint violation} is defined as
$$\mathrm{Violation}(\vec\pi):=\sum_{k=1}^K\max\left\{0, V_1^{\pi_k}(g,P) - \bar C\right\}$$
and $\vec\pi=(\pi_1, \ldots, \pi_K)$ is a shorthand notation for the $K$ policies. As a performance metric for a learning algorithm, we use the following notion of regret.
\begin{align*}
\regret\left(\vec\pi\right):=%
\sum_{k=1}^K \left(V_1^{\pi^*}(f,P) - V_1^{\pi_k}(f,P)\right).
\end{align*}
To satisfy the safety requirement, we assume that a \emph{strictly safe baseline policy} $\pi_b$ is given to the agent.
\begin{assumption}\label{safe-policy}
The agent knows a policy $\pi_b$ and its expected cost $\bar C_b=V_1^{\pi_b}(g,P)$. We further assume that $\pi_b$ is strictly feasible, i.e., $\bar C_b<\bar C$.
\end{assumption}
This assumption is necessary because the learning agent has no information about the underlying MDP at the beginning. Without a safe baseline policy, it is difficult to satisfy the constraint in the initial phase of learning. It is a commonly assumed condition for learning CMDPs~\citep{simao2021always,liu2021learning,bura2022dope}. We also remark that strict feasibility of $\pi_b$ is related to Slater's condition in constrained optimization.

Lastly, we assume that the budget $\bar C$ satisfies $\bar{C}\in(0,H)$. If $\bar C \geq H$, then as $V_1^{\pi}(g,P) \leq H$ for any policy $\pi$, the safety requirement is trivially satisfied. Moreover, we have $\bar C$ is strictly positive because \Cref{safe-policy} imposes that $\bar C>\bar C_b$ and $\bar C_b= V_1^{\pi_b}(g,P)\geq 0$.

\section{Model Estimators}\label{sec:estimator}

In this section, we present our model estimators. In \Cref{sec:estimator:confidence}, we define confidence sets for the transition kernel and confidence intervals for reward and cost functions. In \Cref{sec:estimator:function}, we deduce our tighter optimistic reward and pessimistic cost function estimators. Lastly, in \Cref{sec:proof-thm}, we sketch our technical proof for obtaining the tighter function estimators.

\subsection{Confidence Sets and Intervals}\label{sec:estimator:confidence}

We follow the standard Bernstein inequality-based confidence set construction for estimating the true transition kernel and use confidence intervals based on Hoeffding's inequality for estimating reward and cost functions~\citep{Jin2020,cohen2020}.

As in \cite{efroni2020,bura2022dope}, we maintain counters to keep track of the number of visits to each tuple $(s,a,h)$ and tuple $(s,a,s',h)$. For each $k\in[K]$, we define $N_k(s,a,h)$ and $M_k(s,a,s',h)$ as the number of visits to tuple $(s,a,h)$ and the number of visits to tuple $(s,a,s',h)$ up to the first $k-1$ episodes, respectively, for $(s,a,s',h)\in\cS\times\cA\times\cS\times[H]$. Given $N_k(s,a,h)$ and $M_k(s,a,s',h)$, we define the empirical transition kernel $\bar{P}_k$ for episode $k$ as
\begin{equation*} %
\bar{P}_k(s'\mid s,a,h)=\frac{M_{k}(s,a,s',h)}{\max\{1,N_k(s,a,h)\}}.
\end{equation*}
Next, for some confidence parameter $\delta\in(0,1)$, we define the confidence radius $\epsilon_k(s'\mid s,a,h)$  for $(s,a,s',h)\in \cS\times \cA\times\cS\times [H]$ and $k\in[K]$ as 
\begin{align*}%
\begin{aligned}
\epsilon_k(s'\mid s,a,h)&=2\sqrt{\frac{\bar{P}_k(s'\mid s,a,h)\ln\left({{HSAK}/{\delta}}\right)}{\max\{1,N_k(s,a,h)-1\}}}+\frac{14\ln\left({{HSAK}/{\delta}}\right)}{3\max\{1,N_k(s,a,h)-1\}}.
\end{aligned}
\end{align*}
Based on the empirical transition kernel and the radius, we define the confidence set $\cP_k$ for episode $k$ as 
\begin{equation}\label{confidence-set}
\cP_k= \left\{\widehat P:\ 
\begin{aligned}
     &\left|\widehat P(s'\mid s,a,h)-\bar{P}_k(s'\mid s,a,h)\right| \leq \epsilon_{k}(s'\mid s,a,h)\ \ \forall (s,a,s',h)
\end{aligned}
\right\}.
\end{equation}
Then by the empirical Bernstein inequality due to~\cite{Maurer-bernstein}, we can show the following.
\begin{lemma}\label{lemma:confidence}
With probability at least $1-4\delta$, the true transition kernel $P$ is contained in the confidence set $\cP_k$ for every episode $k\in[K]$.
\end{lemma}

Next, for reward and cost functions, we define the confidence radius $R_k(s,a,h)$ for $(s,a,h)\in\cS\times\cA\times[H]$, $k\in[K]$ and $\delta\in(0,1)$ as 
$$R_k(s,a,h)=\sqrt{\frac{\ln(HSAK/\delta)}{\max\{1,N_k(s,a,h)\}}}.$$

We define empirical estimators $\bar f_k$ and $\bar g_k$ as
\begin{align*}
\bar f_k(s,a,h) = \frac{\sum_{j=1}^{k-1}f_j(s,a,h) n_j(s,a,h)}{\max\{1, N_k(s,a,h)\}},\quad
\bar g_k(s,a,h) = \frac{\sum_{j=1}^{k-1}g_j(s,a,h) n_j(s,a,h)}{\max\{1,N_k(s,a,h)\}}
\end{align*}
where $f_j(s,a,h), \ g_j(s,a,h)$ are the instantaneous reward and cost for episode $j \in [k-1]$ and $n_j(s,a,h)$ is the indicator variable that returns 1 if the agent visited $(s,a,h)$ in episode $j$ and 0 otherwise. Then we may deduce the following from  Hoeffding's inequality.
\begin{lemma}\label{lemma:estimator}
With probability at least $1-4\delta$, it holds that for any $(s,a,h)\in\cS\times\cA\times[H]$ and $k\in[K]$,
	\begin{align*}
		\left|\bar{f}_k(s,a,h) - f(s,a,h) \right| \leq R_k(s,a,h),\quad 
		\left|\bar g_k(s,a,h) - g(s,a,h)\right| \leq R_k(s,a,h).
    \end{align*}
\end{lemma}

\subsection{Tighter Function Estimators}\label{sec:estimator:function} 

\Cref{lemma:confidence,lemma:estimator} motivate the following attempt to deduce feasible policies. For episode $k\in[K]$, we take a transition kernel $P_k$ from the confidence set $\cP_k$ and $\bar g_k + R_k$ as a pessimistic (or conservative) estimator of the cost function $g$. Then we may compute a policy $\pi_k$ that satisfies 
$V_1^{\pi_k}(\bar g_k + R_k, P_k)\leq \bar C$, which is an approximation of the constraint. However, even if $\bar g_k + R_k$ provides an upper bound on $g$, the issue is that $V_1^{\pi_k}(g, P)\not\leq V_1^{\pi_k}(\bar g_k + R_k, P_k)$. This is because the difference between the true transition kernel $P$ and $P_k$ can make $V_1^{\pi_k}(g, P)$ greater than $V_1^{\pi_k}(\bar g_k + R_k, P_k)$. That said, $\pi_k$ does not necessarily satisfy the constraint, although it satisfies the approximate constraint. 

Inspired by the challenge, the next question is as to whether we can design an approximate constraint, satisfying which guarantees that the true constraint is also satisfied. \cite{liu2021learning,bura2022dope} considered this, and their idea was to add an extra pessimism to cost function estimators. Basically, we take functions of the form \begin{equation}\label{eq:cost-estimator}\widehat g_k(s,a,h) = \bar g_k(s,a,h)+R_k(s,a,h)+U_k(s,a,h)
\end{equation}
for $(s,a,h)\in\cS\times\cA\times[H]$ and $k\in[K]$ where $U_k$ captures the error in estimating the true transition kernel $P$. In the above-discussed context, $U_k$ considers the difference between $P$ and $P_k$. Here, one needs to set $U_k$ sufficiently large so that $V_1^{\pi_k}(g, P)\leq V_1^{\pi_k}(\widehat g_k, P_k)$, in which case satisfying the corresponding approximate constraint $V_1^{\pi_k}(\widehat g_k, P_k)\leq \bar C$ guarantees satisfaction of the true constraint. 

On the other hand, choosing the right magnitude of $U_k$ is important to control the regret function. When $U_k$ is too large, $\widehat g_k$ is too conservative, and it prevents from getting a high reward. Indeed, \cite{bura2022dope} improved upon \cite{liu2021learning} by making $U_k$ tighter. Our main contribution is to develop an even tighter $U_k$ function than \cite{bura2022dope}.

Before we present our design of $U_k$, let us briefly discuss how to deduce the extra pessimism term $U_k$ in general. As explained before, we want to guarantee $V_1^{\pi_k}(g, P)\leq V_1^{\pi_k}(\widehat g_k, P_k)$ for any $P_k\in \cP_k$. Then note that
$$
V_1^{\pi_k}(g,P) \leq V_1^{\pi_k}(g,P_k) + |V_1^{\pi_k}(g,P) - V_1^{\pi_k}(g,P_k)|.$$
If the statement of \Cref{lemma:estimator} holds, then $V_1^{\pi_k}(g,P_k)$ is bounded above by $V_1^{\pi_k}(\bar g_k+ R_k,P_k)$. Therefore, once we come up with some $U_k$ such that $|V_1^{\pi_k}(g,P) - V_1^{\pi_k}(g,P_k)|\leq V_1^{\pi_k}(U_k,P_k)$, we get
$$
V_1^{\pi_k}(g,P)\leq V_1^{\pi_k}(\bar g_k + R_k + U_k,P_k).$$
In this case, $\widehat g_k = \bar g_k + R_k + U_k$ gives rise to a valid function estimator. 

We devise our pessimism function $U_k$ as follows.
\begin{theorem}\label{theorem:U_k}
Let $\pi_k$ be any policy for episode $k$. Take
\begin{align}\label{Uk-ours}
\begin{aligned}
    U_k(s,a,h)&=8\sqrt{H}\varepsilon_k(s,a,h)+4S\sqrt{HA/K}+ \frac{2\sqrt{{HK}/{A}}\ln(HSAK/\delta) + \eta}{\max\{1, N_k(s,a,h)-1\}}
    \end{aligned}
\end{align}
for $(s,a,h) \in \cS\times\cA\times[H]$ and $k\in[K]$ where
\begin{align}\label{vareps}
\begin{aligned}
    \varepsilon_k(s,a,h) &= 2\sqrt{\frac{S\ln(HSAK/\delta)}{\max\{1,N_k(s,a,h)-1\}}} +\frac{14S\ln(HSAK/\delta)}{3\max\{1,N_k(s,a,h)-1\}}
    \end{aligned}
\end{align}
and $\eta =(19HS+2H^{1.5}S+ 10^4H^2S^2)(\ln(HSAK/\delta))^2.$
Then it holds with probability at least $1-14\delta$ that 
$$|V^{\pi_k}_1(g, P) - V^{\pi_k}_1(g, P_k)|\leq V_1^{\pi_k}(U_k, P_k)$$
for any $P_k\in \cP_k$ and $g:\cS\times\cA\times[H] \rightarrow [0,1]$.
\end{theorem} 
Next, we demonstrate that our $U_k$ indeed improves upon \cite{bura2022dope}.
\begin{remark}
{\em
\cite{bura2022dope} set $U_k$ as \begin{equation}\label{Uk-bura}
U_k(s,a,h)=2H\sum_{s'\in\mathcal{S}}\epsilon_k(s'\mid s,a,h).
\end{equation}
However, there is a minor issue with this choice. We need the property that $U_k$ is nonincreasing in $k$ to show \Cref{lemma:K0} and \citep[Proposition 4,][]{bura2022dope}, but $U_k$ given in~\eqref{Uk-bura} can increase as $M_k(s,a,s',h)/N_k(s,a,h)^2$ can increase. As a fix, we may take $U_{k}(s,a,h)=2H \varepsilon_k(s,a,h)$ where $\varepsilon_k$ is given in~\eqref{vareps}. By the Cauchy-Schwarz inequality, we have $\sum_{s'\in\mathcal{S}}\epsilon_k(s'\mid s,a,h)\leq \varepsilon_k(s,a,h)$. Moreover, $\varepsilon_k(s,a,h)$ is nonincreasing in $k$, as desired. Comparing $2H\varepsilon_k(s,a,h)$ with our construction from \Cref{theorem:U_k}, we have coefficient $8\sqrt{H}$ instead of $2H$ for the sum. Although we have additional terms for $U_k$, the reduction of $\mathcal{O}(\sqrt{H})$ in the coefficient translates to the improvement of $\mathcal{O}(\sqrt{H})$ factor in the regret upper bound.\qed
}
\end{remark}

Now we present our optimistic reward function estimator $\widehat f_k$. On top of $\bar f_k + R_k$, we take additional optimistic terms for the reward function to compensate for the extra pessimism in $\widehat g_k$, which reduces the search space of policies and hinders exploration. We define the optimistic reward function estimator $\widehat f_k$ as
\begin{equation}\label{eq:reward-estimator}
\begin{aligned}
\widehat f_k(s,a,h) 
&= \min\left\{B,\   
\begin{aligned}
     &\bar{f}_k(s,a,h) + \frac{3H}{\bar{C} - \bar{C}_b}R_k(s,a,h) +\frac{H}{\bar{C}-\bar{C}_b} U_k(s,a,h)
\end{aligned}
\right\}
\end{aligned}
\end{equation}
where $B= 1+ \sqrt{\ln(HSAK/\delta)}$. Here, the extra optimism in $\widehat f_k$ is designed to promote exploration.

\subsection{Proof Outline of Theorem \ref{theorem:U_k}}\label{sec:proof-thm}

The value difference lemma~\citep{dann2017pac} implies 
\begin{align*}
    \begin{aligned}
V_1^{\pi_k}(g,P) - V_1^{\pi_k}(g,P_k)
&=\bbE\left[\sum_{h=1}^H \ell(s_h^{P_k,\pi_k},a_h^{P_k,\pi_k},h)\mid \pi_k,P_k\right]
    \end{aligned}
\end{align*}
where $\ell(s,a,h)$ is given by
\begin{equation}\label{eq:ell}\sum\nolimits_{
    s'\in\mathcal{S}} (P-P_k)(s'\mid s,a,h)V_{h+1}^{\pi_k}(s';g,P)\end{equation}
with $V_{H+1}^{\pi_k}=0$ and $(P-P_k)(s'\mid s,a,h)=P(s'\mid s,a,h)-P_k(s'\mid s,a,h)$. Here, \cite{bura2022dope} used that $V_{h+1}^{\pi_k}\leq H$ and $|P-P_k|\leq |P-\bar P_k|+ |\bar P_k -P_k|\leq 2\epsilon_k$ by \Cref{lemma:confidence}. Then it follows that
\begin{align*}
|V_1^{\pi_k}(g,P) - V_1^{\pi_k}(g,P_k)|
&\leq \mathbb{E}\left[\sum_{h=1}^H 2H\sum_{s'\in\mathcal{S}}\epsilon_k(s'\mid s_h^{P_k,\pi_k},a_h^{P_k,\pi_k},h)\mid \epsilon_k,\pi_k, P_k\right]
\end{align*}
whose right-hand side equals $V_1^{\pi_k}(U_k, P_k)$ where $U_k$ is given as in~\eqref{Uk-bura}. This explains how \cite{bura2022dope} deduced their pessimistic cost estimators.

To prove \Cref{theorem:U_k} that establishes the validity of our choice of tighter $U_k$ in~\eqref{Uk-ours}, we need a more refined analysis of the difference term $|V_1^{\pi_k}(g,P) - V_1^{\pi_k}(g,P_k)|$. Note that $\ell(s,a,h)$ in~\eqref{eq:ell} satisfies
\begin{align*}
    \left|\ell(s,a,h)\right|
    &\leq\left|\sum_{s'\in\cS}(P-P_k)(s'\mid s,a,h)V_{h+1}^{\pi_k}(s'; g,P_k)\right| +\left|\sum_{s'\in\cS}(P-P_k)(s'\mid s,a,h)W_{h+1}^{\pi_k}(s'; g)\right|
\end{align*}
where $W_{h+1}^{\pi_k}(s'; g) = V_{h+1}^{\pi_k}(s'; g,P) - V_{h+1}^{\pi_k}(s'; g,P_k).$ We may argue that the second term on the right-hand side is a small value. That said, let us focus on the first term which is the dominant one. Since $P$ and $P_k$ both define transition functions, the first term equals
\begin{align*}
\begin{aligned}
\left|\sum_{s'\in\cS}(P-P_k)(s'\mid s,a,h)(V_{h+1}^{\pi_k}(s'; g,P_k)-\widehat\mu_k(s,a,h))\right|
\end{aligned}
\end{align*}
where $\widehat\mu_k(s,a,h) = \bbE_{s'\sim P_k(\cdot\mid s,a,h)}[V_{h+1}^{\pi_k}(s';g,P_k)]$. Next, we observe that
$$\left|(P-P_k)(s'\mid s,a,h)\right|\leq 2\epsilon_k(s'\mid s,a,h)$$ due to \Cref{lemma:confidence}. Recall that $\epsilon_k(s'\mid s,a,h)$ contains the term $\sqrt{\bar P_k(s'\mid s,a,h)}$. As $P_k\in\cP_k$ we deduce that 
$\sqrt{\bar P_k(s'\mid s,a,h)}\leq \sqrt{P_k(s'\mid s,a,h)+ \epsilon_k(s'\mid s,a,h)}$. As a result, by the Cauchy-Schwarz inequality, the analysis boils down to providing an upper bound on the term
$$\sum_{s'\in\mathcal{S}}P_k(s'\mid s,a,h)(V_{h+1}^{\pi_k}(s'; g,P_k)-\widehat\mu_k(s,a,h))^2,$$
which equals 
$$\mathbb{\widehat V}_k(s,a,h):=\var_{s'\sim P_k(\cdot \mid s,a,h)}[V^{\pi_k}_{h+1}(s';g, P_k)].$$
Furthermore, our proof reveals that 
$V_1^{\pi_k}(\mathbb{\widehat V}_k, P_k)$ is the important quantity to control. Applying a na\"ive upper bound on value functions as in \cite{bura2022dope} gives $\mathbb{\widehat V}_k\leq H^2$ and thus $V_1^{\pi_k}(\mathbb{\widehat V}_k, P_k)\leq H^3$. However, this bound is not tight enough. Instead, we prove the following lemma based on a Bellman-type law of total variance~\citep{azar2017, ssp-adversarial-unknown}.
\begin{lemma} \label{lemma:variance-aware}%
Let $\pi_k$ be a policy for episode $k$. Then
\[
V_1^{\pi_k}(\mathbb{\widehat V}_k, P_k)\leq 2H^2
\]
for any $P_k\in \cP_k$ and $g:\cS\times\cA\times[H] \rightarrow [0,1]$.
\end{lemma}
This improvement in the variance term leads to our tighter estimators.

\section{Algorithm}\label{sec:alg}

DOPE+, given by \Cref{alg:hard}, is a variant of DOPE by~\cite{bura2022dope} with our novel reward and cost function estimators from \Cref{sec:estimator}. Recall that our pessimistic cost estimator $\widehat g_k$ is given by~\eqref{eq:cost-estimator} with the extra pessimism term $U_k$ given in~\eqref{Uk-ours} and our optimistic reward estimator $\widehat f_k$ is given in~\eqref{eq:reward-estimator}.
\begin{algorithm}[h!]
\caption{Doubly Optimistic Pessimistic Exploration with Tighter Function Estimators (DOPE+)}
\label{alg:hard}
\begin{algorithmic}
\STATE {\bfseries Initialize:} episode counter $k=1$, counters 
$N(s,a,h)=0$ and $M(s,a,s',h)=0$
for $(s,a,s',h)\in \cS\times\cA\times \cS\times [H]$, safe baseline policy $\pi_b$ and its expected cost for a single episode $\bar C_b$, and the number $K_0$ of episodes for the initial phase.

\FOR{$k=1,\ldots, K$}
\STATE Set counters $N_k\leftarrow N$ and $M_k\leftarrow M$.
\STATE Compute $\bar{P}_k$, $\epsilon_k$, and $\cP_k$ (\Cref{sec:estimator:confidence}).

\IF{$k \leq K_0$}
\STATE $\pi_k = \pi_b$
\ENDIF
\IF{$k > K_0$}
\STATE Compute estimators $\widehat f_k$ and $\widehat g_k$ (\Cref{sec:estimator:function}).
\STATE Deduce $\pi_k,P_k$ from \eqref{eq:lp}.
\ENDIF

\STATE Sample state $s_1$ from distribution $p$.
\FOR{$h=1,\ldots, H$}
\STATE Sample  $a_h$ from  $\pi_k(\cdot \mid s_h, h)$
\STATE Observe  $f_k(s_h,a_h,h)$ and  $g_k(s_h,a_h,h)$.
\STATE Observe  $s_{h+1}$ determined by $P(\cdot\mid s_h,a_h,h)$.
\STATE Update counters $N(s_h,a_h,h)\leftarrow N(s_h,a_h,h)+1$ and $M(s_h,a_h,s_{h+1},h)\leftarrow M(s_h,a_h,s_{h+1},h)+1$.
\ENDFOR
\ENDFOR
\end{algorithmic}
\end{algorithm}

As in~\cite{efroni2020,bura2022dope}, we compute our policy $\pi_k$ for episode $k\in[K]$ by solving the following optimization problem.
\begin{equation}\label{eq:lp}   \begin{aligned}
(\pi_k,P_k) 
&\in \argmax_{(\pi,Q)\in\Pi\times \cP_k}\left\{V^{\pi}_1(\widehat f_k, Q): V^{\pi}_1(\widehat g_k, Q) \leq \bar{C} \right\}
\end{aligned}
\end{equation}
where $\cP_k$ is the confidence set given by~\eqref{confidence-set} and
$$\Pi=\left\{\pi: \
\begin{aligned}
\sum\nolimits_{a\in\mathcal{A}}\pi(a\mid s, h)=1,\ \pi(a\mid s,h)&\geq 0 \end{aligned}\ \ \forall(s,a,h)\right\}$$
is the set of valid policies. Note that \eqref{eq:lp} takes the optimistic objective with $\widehat f_k$ and the pessimistic (or conservative) constraint with $\widehat g_k$. The optimistic objective induces exploration while the pessimistic constraint guarantees constraint satisfaction. Moreover, we optimize over the space of the confidence set $\mathcal{P}_k$, which also encourages exploration.

Now that we have the optimization formulation~\eqref{eq:lp}, the next question is as to how we solve it.  We take the standard approach of using \emph{occupancy measures}~\citep{Altman}. An occupancy measure is essentially a joint probability distribution for the event that we observe the state-action pair $(s,a)$ at step $h$ and state $s'$ at step $h+1$. An occupancy measure defines a policy and a transition kernel. The converse is true in that we can define the occupancy measure associated with a given pair of a policy and a transition kernel. Introducing occupancy measure, we can reformulate~\eqref{eq:lp} as an optimization problem in terms of occupancy measures. 

In the reformulation, the objective and the constraint become linear in an occupancy measure. Hence, the reformulation is a linear program, and we refer to it as the \emph{extended linear program}~\citep{Altman}. Again, by solving it, we deduce an optimal occupancy measure, which corresponds to an optimal solution to~\eqref{eq:lp}. We defer the formal description of the extended linear program and occupancy measures to the appendix.

One issue, however, is that \eqref{eq:lp} can be infeasible at the beginning of the algorithm as $\widehat g_k$ can be too large to guarantee feasibility of \eqref{eq:lp}. Hence, the algorithm executes the safe baseline policy $\pi_b$ for the first few episodes until sufficient information is gathered so that \eqref{eq:lp} becomes feasible. The following lemma characterizes a sufficient number of episodes running the safe baseline policy to guarantee feasibility of~\eqref{eq:lp}.
\begin{lemma}\label{lemma:K0}
    With probability at least $1-14\delta$, $(\pi_b, P)$ is a feasible solution of \eqref{eq:lp} for any $k>K_0$ where
    \begin{equation}\label{eq:K0}
    K_0 = \widetilde{\cO}\left(\frac{H^3 S^2 A}{(\bar{C} - \bar{C}_b)^2}\right)
    \end{equation}
    where $\widetilde{\cO}(\cdot)$ hides factors polynomial in $\ln(HSAK/\delta)$.
\end{lemma}

\section{Regret Analysis of DOPE+}\label{sec:analysis}

Let us state our theoretical guarantees for DOPE+.
\begin{theorem}\label{theorem:violation-hard}
    Let $\vec\pi=(\pi_1,\ldots,\pi_K)$ denote policies computed by DOPE+ with $K_0$ given in~\eqref{eq:K0}. Then 
    $$\mathrm{Violation}(\vec\pi) = 0$$ 
    with probability at least $1-14\delta$.
\end{theorem}
Hence, DOPE+ achieves no constraint violation. The next theorem shows a regret upper bound for DOPE+.
\begin{theorem}\label{theorem:regret}
   Let $\vec\pi=(\pi_1,\ldots,\pi_K)$ denote policies computed by DOPE+ with $K_0$ given in~\eqref{eq:K0}. Then, with probability at least $1-16\delta$, we have
	\begin{align*}
    &\regret\left(\vec\pi\right)=\widetilde{\cO}\left(\frac{H}{\bar C - \bar C_b} \left(H^{1.5}S\sqrt{AK}+\frac{H^4S^3A}{\bar C-\bar C_b}\right)\right)
	\end{align*}
	 where $\widetilde{\cO}(\cdot)$ hides factors polynomial in $\ln(HSAK/\delta)$.
\end{theorem}

\begin{remark}
{\em
Note that there is a gap of $\widetilde{\cO}((\bar{C}-\bar{C}_b)^{-1}H\sqrt{S})$ factor between our regret upper bound and the lower bound $\Omega(H^{3/2}\sqrt{SAK})$ due to~\cite{Jin2020,pmlr-v132-domingues21a}. In fact, the instance from~\cite{pmlr-v132-domingues21a} is an unconstrained MDP. We observe that the $\cO(H/(\bar{C}-\bar{C}_b))$ factor in our regret upper bound is due to the constraint, which becomes a constant if $\bar C - \bar C_b=\Omega(H)$. Hence, our regret upper bound nearly matches in terms of $H$ when $\bar C - \bar C_b=\Omega(H)$.\qed
}
\end{remark}

\subsection{Constraint Violation Analysis}\label{sec:analysis:constraint}

We prove \Cref{theorem:violation-hard} as follows. For episode $k$ with $k\leq K_0$, \Cref{alg:hard} takes the safe baseline policy $\pi_b$, so no constraint violation is guaranteed. Then let us consider episode $k$ with $k> K_0$. As explained in \Cref{sec:estimator:function}, we argue that
\begin{align*}
 V_1^{\pi_k}(g, P)&\leq  V_1^{\pi_k}(g, P_k) + \left|V_1^{\pi_k}(g, P)-V_1^{\pi_k}(g, P_k)\right|\\
 &\leq V_1^{\pi_k}(\bar g_k+ R_k, P_k)+V_1^{\pi_k}(U_k, P_k)\\
 &= V_1^{\pi_k}(\widehat g_k, P_k)
\end{align*}
where the second inequality is due to \Cref{lemma:estimator} and \Cref{theorem:U_k}. Since $(\pi_k,P_k)$ is a solution to~\eqref{eq:lp}, it holds that $V_1^{\pi_k}(\widehat g_k, P_k)\leq \bar C$. Therefore, it follows that $V_1^{\pi_k}(g, P)\leq \bar C$ and thus the constraint is satisfied.

\subsection{Regret Decomposition}\label{sec:analysis:regret}
We provide an overview of the proof of \Cref{theorem:regret}. Since we execute the safe baseline policy $\pi_b$ for the first $K_0$ episodes, we decompose the regret function as follows.
{\allowdisplaybreaks
\begin{align*}
	\regret\left(\vec\pi\right)&=\underbrace{\sum_{k=1}^{K_0} \left( V^{\pi^*}_1(f,P) - V^{\pi_b}_1(f, P)\right)}_{\text{(I)}} + \underbrace{\sum_{k=K_0+1}^K \left(V^{\pi^*}_1(f,P)- V^{\pi_k}_1(\widehat f_k, P_k)\right)}_{\text{(II)}}\\ 
    &\quad + \underbrace{\sum_{k=K_0+1}^K \left( V^{\pi_k}_1(\widehat f_k, P_k) - V^{\pi_k}_1(\widehat f_k, P)\right)}_{\text{(III)}}+\underbrace{\sum_{k=K_0+1}^K \left( V^{\pi_k}_1(\widehat f_k, P) - V^{\pi_k}_1(f, P)\right)}_{\text{(IV)}}.
\end{align*}}

Term (I) is due to executing $\pi_b$ for $K_0$ episodes for feasibility. By \Cref{lemma:K0}, term (I) can be bounded by $\widetilde{\cO}((\bar C-\bar C_b)^{-2}(H^4 S^2 A))$ as $V_1^{\pi}\leq H$ for any policy $\pi$.

For term (II), we provide the following upper bound.
\begin{lemma}\label{lemma:term-1}
With probability at least $1-14\delta$,
\begin{align*}
\sum_{k=K_0+1}^K \left(V^{\pi^*}_1(f,P)- V^{\pi_k}_1(\widehat f_k, P_k)\right)
&=\widetilde{\cO}\left(\frac{H}{\bar C - \bar C_b} \left(H^{1.5}S\sqrt{AK}+H^3S^3A\right)\right)
\end{align*}
 where $\widetilde{\cO}(\cdot)$ hides factor $(\ln(HSAK/\delta))^3$.
\end{lemma}
To prove the lemma, we define a new policy $\pi_k^{\alpha_k}$ for $k\in [K]$, which is a convex combination of the occupancy measures associated with $(\pi^*,P)$ and $(\pi_b,P)$ with coefficients $\alpha_k,1-\alpha_k \in (0,1)$. We choose the value of $\alpha_k$ so that $(\pi_k^{\alpha_k},P)$ is feasible to \eqref{eq:lp}. Then the optimality of $(\pi_k,P_k)$ implies $V_1^{\pi_k^{\alpha_k}}(\widehat f_k, P) \leq V_1^{\pi_k}(\widehat f_k, P_k)$, which lets us to analyze $V^{\pi^*}_1(f,P)- V_1^{\pi_k^{\alpha_k}}(\widehat f_k, P)$ with the same transition kernel $P$.

Term (III) comes from learning the unknown transition kernel. We apply a Bellman-type law of total variance to provide an upper bound on term (III).
\begin{lemma}\label{lemma:term-2}
With probability at least $1-16\delta$,
\begin{align*}
\sum_{k=K_0+1}^K \left( V^{\pi_k}_1(\widehat f_k, P_k) - V^{\pi_k}_1(\widehat f_k, P)\right)
&=\widetilde{\cO}\left(H^{1.5}S\sqrt{AK}+H^3S^3A\right)
\end{align*}
where $\widetilde{\cO}(\cdot)$ hides factor $(\ln(HSAK/\delta))^4$.
\end{lemma}

Term (IV) is due to the difference between $f$ and our estimator $\widehat f_k$. 
\begin{lemma}\label{lemma:term-3}
 With probability at least $1-14\delta$,
 \begin{align*}
     \sum_{k=K_0+1}^K \left( V^{\pi_k}_1(\widehat f_k, P) - V^{\pi_k}_1(f, P)\right)
     &=\widetilde{\cO}\left(\frac{H}{\bar C - \bar C_b} \left(H^{1.5}S\sqrt{AK}+H^3S^3A\right)\right)
 \end{align*}
 where $\widetilde{\cO}(\cdot)$ hides factor $(\ln(HSAK/\delta))^3$.
\end{lemma}

\section{Numerical Experiment}\label{sec:numerical}
We evaluate DOPE+ on the three-state CMDP instance of \cite{zheng2020constrained, simao2021always, bura2022dope} with a few modifications. In \Cref{fig:numerical:plot}, we compare regret and constraint violation under DOPE+ and DOPE for $200,000$ episodes when $H=30$. We consider DOPE as a benchmark algorithm because it provides the best regret bound among the existing algorithms while ensuring zero constraint violation. Our results are averaged across 5 runs with different random seeds, and we display the $95\%$ confidence interval with shaded regions. More details of the experiment setup can be found in the appendix including the MDP instance and algorithm parameters. %

In \Cref{fig:numerical:plot}, DOPE+ outperforms DOPE in terms of regret. This result demonstrates that DOPE+ improves upon DOPE computationally, in addition to our theoretical improvement. \Cref{fig:numerical:plot} shows that both algorithms achieve zero constraint violation.

\begin{figure}[ht]
    \centering
    \includegraphics[scale=0.55]{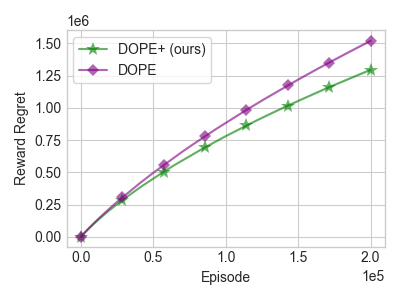}
    \includegraphics[scale=0.55]{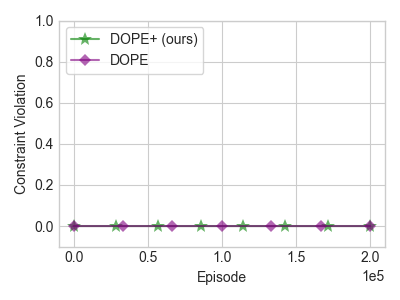}
  \caption{Comparison of DOPE+ and DOPE: Regret (Left) and Hard Constraint Violation (Right)}
  \label{fig:numerical:plot}
\end{figure}

\section{Conclusion}
In this paper, we investigate safe RL formulated as an episodic finite-horizon tabular CMDP. We propose novel reward and cost function estimators with tighter reward optimism and cost pessimism. Based on them, we develop DOPE+, which is a variant of the LP-based DOPE due to~\citep{bura2022dope}. We prove that DOPE+ achieves regret upper bound $\widetilde{\cO}((\bar C- \bar C_b)^{-1}H^{2.5}S\sqrt{AK})$ and zero hard constraint violation. The regret upper bound improves upon the best-known bound by a multiplicative factor of $\widetilde{\cO}(\sqrt{H})$ factor. When $\bar C - \bar C_b=\Omega(H)$, DOPE+ nearly matches the lower bound $\Omega(H^{1.5}\sqrt{SAK})$~\citep{Jin2020,pmlr-v132-domingues21a}. We also present numerical results that demonstrate the computational effectiveness of DOPE+ compared to DOPE.

\acks{This research is supported by the National Research Foundation of Korea (NRF) grant (No. RS-2024-00350703).}

\bibliography{biblio}

\newpage
\appendix

\section{Related Work}

In this section, we provide a more detailed discussion of related work to online learning of constrained Markov decision processes (CMDPs). As explained in the introduction, we review previous works for the three frameworks, cumulative constraint violation, hard constraint violation, and zero constraint violation.

\paragraph{Cumulative Constraint Violation} Starting with the work of~\cite{efroni2020}, online learning of CMDPs has been an active area of research in reinforcement learning, especially with the framework of cumulative (or soft) constraint violation~\citep{10.5555/3495724.3497093,qiu2020,zheng2020constrained,Kalagarla_Jain_Nuzzo_2021,ding2021opdop,chen2021primaldualapproachconstrainedmarkov,pmlr-v139-yu21b,liu2021learning,Wei_Liu_Ying_2022,wei2022triple,9872114,pmlr-v162-miryoosefi22a,NEURIPS2022_56b8f22d,pmlr-v206-wei23b,kalagarla2023safeposteriorsamplingconstrained}. Among these works, \cite{10.5555/3495724.3497093} studied a knapsack constrained formulation, and \cite{qiu2020} studied the setting where the reward functions are adversarially given and the cost functions are sampled from a fixed but unknown distribution. Moreover, \cite{zheng2020constrained} considered the case where the transition kernel is known to the agent, and \cite{Kalagarla_Jain_Nuzzo_2021} studied a PAC bound for learning CMDPs. \cite{ding2021opdop, chen2021primaldualapproachconstrainedmarkov} developed model-free algorithms for CMDPs, although these approaches require access to simulators, while \cite{pmlr-v139-yu21b} studied vector-valued Markov games for a variant of constrained MDPs. \cite{liu2021learning} introduced the first algorithm that achieves zero cumulative constraint violation. \cite{Wei_Liu_Ying_2022} and \cite{9872114} considered the infinite-horizon average-reward setting. Moreover, \cite{wei2022triple} came up with a model-free algorithm for finite-horizon episodic tabular CMDPs. \cite{pmlr-v162-miryoosefi22a} studied the reward-free setting, and \cite{NEURIPS2022_56b8f22d} proposed an algorithm for the linear MDP setting, which leads to a model-free algorithm for tabular CMDPs. Lastly, \cite{pmlr-v206-wei23b} considered non-stationary CMDPs, while~\cite{kalagarla2023safeposteriorsamplingconstrained} developed a posterior sampling-based algorithm that guarantees a Bayesian regret upper bound.

 \cite{wei2022triple} introduced model-free and simulator-free algorithms to solve tabular CMDPs. These algorithms were analyzed under soft constraint violations, thus they do not guarantee safety in all episodes. In contrast, \cite{muller2024truly, ghosh2024towards} presented PD-based algorithms with hard constraint violations, though these suffer from high regret and constraint violations. On the other hand, \cite{liu2021learning} proposed the LP-based algorithm OptPess-LP, which achieves zero hard constraint violations with sublinear regret by employing \emph{optimistic pessimism in the face of uncertainty (OPFU)}. The pessimism in the cost function estimator ensures safety but hampers exploration. To address this, \cite{bura2022dope} recently proposed DOPE, incorporating optimism for the transition kernel to improve the regret bound.

\paragraph{Hard Constraint Violation} The notion of hard constraint violation was introduced by~\cite{efroni2020}. \cite{efroni2020} developed an LP-based algorithm for controlling hard constraint violation and raised an open question of whether there exists a primal-dual algorithm for the setting. Recently, \cite{ghosh2024towards} established an algorithm that guarantees a sublinear regret upper bound and a sublinear upper bound on hard constraint violation. Their algorithm is for the linear MDP setting, and it provides a model-free algorithm for the tabular setting. In fact, their analysis shows that for the tabular case, one may get a tighter performance guarantees. \cite{muller2024truly} developed a simpler primal-dual algorithm that guarantees a sublinear regret upper bound and a sublinear upper bound on hard constraint violation, answering the question of \cite{efroni2020}.

\paragraph{Zero Constraint Violation}

\cite{simao2021always} considered the importance of achieving no constraint violation, which is equivalent to zero hard constraint violation. They showed an algorithm that guarantees no constraint violation, but their result relies on the assumption of some abstraction of the transition model, and moreover, there is no regret upper bound given for the algorithm. \cite{liu2021learning} established the first algorithm that achieves a sublinear regret while guaranteeing zero hard constraint violation. After \cite{liu2021learning}, \citep{bura2022dope} proposed their algorithm, DOPE, which improves upon \cite{liu2021learning} to show a smaller regret upper bound.

\section{Auxiliary Measures and Notations}\label{sec:appendix:auxiliary}

In this section, we first summarize notations in~\Cref{table:notations}. Next, we define some auxiliary measures and notations that are useful for the analysis of DOPE+.

\begin{table}[ht]
\caption{Summary of Notations}
\label{table:notations}
\begin{center}
\begin{tabular}{ll}
\toprule
\multicolumn{1}{c}{Notation}  &\multicolumn{1}{c}{Definition}\\ 
\hline
$K$ & The number of episodes \\
$H$ & The finite horizon \\
$[H]$ & The set $\{1,2,\ldots, H\}$ \\
$\cS,\ S$ & The finite state space $\cS$ and the number of states $S = |\cS|$\\
$\cA,\ A$ & The finite action space $\cA$ and the number of actions $A = |\cA|$\\
$P$ & The true transition kernel $P(s,a,s',h):\cS\times\cA\times\cS\times[H]\to[0,1]$ \\
$p$ & The initial distribution of the states\\
$\cP_k$ & The confidence set of the transition kernel for episode $k\in[K]$\\
$P_k$ & The transition kernel obtained from DOPE+ for episode $k\in [K]$, $P_k \in \cP_k$\\
$f,\ g$ & The reward and cost function\\
$f_k,\ g_k$ & The instantaneous reward and cost for episode $k\in[K]$ \\
$\bar f_k,\ \bar g_k$ & The empirical estimators of $f,g$ for episode $k\in[K]$ \\
$\widehat f_k,\ \widehat g_k$ & The optimistic/pessimistic estimators of $f,g$ for episode $k\in[K]$ \\
$V_h^{\pi}(s;f,P)$ & The value function at state $s$ and step $h$ under $f$ and $P$\\
$Q_h^{\pi}(s,a;f,P)$ & The state-action value function at state $s$ and step $h$ for action $a$ under $f$ and $P$\\
$N_k(s,a,h)$ & The number of visits $(s,a,h)$ up to the first $k-1$ episodes\\
$M_k(s,a,s',h)$ & The number of visits $(s,a,s',h)$ up to the first $k-1$ episodes\\
$n_k(s,a,h)$ & The indicator variable for visits $(s,a,h)$ for episode $k\in[K]$\\
$\pi^*$ & The benchmark policy\\
$\pi_k$ & The policy obtained from DOPE+ for episode $k\in[K]$\\
$\pi_b$ & The safe baseline policy\\
$\bar C_b$ & The expected cost of $\pi_b$ for a single episode\\
$\bar C$ & The budget on the expected cost\\
$q^{P,\pi}$ & The occupancy measure with respect to policy $\pi$ and transition kernel $P$\\
$q^*$ & The occupancy measure  $q^{P,\pi^*}$ \\
$q_b$ & The occupancy measure  $q^{P,\pi_b}$ \\
$q_k$ & The occupancy measure $q^{P,\pi_k}$ \\
$\widehat q_k$ & The occupancy measure $q^{P_k,\pi_k}$\\
$\Delta(P)$ & The set of occupancy measures inducing any policy $\pi$ and the true transition kernel $P$\\
$\Delta(P,k)$ & The set of occupancy measures inducing any policy $\pi$ and transition kernel $P_k\in\cP_k$\\
\bottomrule 
\end{tabular}
\end{center}
\end{table}

We define the \emph{state-action value function} for $(s,a)\in\cS\times \cA$ at step $h$ with a  function $\ell:\cS\times\cA\times[H]\to[0,1]$ and transition kernel $P$ as follows.
\begin{equation*}
Q^{\pi}_h(s,a;\ell,P) = \mathbb{E}\left[\sum_{j=h}^{H}\ell\left(s_{j}^{P,\pi},a_{j}^{P,\pi},j\right)\mid \ell,\pi,P, s_h^{P,\pi}=s,a_h^{P,\pi}=a\right].
\end{equation*}
Let $\bm{Q^{P,\pi,\ell}}$ denote the $(S\times A\times H)$-dimensional vector whose coordinates are for $(s,a,h)\in\cS\times\cA\times[H]$,
$$(\bm{Q^{P,\pi,\ell}})_{(s,a,h)} = Q^{\pi}_h(s,a;\ell,P).$$
Given a policy $\pi$ and transition kernel $P$, we define $q^{P,\pi}\left(s,a,h\mid s',m\right)$ as for $(s,a,s')\in \cS\times \cA\times \cS$ and $1\leq m\leq h\leq H$,
\begin{equation*}%
q^{P,\pi}\left(s,a,h\mid s',m\right)=\mathbb{P}\left[s^{ P, \pi}_{h}=s,\ a^{P,\pi}_{h}=a\mid \pi,P, s^{ P, \pi}_{m}=s'\right].
\end{equation*}

Given two vectors $\bm{u},\bm{v}\in \mathbb{R}^{S\times A\times H}$, let $\bm{u}\odot \bm{v}$, $\bm{u}\wedge\bm{v}$ be defined as the vector obtained from coordinate-wise products and coordinate-wise minimization of $\bm{u}$ and $\bm{v}$, respectively, i.e., for  $(s,a,h)\in\cS\times\cA\times[H],$
$$(\bm{u}\odot \bm{v})_{(s,a,h)} = \bm{u}_{(s,a,h)}\times \bm{v}_{(s,a,h)},\quad (\bm{u}\wedge\bm{v})_{(s,a,h)} = \min\{\bm{u}_{(s,a,h)}, \bm{v}_{(s,a,h)}\}.$$ 

Let $\bm{\vec h}$ and $\bm{\vec{B}}$ be $(S\times A\times H)$-dimensional vectors all of whose coordinates are $h$ and $1+\sqrt{\ln(HSAK/\delta)}$, respectively, i.e., for $(s,a,j)\in \cS\times\cA\times[H],$
$$\bm{\vec h}_{(s,a,j)}=j, \ \bm{\vec{B}}_{(s,a,j)}=1+\sqrt{\ln(HSAK/\delta)}.$$

\section{Extended Linear Program}\label{sec:appendix:occupancy}

In this section, we provide a formal definition of occupancy measures for a finite-horizon MDP. Then we provide a reformulation of \eqref{eq:lp} using occupancy measures, which is called the extended linear program~\citep{efroni2020, bura2022dope}.

 Given a policy $\pi$ and  a transition kernel $P$, let $\bar q^{P, \pi}:\cS\times\cA\times\cS\times [H]\to[0,1]$ be defined as $\bar q^{P,\pi}(s,a,s',h)=\mathbb{P}[(s^{P,\pi}_{h},a^{P,\pi}_{h},\ s^{P, \pi}_{h+1})=(s,a,s')\mid \pi,P]$
for $(s,a,s',h)\in \cS\times \cA\times \cS\times [H]$. Note that any $\bar q$ defined as the above equation has the following properties.
(C1) $\sum_{(s,a,s')\in \cS\times \cA\times \cS}\bar q(s,a,s',h)=1$, (C2) $\sum_{(s',a)\in\cS\times \cA}\bar q(s,a,s',h)=\sum_{(s',a)\in\cS\times\cA}\bar q(s',a,s,h-1), \quad s \in \cS, \ h=2,\ldots,H$.
The \emph{occupancy measure} $q^{P,\pi}:\cS\times\cA\times[H]\to[0,1]$ associated with policy $\pi$ and transition kernel $P$ is defined as (C3) $q^{P,\pi}(s,a,h)=\sum_{s'\in\cS} \bar q^{P, \pi}(s,a,s',h)$.
Then it follows that
$q^{P,\pi}(s,a,h)=\mathbb{P}[(s^{P,\pi}_{h},a^{P,\pi}_{h})=(s,a)\mid\pi,P]$.
Hence, if a policy $\pi$ is chosen, then the occupancy measure for a finite-horizon MDP with transition kernel $P$ is determined. Conversely, any $q\in\cS\times\cA\times [H]\to[0,1]$ with $\bar q:\cS\times\cA\times\cS\times [H]\to[0,1]$ satisfying~(C1), (C2), and (C3) induces a transition kernel $P^{q}$ and a policy $\pi^{q}$ given as follows. 
\begin{align}\label{induced}
\begin{aligned}
&P^{q}(s'\mid s,a,h)=\frac{\bar q(s,a,s',h)}{\sum_{s''\in \cS}\bar q(s,a,s'',h)},\quad \pi^{q}(a\mid s,h)=\frac{q(s,a,h)}{\sum_{b\in \cA} q(s,b,h)}.
\end{aligned}
\end{align}
Next, we provide a lemma that characterizes valid occupancy measures for a finite-horizon MDP.
\begin{lemma}\label{lemma:valid-occupancy}
Let $q:\cS\times \cA\times[H]\to[0,1]$. Then $q$ is a valid occupancy measure that induces transition kernel $P$ if and only if there exists $\bar q:\cS\times \cA\times\cS\times [H]\to[0,1]$ that satisfies~(C1), (C2), (C3), and $P^{q}=P$.
\end{lemma}
\begin{proof}
Given the finite-horizon MDP associated with transition kernel $P$, we may define a loop-free MDP as follows. We define its state space as $\cS':=\cS\times [H+1]$, which can be viewed as $H+1$ layers $\cS\times\{h\}$ for $h\in[H+1]$. Its transition kernel $P'$ is given by $P'(\tshn\mid \tsh,a) = P(s'\mid s, a,h)$ for $(s,a,s',h)\in \cS\times \cA\times \cS\times [H]$. Next, given $\bar q$, we may define an occupancy measure $q'$ for the loop-free MDP as $q'(\tsh,a,\tshn)=\bar q(s,a,s',h)$ for $(s,a,s',h)\in \cS\times \cA\times \cS\times [H]$. Then it follows from~\cite[Lemma 3.1]{rosenberg2019} that $q'$ is a valid occupancy measure for the loop-free MDP with transition kernel $P'$ if and only if $q'$ satisfies
\begin{align}
&\sum_{(s,a,s')\in \cS\times \cA\times \cS}q'(\tsh,a,\tshn)=1\quad\text{for~} h=1,\ldots, H,\tag{C1'}\\
&\sum_{(s',a)\in\cS\times \cA}q'(\tsh,a,\tshn)=\sum_{(s',a)\in\cS\times\cA}q'((s',h-1),a,\tsh)\quad \text{for~}s\in \cS,\ h\in\{2,\ldots, H\},\tag{C2'}
\end{align}
and $P^{q'}=P'$ where $P^{q'}$ is given by 
$$P^{q'}(\tshn\mid \tsh,a)=\frac{q'(\tsh,a,\tshn)}{\sum_{s''\in \cS} q'(\tsh,a,(s'',h+1))}=\frac{\bar q(s,a,s',h)}{\sum_{s''\in \cS}\bar q(s,a,s'',h)}.$$
Here, the conditions are equivalent to (C1), (C2), and $P^{\bar q}=P$. Moreover, $q'$ is a valid occupancy measure with $P'$ if and only if $q$ is a valid occupancy measure with $P$, as required.
\end{proof}

Therefore, there is a one-to-one correspondence between the set of policies and the set of occupancy measures that give rise to transition kernel $P$. We define $\Delta(P)$ as the set of occupancy measures inducing the true transition kernel $P$.
$$\Delta(P)=\left\{\bm{q}:\ \exists \bm{\bar q}\text{ satisfying~(C1),(C2),(C3)},P^{q}=P\right\}.$$
Moreover, the value function for reward function $f$, policy $\pi_k$, and transition kernel $P$ can be written in terms of occupancy measure $q^{P,\pi_k}$ as $V^{\pi_k}_1(f,P)=\sum_{(s,a,h)}q^{P,\pi_k}\left(s,a,h\right)f\left(s,a,h\right)$. Let $\bm{q^{P,\pi}}, \bm{f}$ denote $(S\times A\times H)$-dimensional vector representations for $q^{P,\pi}, f$, respectively. Then it follows that $V_1^{\pi_k}(f, P)=\langle \bm{f}, \bm{q^{P,\pi_k}}\rangle$ where $\langle\cdot,\cdot\rangle$ is the inner product. Consequently, \eqref{eq:lp} is equivalent to 
\begin{equation}\label{eq:lp-occupancy}
    \max_{q\in \Delta(P,k)}\left\{\langle  \bm{\widehat f_k}, \bm{q}\rangle: \langle\bm{\widehat g_k}, \bm{q}\rangle\leq \bar C\right\}
\end{equation}
where $\bm{\widehat f_k}, \bm{\widehat g_k}$ are the vector representations of $\widehat f_k, \widehat g_k$, respectively, and 
$$\Delta(P,k)=\left\{\bm{q}:\ \exists \bm{\bar q}\text{ satisfying~(C1),(C2),(C3)},P^{q}\in\cP_k\right\}.$$

Next, we reformulate \eqref{eq:lp} as an extended linear program. 
Due to the definition of $\Delta(P,k)$, \eqref{eq:lp-occupancy} is equivalent to the following linear program. Given $\widehat f_k(s,a,h),\ \widehat g_k(s,a,h),\ \bar P_k(s'\mid s,a,h),\ \epsilon_k(s'\mid s,a,h),\ p(s)$ for $(s,a,s',h)\in \cS\times\cA\times\cS\times[H]$,
\begin{align*}
    \max&\quad \sum_{(s,a,s',h)\in\cS\times\cA\times\cS\times[H]} \widehat f_k(s,a,h) \bar q(s,a,s',h) \\
    \text{s.t.}&\quad \sum_{(s,a,s',h)\in\cS\times\cA\times\cS\times[H]} \widehat g_k(s,a,h) \bar q(s,a,s',h) \leq \bar C,\\
    &\quad \sum_{(a,s')\in\cA\times\cS}\bar q(s,a,s',h) = \sum_{(a,s')\in\cA\times\cS}\bar q(s',a,s,h-1) \quad \forall s \in \cS, \ h=2,\ldots,H,\\
    &\quad \sum_{(a,s')\in\cA\times\cS}\bar q(s,a,s',1) = p(s) \quad \forall s\in \cS, \\
    &\quad \bar q(s,a,s',h) \leq \left(\bar P_k(s'\mid s,a,h)+\epsilon_k(s'\mid s,a,h)\right)\sum_{s'\in \cS}\bar q(s,a,s',h) \quad \forall (s,a,s',h)\in \cS\times\cA\times\cS\times[H],\\
    &\quad \bar q(s,a,s',h) \geq \left(\bar P_k(s'\mid s,a,h)-\epsilon_k(s'\mid s,a,h)\right) \sum_{s'\in \cS}\bar q(s,a,s',h) \quad \forall (s,a,s',h)\in \cS\times\cA\times\cS\times[H],\\
    &\quad 0\leq \bar q(s,a,s',h) \quad \forall (s,a,s',h)\in \cS\times\cA\times\cS\times[H].
\end{align*}
\noindent
In fact, the constraint $\sum_{(s,a,s')} \bar q(s,a,s',h) = 1$ for $h\in[H]$ corresponding to (C1) is not necessary, because we can derive it from other constraints. To be more specific, the third constraint implies that $\sum_{(s,a,s')}\bar q(s,a,s',1)=1$ as $\sum_{s}p(s)=1$. Then we can deduce from the second constraint that $\sum_{(s,a,s')} \bar q(s,a,s',h)=1$ for $h\in[H]$. Additionally, we
call the above linear program as an extended linear program due to the fifth and sixth constraints.

\section{Good Event}\label{sec:appendix:good-event}

In this section, we first prove \Cref{lem:bounded} which ensures that all instantaneous reward and cost values are bounded. Then we prove \Cref{lemma:confidence} that describes important properties of the confidence sets estimating the true transition kernel. Next, we show \Cref{lemma:estimator} which delineates the accuracy of our estimators of the reward function $f$ and the cost function $g$.   

Furthermore, we prove \Cref{lemma8} that is useful to bound value functions with respect to estimated reward and cost functions. Then we define the notion of the \emph{good event} $\cE$ that the statements of \Cref{lem:bounded,lemma:confidence,lemma:estimator,lemma8} hold. Taking the union bound, we deduce that the good event $\cE$ holds with probability at least $1-14\delta$ (\Cref{lemma:good-event}).

Lastly, we prove \Cref{lemma:confidence'} which considers the difference between the true transition kernel and any $\widehat P$ contained in the confidence set $\cP_k$.

\begin{proof}[\textbf{Proof of \Cref{lem:bounded}}]
It follows from Hoeffding's inequality (\Cref{hoeffding}) and the union bound that for any $(s,a,h)\in\mathcal{S}\times\mathcal{A}\times[H]$ and $k\in[K]$,
$$\mathbb{P}\left(\left|f_k(s,a,h)-f(s,a,h)\right|\geq \sqrt{\ln({HSAK}/{\delta})}\right)\leq 2\cdot\exp\left(-\ln({HSAK}/{\delta})\right)=\frac{2\delta}{HSAK}.$$
Likewise, for any $(s,a,h)\in\mathcal{S}\times\mathcal{A}\times[H]$ and $k\in[K]$,
$$\mathbb{P}\left(\left|g_k(s,a,h)-g(s,a,h)\right|\geq \sqrt{\ln({HSAK}/{\delta})}\right)\leq 2\cdot\exp\left(-\ln({HSAK}/{\delta})\right)=\frac{2\delta}{HSAK}.$$
Taking the union bound, it follows that with probability at least $1-4\delta$,
$$\left|f_k(s,a,h)-f(s,a,h)\right|, \left|g_k(s,a,h)-g(s,a,h)\right|\leq \sqrt{\ln({HSAK}/{\delta})}$$
holds for all $(s,a,h)\in\mathcal{S}\times\mathcal{A}\times[H]$ and $k\in[K]$. Since $f(s,a,h),g(s,a,h)\in[0,1]$ for any $(s,a,h)\in \mathcal{S}\times\mathcal{A}\times[H]$, it holds with probability at least $1-4\delta$ that 
$$\left|f_k(s,a,h)\right|, \left|g_k(s,a,h)\right|\leq 1+\sqrt{\ln({HSAK}/{\delta})},$$
as required.
\end{proof}

The following lemma is a modification of~\cite[Lemma 8]{Jin2020} to our finite-horizon MDP setting.

\begin{proof}[\textbf{Proof of \Cref{lemma:confidence}}]
We will show that with probability at least $1-4\delta$, 
\begin{equation}\label{star}
\left|P(s'\mid s,a,h)-\bar P_k(s'\mid s,a,h)\right|\leq \epsilon_k(s'\mid s,a,h)
\end{equation}
where
\begin{equation*}
\epsilon_k(s'\mid s,a,h)= 2\sqrt{\frac{\bar{P}_{k}(s'\mid s,a,h)\ln\left({{HSAK}/{\delta}}\right)}{\max\{1,N_k(s,a,h)-1\}}}+\frac{14\ln\left({{HSAK}/{\delta}}\right)}{3\max\{1,N_k(s,a,h)-1\}}
\end{equation*}
holds for every $(s,a,s',h)\in \cS\times \cA\times \cS\times [H]$ and every episode $k\in[K]$. 

Let us first consider the case $N_k(s,a,h)\leq 1$. As we may assume that $HSAK\geq 2$, it follows that
$$\epsilon_k(s'\mid s,a,h)=\frac{14\ln\left({{HSAK}/{\delta}}\right)}{3\max\{1,N_k(s,a,h)-1\}}\geq \frac{14}{3} \ln 2>1.$$
Then \eqref{star} holds because $0\leq P(s'\mid s,a,h),\bar P_k(s'\mid s,a,h)\leq 1$.

Assume that $n= N_k(s,a,h)\geq 2$. Then we define $Z_1, \ldots, Z_n$ as follows.
$$Z_j=\begin{cases}
1,\quad &\text{if the transition after the $j$th visit to $(s,a,h)$ is $s'$},\\
0,\quad&\text{otherwise}.
\end{cases}$$
Then $Z_1,\ldots, Z_n$ are i.i.d. with mean $P(s'\mid s, a, h)$, and we have
$$\sum_{j=1}^n Z_j = M_k(s, a,s', h).$$
Moreover, the sample variance $V_n$ of $Z_1,\ldots, Z_n$ is given by
\begin{align}\label{eq:lemma:confidence-variance}
\begin{aligned}
V_n &= \frac{1}{N_k(s,a,h)(N_k(s,a,h)-1)} M_k(s,a,s', h)\left(N_k(s,a,h)- M_k(s, a, s',h)\right)\\
&=\frac{N_k(s,a,h)}{(N_k(s,a,h)-1)} \bar P_k(s'\mid s, a, h)\left(1- \bar P_k(s'\mid s, a, h)\right).
\end{aligned}
\end{align}
Then it follows from \Cref{bernstein} that with probability at least $1- 2\delta/(HS^2AK)$,
\begin{align}\label{eq:lemma:confidence-1}
\begin{aligned}
&P(s'\mid s,a,h)-\bar P_k(s'\mid s,a,h) \leq \sqrt{\frac{2\bar P_k(s'\mid s, a, h)\left(1- \bar P_k(s'\mid s, a, h)\right)\ln\left({{HS^2AK}/{\delta}}\right)}{N_k(s,a,h)-1}}+ \frac{7\ln\left(HS^2AK/\delta\right)}{3(N_k(s,a,h)-1)}.
\end{aligned}
\end{align}
Here, as we assumed that $N_k(s,a,h)\geq 2$, we have $N_k(s,a,h)-1=\max\{1,N_k(s,a,h)-1\}$. In addition, we know that $1- \bar P_k(s'\mid s, a, h)\leq 1$ and that $\ln\left({{HS^2AK}/{\delta}}\right)\leq 2\ln\left(HSAK/\delta\right)$. Then \eqref{eq:lemma:confidence-1} implies that with probability at least $1- 2\delta/(HS^2AK)$,
\begin{equation}\label{eq:lemma:confidence-2}
    P(s'\mid s,a,h)-\bar P_k(s'\mid s,a,h)\leq \epsilon_k(s'\mid s,a,h).
\end{equation}
Next, we apply \Cref{bernstein} to variables $1-Z_1,\ldots, 1-Z_n$ that are i.i.d. and have mean $1-\bar P_k(s'\mid s, a, h)$. Moreover, the sample variance of $1-Z_1,\ldots, 1-Z_n$ is also equal to $V_n$ defined as in~\eqref{eq:lemma:confidence-variance}. Therefore, based on the same argument, we deduce  that 
with probability at least $1- 2\delta/(HS^2AK)$,
\begin{equation}\label{eq:lemma:confidence-3}
   - P(s'\mid s,a,h)+\bar P_k(s'\mid s,a,h)\leq \epsilon_k(s'\mid s,a,h).
\end{equation}
By applying union bound to~\eqref{eq:lemma:confidence-2} and~\eqref{eq:lemma:confidence-3}, with probability at least $1- 4\delta/(HS^2AK)$,~\eqref{star} holds for $(s,a,s',h)$.
Furthermore, by applying union bound over all $(s,a,s',h)\in \cS\times \cA\times\cS\times [H]$, it follows that with probability at least $1-4\delta$,~\eqref{star} holds for every $(s,a,s',h)\in \cS\times \cA\times\cS\times [H]$, as required.
\end{proof}

Next, we state the proof of \Cref{lemma:estimator} based on Hoeffding's inequality.

\begin{proof}[\textbf{Proof of \Cref{lemma:estimator}}]
 If $N_k(s,a,h)=\sum_{j=1}^{k-1} n_j(s,a,h)=0$, then $\bar f_k(s,a,h) = \bar g_k(s,a,h)=0$ while $R_k(s,a,h) \geq 1$ when we may assume that $HSAK \geq 4.$ In this case, the statements trivially hold. Now we consider when $\sum_{j=1}^{k-1} n_j(s,a,h)\geq 1$. Note that $f_k(s,a,h)=f(s,a,h)+\zeta_k^f(s,a,h)$ and $g_k(s,a,h)=g(s,a,h)+\zeta_k^g(s,a,h)$ where $\zeta_k^f(s,a,h)$ and $\zeta_k^g(s,a,h)$ are i.i.d. $1/2$-sub-Gaussian random variables with zero mean for each $(s,a,h) \in \cS \times \cA \times [H]$ and $k\in[K]$. Then it follows from the Hoeffding's inequality provided in \Cref{hoeffding} that for a given $(s,a,h)\in\cS\times\cA\times[H]$ and $k\in[K]$, 
\begin{equation}\label{eq:lemma:estimator:1}
    \left| \bar{f}_k(s,a,h) - f(s,a,h) \right| \leq R_k(s,a,h)
\end{equation}
with probability at least $1-2\delta / (HSAK)$. By applying union bound, \eqref{eq:lemma:estimator:1} holds with probability at least $1-2\delta$ for all $(s,a,h)\in\cS\times\cA\times[H]$ and $k\in [K]$.
Likewise, we deduce for $g$ that with probability at least $1-2\delta$, 
$$\left| \bar{g}_k(s,a,h) - g(s,a,h) \right| \leq R_k(s,a,h)$$ for $(s,a,h)\in\cS\times\cA\times[H]$ and $k\in [K]$ as desired.
\end{proof}

Next, using~\Cref{bernstein2} that states the Bernstein-type concentration inequality for a martingale difference sequence, we prove the following lemma that is useful for our analysis. \Cref{lemma8} is a modification of \citep[Lemma 10]{Jin2020} and \citep[Lemma 8]{ssp-adversarial-unknown} to our finite-horizon MDP setting.

\begin{lemma}\label{lemma8}
With probability at least $1-2\delta$, we have
\begin{align}
\sum_{k=1}^K\sum_{(s,a,h)\in \cS\times \cA\times[H]}\frac{q_k(s,a,h)}{\max\left\{1,N_k(s,a,h)\right\}}&
\leq 2HSA\ln K + 2HSA + 4H\ln(H/\delta)\label{first-ineq}\\
\sum_{k=1}^K\sum_{(s,a,h)\in \cS\times \cA\times[H]}\frac{q_k(s,a,h)}{\sqrt{\max\left\{1,N_k(s,a,h)\right\}}}&
\leq 2H\sqrt{SAK}+2HSA\ln K+3HSA+5H\ln(H/\delta) \label{second-ineq}
\end{align}
\end{lemma}
\begin{proof}
We define $\xi_1$ as $\xi_1 = \emptyset$
and for $k\geq 2$, we define $\xi_k$ as 
\begin{align*}
{\left\{s_h^{P,\pi_{k-1}},a_h^{P,\pi_{k-1}},f_{k-1}(s_h^{P,\pi_{k-1}}, a_h^{P,\pi_{k-1}}, h),g_{k-1}(s_h^{P,\pi_{k-1}}, a_h^{P,\pi_{k-1}}, h)\right\}_{h=1}^H}
\end{align*}
where $\pi_{k-1}$ denotes the policy for episode $k-1$ and $$\left(s_1^{P,\pi_{k-1}},a_1^{P,\pi_{k-1}},\ldots, s_h^{P,\pi_{k-1}}, a_h^{P,\pi_{k-1}}\right)$$
is the trajectory generated under policy $\pi_{k-1}$ and transition kernel $P$. Then for $k\in[K]$, let $\cF_k$ be defined as the $\sigma$-algebra generated by the random variables in $\xi_1\cup\cdots\cup \xi_{k}$. Then it follows that $\cF_1,\ldots, \cF_k$ give rise to a filtration.

Note that
\begin{equation}\label{mds:bound}
\sum_{k=1}^K\sum_{(s,a)\in \cS\times \cA}\frac{q_k(s,a,h)}{\max\left\{1,N_k(s,a,h)\right\}}=\sum_{k=1}^K\sum_{(s,a)\in \cS\times \cA}\frac{n_k(s,a,h)}{\max\left\{1,N_k(s,a,h)\right\}}+\sum_{k=1}^KY_k
\end{equation}
where
$$Y_k=\sum_{(s,a)\in \cS\times \cA}\frac{-n_k(s,a,h)+q_k(s,a,h)}{\max\left\{1,N_k(s,a,h)\right\}}.$$
As $\mathbb{E}\left[n_k(s,a,h)\mid \pi_k,P\right] = q_k(s,a,h)$ holds for every $(s,a,h)\in\cS\times \cA\times [H]$, we know that $Y_1,\ldots, Y_K$ is a martingale difference sequence. We know that $Y_k\leq 1$ for each $k\in[K]$. Let $\mathbb{E}_k\left[\cdot\right]$ denote $\mathbb{E}\left[\cdot\mid  \cF_k,P\right]$. Since $\pi_k$ is $\cF_k$-measurable, we have $\mathbb{E}_k\left[n_k(s,a,h)\right]=q_k(s,a,h)$.
Then we deduce 
\begin{align*}
\mathbb{E}_k\left[ Y_k^2\right]&=\sum_{(s,a),(s',a')\in\cS\times \cA}\frac{\mathbb{E}_k\left[(n_k(s,a,h)-q_k(s,a,h))(n_k(s',a',h)-q_k(s',a',h))\right]}{\max\left\{1,N_k(s,a,h)\right\}\cdot \max\left\{1,N_k(s',a',h)\right\}}\\
&=\sum_{(s,a),(s',a')\in\cS\times \cA}\frac{\mathbb{E}_k\left[n_k(s,a,h)n_k(s',a',h)-q_k(s,a,h)q_k(s',a',h)\right]}{\max\left\{1,N_k(s,a,h)\right\}\cdot \max\left\{1,N_k(s',a',h)\right\}}\\
&\leq\sum_{(s,a),(s',a')\in\cS\times \cA}\frac{\mathbb{E}_k\left[n_k(s,a,h)n_k(s',a',h)\right]}{\max\left\{1,N_k(s,a,h)\right\}\cdot \max\left\{1,N_k(s',a',h)\right\}}\\
&\leq\sum_{(s,a)\in\cS\times \cA}\frac{\mathbb{E}_k\left[n_k(s,a,h)\right]}{\max\left\{1,N_k(s,a,h)\right\}}\\
&=\sum_{(s,a)\in\cS\times \cA}\frac{q_k(s,a,h)}{\max\left\{1,N_k(s,a,h)\right\}}
\end{align*}
where the second equality holds because it follows from $\mathbb{E}_k\left[n_k(s,a,h)\right] = q_k(s,a,h)$ for $(s,a,h)\in\cS\times\cA\times[H]$ that $$\mathbb{E}_k\left[q_k(s,a,h)n_k(s',a',h)\right]=\mathbb{E}_k\left[q_k(s',a',h)n_k(s,a,h)\right]=q_k(s,a,h)q_k(s',a',h),$$
the second inequality holds because $n_k(s,a,h)n_k(s',a',h)=0$ if $(s,a)\neq (s',a')$, and the last equality holds true because $\mathbb{E}_k\left[n_k(s,a,h)\right] = q_k(s,a,h)$ for any $(s,a,h)\in\cS\times\cA\times[H]$. Then we may apply \Cref{bernstein2} with $\lambda=1/2$, and we deduce that with probability at least $1-\delta/H$,
$$\sum_{k=1}^KY_k\leq \frac{1}{2}\sum_{k=1}^K\sum_{(s,a)\in\cS\times \cA}\frac{q_k(s,a,h)}{\max\left\{1,N_k(s,a,h)\right\}} + 2\ln(H/\delta).$$
Plugging this inequality to~\eqref{mds:bound}, it follows that
$$\sum_{k=1}^K\sum_{(s,a)\in \cS\times \cA}\frac{q_k(s,a,h)}{\max\left\{1,N_k(s,a,h)\right\}}=2\sum_{k=1}^K\sum_{(s,a)\in \cS\times \cA}\frac{n_k(s,a,h)}{\max\left\{1,N_k(s,a,h)\right\}}+4\ln(H/\delta).$$
Here, the first term on the right-hand side can be bounded as follows. We have
\begin{align*}
    \sum_{k=1}^K\frac{n_k(s,a,h)}{\max\left\{1,N_k(s,a,h)\right\}}
    &=\sum_{k=1}^K\frac{n_k(s,a,h)}{\max\left\{1,N_{k+1}(s,a,h)\right\}} +\sum_{k=1}^K\left(\frac{n_k(s,a,h)}{\max\left\{1,N_{k}(s,a,h)\right\}}-\frac{n_k(s,a,h)}{\max\left\{1,N_{k+1}(s,a,h)\right\}}\right)\\
    &\leq\sum_{k=1}^K\frac{n_k(s,a,h)}{\max\left\{1,N_{k+1}(s,a,h)\right\}} +\sum_{k=1}^K\left(\frac{1}{\max\left\{1,N_{k}(s,a,h)\right\}}-\frac{1}{\max\left\{1,N_{k+1}(s,a,h)\right\}}\right)\\
    &\leq\sum_{k=1}^K\frac{n_k(s,a,h)}{\max\left\{1,N_{k+1}(s,a,h)\right\}} +1\\
    &\leq \ln K + 1.
    \end{align*}
    where the first inequality is due to $n_k(s,a,h)\leq 1$ and  the last inequality holds because
    $$n_k(s,a,h)= N_{k+1}(s,a,h) - N_k(s,a,h)\quad\text{and}\quad N_K(s,a,h)+n_K(s,a,h)\leq K.$$
Therefore, it follows that
\begin{align*}
\sum_{k=1}^K\sum_{(s,a)\in \cS\times \cA}\frac{n_k(s,a,h)}{\max\left\{1,N_k(s,a,h)\right\}}=\sum_{(s,a)\in \cS\times \cA}\sum_{k=1}^K\frac{n_k(s,a,h)}{\max\left\{1,N_k(s,a,h)\right\}}=SA\ln K+ SA.
\end{align*}
As a result, for any fixed $h\in [H]$,
$$\sum_{k=1}^K\sum_{(s,a)\in \cS\times \cA}\frac{q_k(s,a,h)}{\max\left\{1,N_k(s,a,h)\right\}}\leq 2SA\ln K + 2SA + 4\ln\left(H/\delta\right)$$
holds with probability at least $1-\delta/H$. By union bound,~\eqref{first-ineq} holds with probability at least $1-\delta$. 

Next, we will show that~\eqref{second-ineq} holds.
\begin{equation}\label{mds:bound'}
\sum_{k=1}^K\sum_{(s,a)\in \cS\times \cA}\frac{q_k(s,a,h)}{\sqrt{\max\left\{1,N_k(s,a,h)\right\}}}=\sum_{k=1}^K\sum_{(s,a)\in \cS\times \cA}\frac{n_k(s,a,h)}{\sqrt{\max\left\{1,N_k(s,a,h)\right\}}}+\sum_{k=1}^KZ_k
\end{equation}
where
$$Z_k=\sum_{(s,a)\in \cS\times \cA}\frac{-n_k(s,a,h)+q_k(s,a,h)}{\sqrt{\max\left\{1,N_k(s,a,h)\right\}}}.$$
As~$\mathbb{E}_k\left[n_k(s,a,h)\right] = q_k(s,a,h)$ holds for every $(s,a,h)\in\cS\times \cA\times [H]$, we know that $Z_1,\ldots, Z_K$ is a martingale difference sequence. We know that $Z_k\leq 1$ for each $k\in[K]$. Then we deduce 
\begin{align*}
\mathbb{E}_k\left[ Z_k^2\right]&\leq\sum_{(s,a),(s',a')\in\cS\times \cA}\frac{\mathbb{E}_k\left[n_k(s,a,h)n_k(s',a',h)\right]}{\sqrt{\max\left\{1,N_k(s,a,h)\right\}}\cdot \sqrt{\max\left\{1,N_k(s',a',h)\right\}}}\\
&=\sum_{(s,a)\in\cS\times \cA}\frac{\mathbb{E}_k\left[n_k(s,a,h)\right]}{\max\left\{1,N_k(s,a,h)\right\}}\\
&=\sum_{(s,a)\in\cS\times \cA}\frac{q_k(s,a,h)}{\max\left\{1,N_k(s,a,h)\right\}}
\end{align*}
where the first inequality is derived by the same argument when bounding $\mathbb{E}_k[Y_k^2]$,
the first equality holds because $n_k(s,a,h)n_k(s',a',h)=0$ if $(s,a)\neq (s',a')$, and the last equality holds true because $\mathbb{E}_k\left[n_k(s,a,h)\right] = q_k(s,a,h)$ for any $(s,a,h)\in\cS\times\cA\times[H]$. Then we may apply \Cref{bernstein2} with $\lambda=1$, and we deduce that with probability at least $1-\delta/H$,
$$\sum_{k=1}^KZ_k\leq \sum_{k=1}^K\sum_{(s,a)\in\cS\times \cA}\frac{q_k(s,a,h)}{\max\left\{1,N_k(s,a,h)\right\}} + \ln(H/\delta).$$
Then with probability at least $1-\delta$,~\eqref{first-ineq} holds and 
\begin{align}\label{mds:bound''}
\begin{aligned}
\sum_{h\in[H]}\sum_{k=1}^KZ_k&\leq \sum_{k=1}^K\sum_{(s,a,h)\in\cS\times \cA\times[H]}\frac{q_k(s,a,h)}{\max\left\{1,N_k(s,a,h)\right\}} + H\ln(H/\delta)\\
&=2HSA\ln K + 2HSA + 5H\ln(H/\delta).
\end{aligned}
\end{align}
holds.
Moreover, we have
\begin{align*}
    &\sum_{k=1}^K\frac{n_k(s,a,h)}{\sqrt{\max\left\{1,N_k(s,a,h)\right\}}}\\
    &=\sum_{k=1}^K\frac{n_k(s,a,h)}{\sqrt{\max\left\{1,N_{k+1}(s,a,h)\right\}}} +\sum_{k=1}^K\left(\frac{n_k(s,a,h)}{\sqrt{\max\left\{1,N_{k}(s,a,h)\right\}}}-\frac{n_k(s,a,h)}{\sqrt{\max\left\{1,N_{k+1}(s,a,h)\right\}}}\right)\\
    &\leq\sum_{k=1}^K\frac{n_k(s,a,h)}{\sqrt{\max\left\{1,N_{k+1}(s,a,h)\right\}}} +\sum_{k=1}^K\left(\frac{1}{\sqrt{\max\left\{1,N_{k}(s,a,h)\right\}}}-\frac{1}{\sqrt{\max\left\{1,N_{k+1}(s,a,h)\right\}}}\right)\\
    &\leq\sum_{k=1}^K\frac{n_k(s,a,h)}{\sqrt{\max\left\{1,N_{k+1}(s,a,h)\right\}}} +1\\
    &\leq 2\sqrt{N_{K+1}(s,a,h)} + 1.
    \end{align*}
    where the last equality holds because
    $n_k(s,a,h)= N_{k+1}(s,a,h) - N_k(s,a,h)$.
    Then
    \begin{align*}
        \sum_{k=1}^K\sum_{(s,a,h)\in\cS\times\cA\times[H]}\frac{n_k(s,a,h)}{\sqrt{\max\left\{1,N_k(s,a,h)\right\}}}
        &\leq \sum_{(s,a,h)\in\cS\times\cA\times[H]} 2\sqrt{N_{K+1}(s,a,h)} + HSA \\
        &\leq 2\sqrt{HSA \sum_{(s,a,h)}N_{K+1}(s,a,h)} + HSA\\
        &\leq 2H\sqrt{SAK} + HSA
    \end{align*}
    where the second equality is due to the Cauchy-Schwarz inequality. Then it follows from~\eqref{mds:bound'} and~\eqref{mds:bound''} that~\eqref{second-ineq} holds. 
\end{proof}

 Recall that the good event $\cE$ is the event that the statements of \Cref{lem:bounded,lemma:confidence,lemma:estimator,lemma8} hold. 
\begin{lemma}\label{lemma:good-event}
The good event $\cE$ holds with probability at least $1-14\delta$, i.e., $\mathbb{P}\left[\cE\right]\geq 1- 14\delta$. 
\end{lemma}
\begin{proof}
The proof follows from the union bound.
\end{proof}

\Cref{lemma:confidence} bounds the difference between the true transition kernel $P$ and the empirical transition kernel $\bar P_k$. Based on \Cref{lemma:confidence}, the next lemma bounds the difference between the true transition kernel and any $\widehat P$ contained in the confidence set $\cP_k$. \Cref{lemma:confidence'} is a modification of \citep[Lemma 8]{Jin2020} to our finite-horizon MDP setting.

\begin{lemma}\label{lemma:confidence'}
Under the good event $\cE$, we have
\begin{equation}\label{eq:epsilon-star}
\left|\widehat P(s'\mid s,a,h)-P(s'\mid s,a,h)\right|\leq \epsilon_k^\star(s'\mid s,a,h)
\end{equation}
where
\begin{equation*}
\epsilon_k^\star(s'\mid s,a,h)= 6\sqrt{\frac{P(s'\mid s,a,h)\ln\left({{HSAK}/{\delta}}\right)}{\max\{1,N_k(s,a,h)\}}}+94\frac{\ln\left({{HSAK}/{\delta}}\right)}{\max\{1,N_k(s,a,h)\}}
\end{equation*}
for every $\widehat P\in \cP_k$ and every $(s,a,s',h)\in\cS\times \cA\times \cS\times [H]$.
\end{lemma}

\begin{proof}
We follow the proof of \citep[Lemma B.13]{cohen2020}. Note that 
$$\max\{1,N_k(s,a,h)-1\}\geq \frac{1}{2}\cdot \max\{1,N_k(s,a,h)\}$$
holds for any value of $N_k(s,a,h)$. As we assumed that $P\in \cP_k$, we have that
$$
\bar P_k(s'\mid s,a,h)\leq P(s'\mid s,a,h)+\sqrt{\frac{8\bar P_k(s'\mid s,a,h)\ln\left({{HSAK}/{\delta}}\right)}{\max\{1,N_k(s,a,h)\}}}+\frac{28\ln\left({{HSAK}/{\delta}}\right)}{3\max\{1,N_k(s,a,h)\}}.$$
We may view this as a quadratic inequality in terms of $x=\sqrt{\bar P_k(s'\mid s,a,h)}$. Note that $x^2\leq ax + b +c$ for any $a,b,c\geq 0$ implies that $x\leq a +\sqrt{b}+\sqrt{c}$. Therefore, we deduce that
\begin{align*}
\sqrt{\bar P_k(s'\mid s,a,h)}&\leq \sqrt{P(s'\mid s,a,h)}+\left(2\sqrt{2} + \sqrt{\frac{28}{3}}\right)\sqrt{\frac{\ln\left({{HSAK}/{\delta}}\right)}{\max\{1,N_k(s,a,h)\}}}\\
&\leq \sqrt{P(s'\mid s,a,h)}+13\sqrt{\frac{\ln\left({{HSAK}/{\delta}}\right)}{\max\{1,N_k(s,a,h)\}}}.
\end{align*}
Using this bound on $\sqrt{\bar P_k(s'\mid s,a,h)}$, we obtain the following. 
\begin{align}\label{eq:lemma:confidence'1}
\begin{aligned}
\epsilon_k(s'\mid s,a,h)
&\leq \sqrt{\frac{8\bar P_k(s'\mid s,a,h)\ln\left({{HSAK}/{\delta}}\right)}{\max\{1,N_k(s,a,h)\}}}+\frac{28\ln\left({{HSAK}/{\delta}}\right)}{3\max\{1,N_k(s,a,h)\}}\\
&\leq \sqrt{\frac{8P(s'\mid s,a,h)\ln\left({{HSAK}/{\delta}}\right)}{\max\{1,N_k(s,a,h)\}}}+\left(13\sqrt{8}+\frac{28}{3}\right)\frac{\ln\left({{HSAK}/{\delta}}\right)}{\max\{1,N_k(s,a,h)\}}\\
&\leq3\sqrt{\frac{P(s'\mid s,a,h)\ln\left({{HSAK}/{\delta}}\right)}{\max\{1,N_k(s,a,h)\}}}+47\frac{\ln\left({{HSAK}/{\delta}}\right)}{\max\{1,N_k(s,a,h)\}}\\
&=\frac{1}{2}\cdot \epsilon_k^\star(s'\mid s,a,h)
\end{aligned}
\end{align}
Since we assumed that $P\in \cP_k$, 
$$\left|P(s'\mid s,a,h)-\bar P_k(s'\mid s,a,h)\right|\leq \frac{1}{2}\cdot\epsilon_k^\star(s'\mid s,a,h).$$
Moreover, for any $\widehat P\in\cP_k$, we have
$$\left|\widehat P(s'\mid s,a,h)-\bar P_k(s'\mid s,a,h)\right|\leq \epsilon_k(s'\mid s,a,h)\leq \frac{1}{2}\cdot\epsilon_k^\star(s'\mid s,a,h).$$
By the triangle inequality, it follows that
$$\left|\widehat P(s'\mid s,a,h)-P(s'\mid s,a,h)\right|\leq \epsilon_k^\star(s'\mid s,a,h),$$
as required.
\end{proof}

We note that the above lemma holds when we replace $P(s'\mid s,a,h)$ of $\epsilon_k^\star(s'\mid s,a,h)$ into $\widehat P(s'\mid s,a,h)$ for any $\widehat P \in \cP_k$. Specifically, under the good event $\cE$, we have for $(s,a,s',h) \in \cS\times\cA\times\cS\times[H]$,
\begin{equation}\label{eq:epsilon-star-2}
    \left|\widehat P(s'\mid s,a,h) - P(s'\mid s,a,h)\right| \leq 6\sqrt{\frac{\widehat P(s'\mid s,a,h)\ln(HSAK/\delta)}{\max\{1,N_k(s,a,h)\}}} + 94\frac{\ln(HSAK/\delta)}{\max\{1,N_k(s,a,h)\}}.
\end{equation}
It can be obtained by applying 
$$\bar P_k(s'\mid s,a,h) \leq \widehat P(s'\mid s,a,h) + \sqrt{\frac{8\bar P_k(s'\mid s,a,h)\ln(HSAK/\delta)}{\max\{1,N_k(s,a,h)\}}} + \frac{28\ln(HSAK/\delta)}{3\max\{1,N_k(s,a,h)\}}$$
with the same argument for the remaining part of the proof.

\section{Missing Proofs for Section \ref{sec:estimator}: Tighter Function Estimators}

In this section, we first show \Cref{lemma:variance-aware} that bounds the expected sum of the variance values of value function estimates. 

\begin{proof}[{\rm \bfseries Proof of \Cref{lemma:variance-aware}}]
Let $\pi_k$ be a policy for episode $k$. Moreover, let  $P_k\in\cP_k$, and let $g:\mathcal{S}\times\mathcal{A}\times[H]\to[0,1]$ be an arbitrary cost function. Then we may define the occupancy measure $\widehat q_k = q^{P_k,\pi_k}$ associated with policy $\pi_k$ and transitional kernel $P_k$. Then we know that $V_1^{\pi_k}(\mathbb{\widehat V}_k, P_k) = \langle \bm{\widehat q_k}, \bm{\mathbb{\widehat V}_k} \rangle$. Moreover, it follows from \Cref{lemma4} that
    $$\langle \bm{\widehat q_k}, \bm{\mathbb{\widehat V}_k} \rangle \leq \var\left[\langle \bm{\widehat n_k},\bm{g}\rangle\mid g,\pi_k, P_k\right]$$
    where $\bm{\widehat n_k}$ is a vector representation of $\widehat n_k=n^{P_k, \pi_k}.$ Furthermore, by \Cref{lemma2} with $B=1$, we have 
    \begin{align*}
    \var\left[\langle \bm{\widehat n_k},\bm{g}\rangle\mid g,\pi_k, P_k\right]&\leq \bbE[\langle \bm{\widehat n_k},\bm{g}\rangle^2 \mid g,\pi_k, P_k]\\
    &\leq 2\langle \bm{\widehat q_k}, \bm{\vec h}\odot \bm{g}\rangle\\
    &\leq 2H^2
    \end{align*}
    as desired.
\end{proof}

Having proved \Cref{lemma:variance-aware}, we are ready to prove \Cref{theorem:U_k} which is the crucial part of deducing our tighter function estimators. 

\begin{proof}[{\rm \bfseries Proof of \Cref{theorem:U_k}}]
We assume that the good event $\cE$ holds, which holds with probability at least $1-14\delta$ according to \Cref{lemma:good-event}. We observe that $\left|V_1^{\pi_k}(g,P) - V_1^{\pi_k}(g,P_k)\right|$ can be rewritten by $|\langle\bm{g},\bm{q_k}-\bm{\widehat q_k}\rangle|$ using occupancy measures. By \Cref{lemma:confidence''}, it follows that
\begin{align*}
\left|\langle \bm{g}, \bm{q_k} - \bm{\widehat q_k}\rangle\right|&=\left|\sum_{(s,a,s',h)\in\cS\times\cA\times\cS\times[H]} \widehat q_k(s,a,h) \left(P - P_k\right)(s'\mid s,a,h) V_{h+1}^{\pi_k}(s'; g, P)\right|\\
&\leq \underbrace{\left|\sum_{(s,a,s',h)\in\cS\times\cA\times\cS\times[H]} \widehat q_k(s,a,h) (P-P_k)(s'\mid s,a,h) V_{h+1}^{\pi_k}(s'; g, P_k)\right|}_{\text{Term 1}}\\
&\quad+\underbrace{\left|\sum_{(s,a,s',h)\in\cS\times\cA\times\cS\times[H]} \widehat q_k(s,a,h)(P - P_k)(s'\mid s,a,h) \left(V_{h+1}^{\pi_k}(s'; g, P) - V_{h+1}^{\pi_k}(s'; g, P_k)\right) \right|}_{\text{Term 2}}
\end{align*}
where $(P-P_k)(s'\mid s,a,h)=P(s'\mid s,a,h)-P_k(s'\mid s,a,h)$. 

To bound Term 2, we use bound 
\[
P(s'\mid s,a,h)-P_k(s'\mid s,a,h) \leq 6\sqrt{\frac{P_k(s'\mid s,a,h)\ln(HSAK/\delta)}{\max\{1,N_k(s,a,h)\}}} + 94\frac{\ln(HSAK/\delta)}{\max\{1,N_k(s,a,h)\}}
\]
as explained in \eqref{eq:epsilon-star-2}. This is because $\widehat q_k = q^{P_k, \pi_k}$ is an occupancy measure with respect to $P_k\in \cP_k$, not $P$. Then we can apply \Cref{lemma10} and obtain
$$\text{Term 2} \leq 10^4H^2S^2\left(\ln\frac{HSAK}{\delta}\right)^2 \sum_{(s,a,h)}\frac{\widehat q_k(s,a,h)}{\max\{1, N_k(s,a,h)\}}.$$

Next, we bound Term 1. Note that $\sum_{s'}\left(P(s'\mid s,a,h) - P_k(s'\mid s,a,h)\right)=0.$ Then it follows that
\begin{align*}
\text{Term 1}&= \left|\sum_{(s,a,s',h)} \widehat q_k(s,a,h)(P - P_k)(s'\mid s,a,h) (V_{h+1}^{\pi_k}(g, P_k) - \widehat \mu_k(s,a,h))\right|\\
&\leq 2\sum_{(s,a,s',h)} \widehat q_k(s,a,h)\epsilon_k(s'\mid s,a,h) \left|V_{h+1}^{\pi_k}(g, P_k) - \widehat \mu_k(s,a,h)\right|\\
&= 4\underbrace{\sum_{(s,a,s',h)} \widehat q_k(s,a,h)\sqrt{\frac{\bar P_k(s'\mid s,a,h)\ln(HSAK/\delta)}{\max\{1,N_k(s,a,h)-1\}}} \left|V_{h+1}^{\pi_k}(s'; g, P_k) - \widehat \mu_k(s,a,h)\right|}_{\text{Term 3}}\\
&\quad+\frac{28}{3}\underbrace{\sum_{(s,a,s',h)} \widehat q_k(s,a,h)\frac{\ln(HSAK/\delta)}{\max\{1,N_k(s,a,h)-1\}} \left|V_{h+1}^{\pi_k}(s'; g, P) - \widehat \mu_k(s,a,h)\right|}_{\text{Term 4}}
\end{align*}
where $\widehat \mu_k(s,a,h)=\bbE_{s'\sim P_k(\cdot\mid s,a,h)}[V_{h+1}^{\pi_k}(s';g, P_k)].$ The first inequality is from $|(P-P_k)(s'\mid s,a,h)|\leq |(P-\bar P_k)(s'\mid s,a,h)| + |(\bar P_k-P_k)(s'\mid s,a,h)| \leq 2 \epsilon_k(s'\mid s,a,h)$ for any $(s,a,s',h) \in \cS\times\cA\times\cS\times[H]$ under the good event $\cE$. We note that $\bar P_k(s'\mid s,a,h) \leq P_k(s'\mid s,a,h) + \epsilon_k(s'\mid s,a,h)$ and define 
$$\widehat{\mathbb{V}}_k(s,a,h)=\sum_{s'} P_k(s'\mid s,a,h) \left|V_{h+1}^{\pi_k}(s';g, P_k)-\widehat\mu_k(s,a,h)\right|^2.$$

Then we can bound Term 3 as the following.
\begin{align*}
&\frac{\text{Term 3}}{\sqrt{\ln(HSAK/\delta)}} \\
&\leq \sum_{(s,a,s',h)} \widehat q_k(s,a,h)\sqrt{\frac{(P_k+\epsilon_k)(s'\mid s,a,h)}{\max\{1,N_k(s,a,h)-1\}}}\left|V_{h+1}^{\pi_k}(s';g, P_k)-\widehat \mu_k(s,a,h)\right|\\
&\leq \sqrt{\sum_{(s,a,s',h)}\widehat q_k(s,a,h)(P_k + \epsilon_k)(s'\mid s,a,h)\left|V_{h+1}^{\pi_k}(s';g, P_k)-\widehat \mu_k(s,a,h)\right|^2}\sqrt{\sum_{(s,a,s',h)} \frac{\widehat q_k(s,a,h)}{\max\{1, N_k(s,a,h)-1\}}}\\
&\leq\sqrt{\sum_{(s,a,h)}\widehat q_k(s,a,h)\widehat{\mathbb{V}}_k(s,a,h) + 4H^2\sum_{(s,a,s',h)}\widehat q_k(s,a,h)\epsilon_k(s'\mid s,a,h)}\sqrt{\sum_{(s,a,s',h)}\frac{\widehat q_k(s,a,h)}{\max\{1, N_k(s,a,h)-1\}}}
\end{align*}
where the second inequality follows from the Cauchy-Schwarz inequality and the last inequality is due to $\left|V_{h+1}^{\pi_k}(s';g, P_k)-\widehat \mu_k(s,a,h)\right|\leq 2H$. 

By \Cref{lemma:variance-aware}, we deduce that
$$\sum_{(s,a,h)}\widehat q_k(s,a,h)\widehat{\mathbb{V}}_k(s,a,h) \leq 2H^2.$$
Due to AM-GM inequality, we have
\begin{align*}
&\sqrt{2H^2 + 4H^2\sum_{(s,a,s',h)}\widehat q_k(s,a,h)\epsilon_k(s'\mid s,a,h)}\sqrt{\sum_{(s,a,s',h)} \frac{\widehat q_k(s,a,h)}{\max\{1, N_k(s,a,h)-1\}}}\\
&\leq\left(\sqrt{2H^2} + \sqrt{4H^2\sum_{(s,a,s',h)}\widehat q_k(s,a,h)\epsilon_k(s'\mid s,a,h)}\right)\sqrt{\sum_{(s,a,s',h)} \frac{\widehat q_k(s,a,h)}{\max\{1, N_k(s,a,h)-1\}}}\\
&\leq \frac{H^2}{\alpha_1} + \frac{2H^2}{\alpha_2} \sum_{(s,a,s',h)} \widehat q_k(s,a,h)\epsilon_k(s'\mid s,a,h) + \frac{\alpha_1 + \alpha_2}{2}\sum_{(s,a,h)} \frac{S\cdot\widehat q_k(s,a,h)}{\max\{1, N_k(s,a,h)-1\}}
\end{align*}
for any $\alpha_1, \alpha_2 > 0.$ By taking $\alpha_1 = \frac{\sqrt{HK\ln(HSAK/\delta)}}{S\sqrt{A}}, \ \alpha_2 = \sqrt{H^3\ln(HSAK/\delta)}$, we obtain
\begin{align*}
\text{Term 3}&\leq \sum_{(s,a,h)} \widehat q_k(s,a,h)\left( \frac{S\sqrt{HA}}{\sqrt{K}} + 2\sqrt{H}\sum_{s'}\epsilon_k(s'\mid s,a,h) + \frac{\sqrt{HK}+\sqrt{H^3S^2A}}{2\sqrt{A}}\frac{\ln(HSAK/\delta)}{\max\{1, N_k(s,a,h)-1\}}\right)\\
&\leq \sum_{(s,a,h)} \widehat q_k(s,a,h)\left(\frac{S\sqrt{HA}}{\sqrt{K}} + 2\sqrt{H}\varepsilon_k(s,a,h) + \frac{\sqrt{HK}+\sqrt{H^3S^2A}}{2\sqrt{A}}\frac{\ln(HSAK/\delta)}{\max\{1, N_k(s,a,h)-1\}}\right).
\end{align*}
Note that the last inequality follows from
\begin{align*}
    \sum_{s'}\epsilon_k(s'\mid s,a,h) 
    &=\sum_{s'}\left(\sqrt{\frac{4\bar P_k(s'\mid s,a,h)\ln(HSAK/\delta)}{\max\{1,N_k(s,a,h)-1\}}} + \frac{14\ln(HSAK/\delta)}{3\max\{1,N_k(s,a,h)-1\}}\right)\\
    &\leq\sqrt{S}\sqrt{\frac{4\sum_{s'} \bar P_k(s'\mid s,a,h)\ln(HSAK/\delta)}{\max\{1,N_k(s,a,h)-1\}}} + \frac{14S\ln(HSAK/\delta)}{3\max\{1,N_k(s,a,h)-1\}}\\
    &=\sqrt{\frac{4S\ln(HSAK/\delta)}{\max\{1,N_k(s,a,h)-1\}}} + \frac{14S\ln(HSAK/\delta)}{3\max\{1,N_k(s,a,h)-1\}}\\
    &=\varepsilon_k(s,a,h)
\end{align*}
where the inequality is due to the Cauchy-Schwarz inequality and the second equality is due to $\sum_{s'}\bar P_k(s'\mid s,a,h) \leq 1$.

Since $\left|V_{h+1}^{\pi_k}(s'; g, P) - \widehat \mu_k(s,a,h)\right|\leq 2H$, Term 4 can be bounded as follows.
$$\text{Term 4}\leq 2HS\ln(HSAK/\delta) \sum_{(s,a,h)}\frac{\widehat q_k(s,a,h)}{\max\{1,N_k(s,a,h)-1\}}.$$

Finally, we proved that
\begin{align*}
&\left|\langle \bm{g}, \bm{q_k} - \bm{\widehat q_k}\rangle\right|\\
&\leq 4\cdot \text{(Term 3)} + \frac{28}{3} \cdot \text{(Term 4)} + \text{(Term 2)}\\
&\leq \sum_{(s,a,h)} \widehat q_k(s,a,h)\left(\frac{4S\sqrt{HA}}{\sqrt{K}} + 8\sqrt{H}\varepsilon_k(s,a,h) + \frac{2\sqrt{HK}\ln(HSAK/\delta)}{\sqrt{A}\max\{1, N_k(s,a,h)-1\}}\right)\\
&\quad+\left(\left(\frac{56}{3}HS+2H^{1.5}S\right)\ln(HSAK/\delta)+ 10^4H^2S^2(\ln(HSAK/\delta))^2\right) \sum_{(s,a,h)}\frac{\widehat q_k(s,a,h)}{\max\{1,N_k(s,a,h)-1\}}
\end{align*}
as required.
\end{proof}

\section{Missing Proofs for Section \ref{sec:alg}: Safe Exploration}

In this section, we prove \Cref{lemma:K0} that provides an asymptotic upper bound on a sufficient number of episodes executing $\pi_b$, which is denoted by $K_0$, for feasibility of \eqref{eq:lp}.

\begin{lemma}\label{lemma:sum-U_k}
Assume that the good event $\cE$ holds. Let $\pi_k$ be any policy for episode $k$, and let $P$ be the true transition kernel. Let $q_k$ denote the occupancy measure $q^{P,\pi_k}$ associated with $\pi_k$ and $P$. For $R_k, U_k$, we have
$$\sum_{k=1}^K \langle \bm{R_k}+\bm{U_k}, \bm{q_k}\rangle=\cO\left(\left(H^{1.5}S\sqrt{AK}+H^3S^3A\right)\left(\ln\frac{HSAK}{\delta}\right)^3\right).$$
\end{lemma}
\begin{proof}

Note that $\sum_{k=1}^K \langle \bm{R_k} + \bm{U_k}, \bm{q_k} \rangle$ can be rewritten as
\begin{align*}
    &\sum_{k=1}^K \langle\bm{R_k} + \bm{U_k}, \bm{q_k}\rangle \\
    &= \sum_{k=1}^K \sum_{(s,a,h)} q_k(s,a,h) \sqrt{\frac{\ln(HSAK/\delta)}{\max\{1, N_k(s,a,h)\}}} \\
    &\quad+ \sum_{k=1}^K \sum_{(s,a,h)} q_k(s,a,h)\left(\frac{4S\sqrt{HA}}{\sqrt{K}} + 8\sqrt{H}\varepsilon_k(s,a,h)+ \frac{2(\sqrt{HK}+\sqrt{H^3S^2A})\ln(HSAK/\delta)}{\sqrt{A}\max\{1, N_k(s,a,h)-1\}}\right)\\
    &\quad+\left(\frac{56}{3}HS\ln(HSAK/\delta) + 10^4H^2S^2(\ln(HSAK/\delta))^2\right) \sum_{k=1}^K \sum_{(s,a,h)}\frac{q_k(s,a,h)}{\max\{1,N_k(s,a,h)-1\}}.
\end{align*}

Since $\sum_{(s,a,h)} \widehat q_k(s,a,h) = H$, we have 
$$\sum_{k=1}^K \sum_{(s,a,h)} q_k(s,a,h)\cdot\frac{4S\sqrt{HA}}{\sqrt{K}} = \cO(H^{1.5}S\sqrt{AK}).$$

Furthermore, \Cref{lemma8} implies that
\begin{align*}
    &\sum_{k=1}^K \sum_{(s,a,h)} \frac{q_k(s,a,h)}{\max\{1, N_k(s,a,h)\}} = \cO(HSA\ln K + H\ln(H/\delta)),\\
    &\sum_{k=1}^K \sum_{(s,a,h)} \frac{q_k(s,a,h)}{\sqrt{\max\{1, N_k(s,a,h)\}}} = \cO(H\sqrt{SAK} + HSA\ln K + H\ln(H/\delta)).
\end{align*}
Then it follows that
\begin{align*}
    &\sum_{k=1}^K \sum_{(s,a,h)} q_k(s,a,h) \sqrt{\frac{\ln(HSAK/\delta)}{\max\{1, N_k(s,a,h)\}}} = \cO\left((H\sqrt{SAK} + HSA)\left(\ln\frac{HSAK}{\delta}\right)^2 \right).
\end{align*}
Since $\max\{1, N_k(s,a,h)-1\} \geq \frac{1}{2}\max\{1, N_k(s,a,h)\}$, we have
\begin{align*}
    \sum_{k=1}^K \sum_{(s,a,h)} q_k(s,a,h) \frac{(\sqrt{HK}+\sqrt{H^3S^2A})\ln(HSAK/\delta)}{\sqrt{A}\max\{1, N_k(s,a,h)-1\}} = \cO\left((H^{1.5}S\sqrt{AK} + H^{2.5}S^2A)\left(\ln\frac{HSAK}{\delta}\right)^2\right),
\end{align*}
and moreover,
\begin{align*}
\left(HS\ln(HSAK/\delta) + H^2S^2(\ln(HSAK/\delta))^2\right) \sum_{k=1}^K \sum_{(s,a,h)}\frac{q_k(s,a,h)}{\max\{1,N_k(s,a,h)-1\}}=\cO\left(H^3S^3A\left(\ln\frac{HSAK}{\delta}\right)^3\right).
\end{align*}

Next, by \Cref{lemma8}, $\sum_{k=1}^K \sum_{(s,a,h)} q_k(s,a,h)  \left(\sqrt{H}\varepsilon_k(s,a,h)\right)$ can be bounded as follows. 
\begin{align*}
&\sum_{k=1}^K \sum_{(s,a,h)} q_k(s,a,h)  \left(\sqrt{H}\varepsilon_k(s,a,h)\right)\\
&=\sqrt{H}\sum_{k=1}^K \sum_{(s,a,h)} q_k(s,a,h) \left(\sqrt{\frac{4S\ln(HSAK/\delta)}{\max\{1,N_k(s,a,h)-1\}}} + \frac{14S\ln(HSAK/\delta)}{3\max\{1,N_k(s,a,h)-1\}}\right)\\
&={\cO}\left(\left(H^{1.5}S\sqrt{AK} + H^{1.5}S^2A\right)\left(\ln\frac{HSAK}{\delta}\right)^2\right).
\end{align*}

As a result, we have proved that
$$\sum_{k=1}^K \langle \bm{R_k}+\bm{U_k}, \bm{q_k}\rangle = \cO\left((H^{1.5}S\sqrt{AK} + H^3S^3A)\left(\ln\frac{HSAK}{\delta}\right)^3\right),$$
as required.
\end{proof}

We are ready to prove \Cref{lemma:K0} based on \Cref{lemma:sum-U_k}.
\begin{proof}[\textbf{Proof of \Cref{lemma:K0}}]
We closely follow the proof of \citep[Proposition 4]{bura2022dope}. We assume that the good event $\cE$ holds, which holds with probability at least $1-14\delta$. Let $q_b = q^{P,\pi_b}$ be the occupancy measure associated with the safe baseline policy $\pi_b$ and the true transition kernel $P$. Then $q_b$ is a feasible solution of \eqref{eq:lp-occupancy} if $\langle \bm{\widehat g_k}, \bm{q_b} \rangle \leq \bar{C}$ holds. To find a sufficient condition, we deduce that
\begin{align*}
\langle \bm{\widehat g_k}, \bm{q_b}\rangle &= \langle \bm{\bar g_k}+\bm{R_k}+\bm{U_k},\bm{q_b}\rangle\\
&\leq\langle \bm{g}+2\bm{R_k}+\bm{U_k}, \bm{q_b}\rangle\\
&= \bar{C}_b + \langle 2\bm{R_k}+\bm{U_k}, \bm{q_b}\rangle
\end{align*}
where the first equality is from the definition of ${\widehat g_k}$, the inequality is from \Cref{lemma:estimator}, and the last equality follows from $\langle \bm{g}, \bm{q_b} \rangle=\bar{C}_b$. This implies that a sufficient condition for $\langle \bm{\widehat g_k}, \bm{q_b} \rangle \leq \bar{C}$ is given by
\begin{equation}\label{eq:sufficient-condition}
\langle 2\bm{R_k}+\bm{U_k}, \bm{q_b}\rangle < \bar{C} - \bar{C}_b.
\end{equation}
Note that $\langle 2\bm{R_k}+\bm{U_k},\bm{q_b} \rangle$ decreases as $k$ increases because $$\frac{1}{\max\{1,N_k(s,a,h)\}},\quad  \frac{1}{\sqrt{\max\{1,N_k(s,a,h)\}}}
$$ 
can only decrease as $k$ increases. Then suppose that $K_0$ is the last episode where \eqref{eq:sufficient-condition} 
does not hold. By definition, $K_0+1$ is the first episode satisfying $\langle\bm{\widehat g_k},\bm{q_b}\rangle < \bar C$. Due to the strict inequality, occupancy measures other than $q_b$ can be potentially feasible to \eqref{eq:lp-occupancy}. This implies that DOPE+ can sufficiently explore policies other than $\pi_b$ after $K_0$ episodes. Then we have
$$K_0(\bar{C} - \bar{C}_b)<\sum_{k=1}^{K_0}\langle 2\bm{R_k}+\bm{U_k}, \bm{q_b}\rangle.$$
Since $q_b$ induces the true transition kernel, we can apply \Cref{lemma:sum-U_k}. Then the right-hand side is bounded as follows.
$$\sum_{k=1}^{K_0}\langle 2\bm{R_k}+\bm{U_k}, \bm{q_b}\rangle=\widetilde{\cO}\left(H^{1.5}S\sqrt{AK_0}\right).$$
Hence, $K_0$ satisfies
$$K_0 =\widetilde{\cO}\left(\frac{H^{3}S^2A}{(\bar{C}-\bar{C}_b)^2}\right).$$
Then we have
$$\langle 2\bm{R_{k}}+\bm{U_{k}}, \bm{q_b}\rangle \leq \langle 2\bm{R_{K_0+1}}+\bm{U_{K_0+1}}, \bm{q_b}\rangle \leq\bar C - \bar C_b \quad\forall k = K_0+1, \ldots, K.$$
This implies that \eqref{eq:lp} is feasible after episode $K_0$ when $(\pi_b,P)$ becomes a feasible solution in episode $K_0$.
\end{proof}

\section{Detailed Proofs for the Regret Analysis}\label{sec:appendix:details-of-the-analysis}

In this section, we prove \Cref{theorem:violation-hard} that guarantees zero constraint violation for DOPE+. Next, we provide the proofs of Lemmas \ref{lemma:term-1}, \ref{lemma:term-2} and \ref{lemma:term-3}. Lastly, we show \Cref{theorem:regret} that gives us the regret upper bound.

\subsection{Details of Constraint Violation Analysis}\label{sec:appendix:details-of-the-analysis:violation-analysis}

\begin{proof}[\textbf{Proof of \Cref{theorem:violation-hard}}] 
We assume that the good event $\cE$ holds, which is the case with probability at least $1-14\delta$. Let $\pi_k, P_k$ denote the policy and the transition kernel obtained from DOPE+ for episode $k$, respectively. Let $q_k = q^{P,\pi_k}, \widehat q_k = q^{P_k, \pi_k}$. We know that the constraint is satisfied if
$V_1^{\pi_k}(g,P)=\langle \bm{g}, \bm{q_k} \rangle\leq\bar{C}$ for each $k\in[K]$. For $k\leq K_0$, there is no constraint violation because we take $\pi_k = \pi_b$. Now we consider the case when $k>K_0$. We have
\begin{align*}
\langle \bm{g}, \bm{q_k}\rangle &= \langle \bm{g}, \bm{\widehat q_k} \rangle + \langle \bm{g}, \bm{q_k} - \bm{\widehat q_k}\rangle\\
&\leq\langle \bm{\bar{g}_k}+\bm{R_k}, \bm{\widehat q_k}\rangle + \langle \bm{g}, \bm{q_k} - \bm{\widehat q_k}\rangle\\
&\leq\langle \bm{\bar{g}_k}+\bm{R_k}, \bm{\widehat q_k}\rangle + \langle \bm{U_k}, \bm{\widehat q_k}\rangle\\
&=\langle \bm{\widehat g_k}, \bm{\widehat q_k}\rangle\\
&\leq \bar{C}
\end{align*}
where the first inequality follows from \Cref{lemma:estimator}, the second inequality is from \Cref{theorem:U_k}, and the last inequality is due to the update rule of DOPE+. This implies that $\pi_k$ holds $\langle \bm{g},\bm{q_k}\rangle\leq \bar C$ for $k > K_0$. Thus, we showed that $\text{Violation}(\vec\pi)=0$ with probability at least $1-14\delta$.
\end{proof}

\subsection{Details of Regret Analysis}\label{sec:appendix:details-of-the-analysis:regret-analysis}

\begin{proof}[\textbf{Proof of \Cref{lemma:term-1}}] 
We closely follow the proof of \citep[Lemma 18]{bura2022dope}. We assume that the good event $\cE$ holds, which is the case with probability at least $1-14\delta$. We observe that 
$$\sum_{k=K_0+1}^K \left(V^{\pi^*}_1(f,P)- V^{\pi_k}_1(\widehat f_k, P_k)\right)={\sum_{k=K_0+1}^K \langle \bm{f}, \bm{q^*}\rangle-\sum_{k=K_0+1}^K \langle \bm{\widehat f_k}, \bm{\widehat q_k}\rangle}.$$ 
By \Cref{lemma:valid-occupancy}, there exist $\bar q_b(s,a,s',h)$ and $\bar q^*(s,a,s',h)$ such that $q_b(s,a,h) = \sum_{s'\in\cS} \bar q_b(s,a,s',h)$ and $\ q^*(s,a,h) = \sum_{s'\in\cS} \bar q^*(s,a,s',h)$, respectively.
Then we define the new occupancy measure $q_{\alpha_k}(s,a,h)$ satisfying $q_{\alpha_k}(s,a,h) = \sum_{s'\in \cS} \bar q_{\alpha_k}(s,a,s',h)$ 
where
\begin{equation}\label{eq:convex-combi-occupancy}
\bar q_{\alpha_k}(s,a,s',h)=(1-\alpha_k)\bar q_b(s,a,s',h)+\alpha_k\bar q^*(s,a,s',h)    
\end{equation}
for $(s,a,s',h)\in\cS\times\cA\times\cS\times[H]$ and $\alpha_k\in [0,1]$. Now we verify (C1),(C2) and (C3) in \Cref{lemma:valid-occupancy} to say $q_{\alpha_k}$ is a valid occupancy measure. Since $\bar q_{\alpha_k}$ is a convex combination of $\bar q_b$ and $\bar q^*$, (C1),(C2) hold. For (C3), we can show that $q_{\alpha_k}$ induces the true transition kernel $P$ as follows. Since we know $q_b$ and $q^*$ induce $P$, it follows that $\bar q_b(s,a,s',h)=P(s'\mid s,a,h) \sum_{s''\in \cS}\bar q_b(s,a,s'',h)$ and $\bar q^*(s,a,s',h)=P(s'\mid s,a,h) \sum_{s''\in \cS}\bar q^*(s,a,s'',h)$ for $(s,a,s',h)\in\cS\times\cA\times\cS\times[H]$. Then $\bar q_{\alpha_k}(s,a,s',h)=P(s'\mid s,a,h) \sum_{s''\in\cS} \bar q_{\alpha_k}(s,a,s'',h)$ can be derived from \eqref{eq:convex-combi-occupancy}, which implies that $q_{\alpha_k}$ induces the true transition kernel $P$. Hence, $q_{\alpha_k}$ is a valid occupancy measure inducing the true transition kernel $P$. 

To use the optimality of $\widehat q_k$ in our analysis, we expect that $q_{\alpha_k}$ is a feasible solution for \eqref{eq:lp-occupancy}. Under the good event $\cE$, we know that $q_{\alpha_k}\in\Delta(P,k)$ due to $P\in \cP_k$. Then it is sufficient to find a condition for $\alpha_k$ satisfying $\langle\bm{\widehat g_k}, \bm{q_{\alpha_k}} \rangle \leq \bar{C}$. We deduce that
\begin{align*}
\langle\bm{\widehat g_k}, \bm{q_{\alpha_k}} \rangle
&=\langle \bm{\bar g_k} + \bm{R_k} + \bm{U_k}, \bm{q_{\alpha_k}}\rangle \\
&\leq\langle \bm{g} + 2\bm{R_k} + \bm{U_k}, \bm{q_{\alpha_k}}\rangle \\
&= (1-\alpha_k)\langle \bm{g} + 2\bm{R_k} + \bm{U_k}, \bm{q_b}\rangle + \alpha_k\langle \bm{g} + 2\bm{R_k} + \bm{U_k}, \bm{q^*}\rangle\\
&\leq (1-\alpha_k)(\bar{C}_b + \langle \bm{2R_k}+\bm{U_k}, \bm{q_b}\rangle) + \alpha_k(\bar{C} + \langle \bm{2R_k}+\bm{U_k}, \bm{q^*}\rangle)
\end{align*}
where the first inequality is from \Cref{lemma:estimator} and the last inequality is from $\langle \bm{g}, \bm{q_b} \rangle=\bar{C}_b$ and $\langle \bm{g},\bm{q^*}\rangle\leq\bar{C}$. Furthermore, the second equality is true because \eqref{eq:convex-combi-occupancy} implies that $q_{\alpha_k}(s,a,h)=(1-\alpha_k)q_b(s,a,h)+\alpha_k q^*(s,a,h)$. Hence, a sufficient condition of $\alpha_k$ for $\langle\bm{\widehat g_k}, \bm{q_{\alpha_k}} \rangle \leq \bar{C}$ is given by
\begin{equation*}
\alpha_k \leq \frac{\bar{C} - \bar{C}_b - \langle 2\bm{R_k} + \bm{U_k}, \bm{q_b} \rangle}{\bar{C}-\bar{C}_b + \langle 2\bm{R_k}+\bm{U_k}, \bm{q^*}\rangle - \langle 2\bm{R_k}+\bm{U_k}, \bm{q_b}\rangle}.    
\end{equation*}
Remember that, in the proof of \Cref{lemma:K0}, we defined $K_0$ so that $K_0+1$ is the first episode satisfying $\langle 2\bm{R_k}+\bm{U_k}, \bm{q_b} \rangle \leq \bar C - \bar C_b$. This guarantees that there exists some $\alpha_k\in[0,1]$ satisfying the above inequality for $k > K_0$.

Now, for some $\alpha_k$, we claim that
\begin{equation}\label{eq:claim}
    \langle \bm{f}, \bm{q^*}\rangle \leq \langle \bm{\bar f_k}+\frac{3H}{\bar{C}-\bar{C}_b}\bm{R_k} + \frac{H}{\bar{C} -\bar{C}_b}\bm{U_k}, \bm{q_{\alpha_k}}\rangle.
\end{equation}

To show \eqref{eq:claim}, we first take for $\beta\geq 1$,
$$\bm{f_\beta}=\bm{\bar f_k}+3\beta\bm{R_k} + \beta\bm{U_k}.$$
Then we find $\alpha_k,\beta$ satisfying $\langle \bm{f},\bm{q^*} \rangle \leq \langle \bm{f_\beta}, \bm{q_{\alpha_k}} \rangle$. 
By \Cref{lemma:estimator}, we have 
\begin{align*}
\langle \bm{f_\beta}, \bm{q_{\alpha_k}} \rangle &= \langle \bm{\bar f_k}+3\beta\bm{R_k} + \beta\bm{U_k}, \bm{q_{\alpha_k}}\rangle\\
&\geq\langle \bm{f}+2\beta\bm{R_k} + \beta\bm{U_k}, \bm{q_{\alpha_k}}\rangle\\
&=(1-\alpha_k)\langle \bm{f}+2\beta\bm{R_k} + \beta\bm{U_k}, \bm{q_b}\rangle+\alpha_k\langle \bm{f}+2\beta\bm{R_k} + \beta\bm{U_k}, \bm{q^*}\rangle.
\end{align*}
We have $\langle \bm{f},\bm{q^*} \rangle \leq \langle \bm{f_\beta}, \bm{q_{\alpha_k}} \rangle$ if $\beta$ satisfies
\begin{align*}
\beta \geq \frac{(1-\alpha_k)(\langle \bm{f}, \bm{q^*}\rangle - \langle \bm{f}, \bm{q_b}\rangle)}{(1-\alpha_k)\langle 2\bm{R_k}+\bm{U_k}, \bm{q_b}\rangle + \alpha_k \langle 2\bm{R_k}+\bm{U_k}, \bm{q^*}\rangle}.
\end{align*}

By taking 
\begin{equation}\label{eq:alpha}
\alpha_k =\frac{\bar{C} - \bar{C}_b - \langle 2\bm{R_k} + \bm{U_k},\bm{q_b}\rangle}{\bar{C}-\bar{C}_b + \langle 2\bm{R_k}+\bm{U_k},\bm{q^*}\rangle - \langle 2\bm{R_k}+\bm{U_k},\bm{q_b}\rangle},    
\end{equation}
it follows that
\begin{align*}
\beta\geq\frac{\langle \bm{f}, \bm{q^*}\rangle - \langle \bm{f}, \bm{q_b}\rangle}{\bar{C}-\bar{C}_b}.
\end{align*}

Since $\langle \bm{f}, \bm{q^*}\rangle - \langle \bm{f}, \bm{q_b}\rangle\leq H$, it is sufficient to take 
\begin{equation}\label{eq:beta}
\beta = \frac{H}{\bar{C}-\bar{C}_b}.    
\end{equation}

For $\alpha_k$ satisfying \eqref{eq:alpha}, we showed that $q_{\alpha_k}$ is a feasible solution for \eqref{eq:lp-occupancy}. Then it follows $\langle \bm{\widehat f_k}, \bm{q_{\alpha_k}}\rangle \leq \langle \bm{\widehat f_k}, \bm{\widehat q_k} \rangle$ due to optimality of $\widehat q_k$. Furthermore, for $\beta$ satisfying \eqref{eq:beta}, we have \eqref{eq:claim}.
Hence, we deduce that
\begin{align*}
\langle \bm{f},\bm{q^*} \rangle - \langle \bm{\widehat f_k}, \bm{\widehat q_k} \rangle &\leq \langle \bm{f},\bm{q^*} \rangle - \langle \bm{\widehat f_k}, \bm{q_{\alpha_k}} \rangle\\
&=\langle \bm{f},\bm{q^*} \rangle - \langle \bm{\bar f_k}+\frac{3H}{\bar{C}-\bar{C}_b}\bm{R_k} + \frac{H}{\bar{C} -\bar{C}_b}\bm{U_k}, \bm{q_{\alpha_k}}\rangle\\
&\quad+\langle \bm{\bar f_k}+\frac{3H}{\bar{C}-\bar{C}_b}\bm{R_k} + \frac{H}{\bar{C} -\bar{C}_b}\bm{U_k}, \bm{q_{\alpha_k}}\rangle-\langle \bm{\vec{B}}\wedge(\bm{\bar f_k}+\frac{3H}{\bar{C}-\bar{C}_b}\bm{R_k} + \frac{H}{\bar{C} -\bar{C}_b}\bm{U_k}), \bm{q_{\alpha_k}}\rangle\\
&\leq\langle \bm{\bar f_k}+\frac{3H}{\bar{C}-\bar{C}_b}\bm{R_k} + \frac{H}{\bar{C} -\bar{C}_b}\bm{U_k}, \bm{q_{\alpha_k}}\rangle-\langle \bm{\vec{B}}\wedge(\bm{\bar f_k}+\frac{3H}{\bar{C}-\bar{C}_b}\bm{R_k} + \frac{H}{\bar{C} -\bar{C}_b}\bm{U_k}), \bm{q_{\alpha_k}}\rangle
\end{align*}
where the last inequality is from \eqref{eq:claim}. Furthermore, under the good event $\cE$, we know that $f_k(s,a,h) \leq B$ for $(s,a,h)\in \cS\times\cA\times[H]$ and $k\in[K]$, where $B=1+\sqrt{\ln(HSAK/\delta)}$. This implies that $\bar f_k(s,a,h)\leq B$. Thus, we have
$$\langle\bm{\bar f_k}, \bm{q_{\alpha_k}}\rangle\leq \langle\bm{\vec{B}}\wedge(\bm{\bar f_k}+\frac{3H}{\bar{C}-\bar{C}_b}\bm{R_k} + \frac{H}{\bar{C} -\bar{C}_b}\bm{U_k}), \bm{q_{\alpha_k}}\rangle.$$
Then it follows that
\begin{align*}
&\langle \bm{\bar f_k}+\frac{3H}{\bar{C}-\bar{C}_b}\bm{R_k} + \frac{H}{\bar{C} -\bar{C}_b}\bm{U_k}, \bm{q_{\alpha_k}}\rangle-\langle \bm{\vec{B}}\wedge(\bm{\bar f_k}+\frac{3H}{\bar{C}-\bar{C}_b}\bm{R_k} + \frac{H}{\bar{C} -\bar{C}_b}\bm{U_k}), \bm{q_{\alpha_k}}\rangle\\
&\leq\langle \bm{\bar f_k}+\frac{3H}{\bar{C}-\bar{C}_b}\bm{R_k} + \frac{H}{\bar{C} -\bar{C}_b}\bm{U_k}, \bm{q_{\alpha_k}}\rangle-\langle \bm{\bar f_k}, \bm{q_{\alpha_k}}\rangle\\
&=\langle\frac{3H}{\bar{C}-\bar{C}_b}\bm{R_k} + \frac{H}{\bar{C} -\bar{C}_b}\bm{U_k}, \bm{q_{\alpha_k}}\rangle.
\end{align*}
Finally, we proved that
\begin{align*}
\langle \bm{f},\bm{q^*} \rangle - \langle \bm{\widehat f_k}, \bm{\widehat q_k} \rangle\leq \langle\frac{3H}{\bar{C}-\bar{C}_b}\bm{R_k} + \frac{H}{\bar{C} -\bar{C}_b}\bm{U_k}, \bm{q_{\alpha_k}}\rangle.
\end{align*}
By \Cref{lemma:sum-U_k}, we have
\begin{align*}
\sum_{k=K_0+1}^K\langle \bm{f},\bm{q^*} \rangle - \sum_{k=K_0+1}^K\langle \bm{\widehat f_k}, \bm{\widehat q_k} \rangle&\leq \sum_{k=K_0+1}^K\langle\frac{3H}{\bar{C}-\bar{C}_b}\bm{R_k} + \frac{H}{\bar{C} -\bar{C}_b}\bm{U_k}, \bm{q_{\alpha_k}}\rangle\\
&=\cO\left(\left(\frac{H^{2.5}}{\bar{C}-\bar{C}_b}S\sqrt{AK}+\frac{H^4}{\bar{C}-\bar{C}_b}S^3A\right)\left(\ln \frac{HSAK}{\delta}\right)^3\right)
\end{align*}
as desired.
\end{proof}

\begin{proof}[\textbf{Proof of \Cref{lemma:term-2}}]
    The lemma is a direct consequence of \Cref{lemma9} with $B = \cO(\ln(HSAK/\delta)).$ Hence, we have
    $$\sum_{k=K_0+1}^K \langle \bm{\widehat f_k}, \bm{\widehat q_k}-\bm{q_k}\rangle =\cO\left(\left(H^{1.5}S\sqrt{AK}+H^3S^3A\right)\left(\ln\frac{HSAK}{\delta}\right)^4 \right)$$
    with probability at least $1-2\delta$ under the good event $\cE$. By taking the union bound, the statement holds with probability at least $1-16\delta$.
\end{proof}

\begin{proof}[\textbf{Proof of \Cref{lemma:term-3}}]
We assume that the good event $\cE$ holds, which is the case with probability at least $1-14\delta$. The left-hand side of \Cref{lemma:term-3} can be rewritten as 
$$\sum_{k=K_0+1}^K \langle \bm{\widehat f_k}-\bm{f},\bm{q_k} \rangle.$$
Under the good event $\cE$, we have $\bar f_k(s,a,h)\leq f(s,a,h) + R_k(s,a,h)$ for $(s,a,h)\in\cS\times\cA\times[H]$ and $k\in[K]$. Furthermore, ${H}/({\bar C-\bar C_b})\geq 1$ due to $\bar C -\bar C_b \leq H$. Then it follows that
\begin{align*}
\sum_{k=K_0+1}^K \langle \bm{\widehat f_k}-\bm{f},\bm{q_k} \rangle&=\sum_{k=K_0+1}^K \langle \bm{\vec{B}} \wedge(\bm{\bar f_k}+\frac{3H}{\bar{C}-\bar{C}_b}\bm{R_k} + \frac{H}{\bar{C}-\bar{C}_b}\bm{U_k}) -\bm{f},\bm{q_k} \rangle\\
&\leq\sum_{k=K_0+1}^K \langle \bm{\bar f_k}+\frac{3H}{\bar{C}-\bar{C}_b}\bm{R_k} + \frac{H}{\bar{C}-\bar{C}_b}\bm{U_k} -\bm{f},\bm{q_k} \rangle\\
&\leq\frac{H}{\bar{C}-\bar{C}_b}\sum_{k=K_0+1}^K \langle4\bm{R_k} + \bm{U_k},\bm{q_k} \rangle\\
&=\cO\left(\left(\frac{H^{2.5}}{\bar{C}-\bar{C}_b}S\sqrt{AK}+\frac{H^4}{\bar{C}-\bar{C}_b}S^3A\right)\left(\ln\frac{HSAK}{\delta}\right)^3\right)
\end{align*}
where the last equality is due to \Cref{lemma:sum-U_k}.
\end{proof}

\begin{proof}[\textbf{Proof of \Cref{theorem:regret}}]
We assume that the good event $\cE$ holds, which is the case with probability at least $1-14\delta$. We decompose the regret as follows using occupancy measures.
\begin{align*}
&\regret\left(\vec\pi\right)\\
 	&=\underbrace{\sum_{k=1}^{K_0} \langle \bm{f}, \bm{q^*}\rangle - \sum_{k=1}^{K_0} \langle \bm{f}, \bm{q_k}\rangle}_{\text{(I)}}+\underbrace{\sum_{k=K_0+1}^K \langle \bm{f}, \bm{q^*}\rangle-\sum_{k=K_0+1}^K \langle \bm{\widehat f_k}, \bm{\widehat q_k}\rangle}_{\text{(II)}} + \underbrace{\sum_{k=K_0+1}^K \langle \bm{\widehat f_k}, \bm{\widehat q_k}-\bm{q_k}\rangle}_{\text{(III)}}+\underbrace{\sum_{k=K_0+1}^K \langle \bm{\widehat f_k}-\bm{f}, \bm{q_k}\rangle}_{\text{(IV)}}.
\end{align*}
As explained in \Cref{sec:analysis:regret}, we can upper bound term (I) as
\[
\widetilde{\cO}\left(\frac{H^4 S^2 A}{(\bar{C}-\bar{C}_b)^2}\right).
\]
because $K_0=\widetilde{\cO}\left(\frac{H^3S^2A}{(\bar C - \bar C_b)^2}\right)$ due to \Cref{lemma:K0} and $\langle \bm{f},\bm{q^*} \rangle\leq H$.

By \Cref{lemma:term-1}, we have
\begin{align*}
\text{Term (II)}&=\cO\left(\left(\frac{H^{2.5}}{\bar{C}-\bar{C}_b}S\sqrt{AK}+\frac{H^4}{\bar{C}-\bar{C}_b}S^3A\right)\left(\ln \frac{HSAK}{\delta}\right)^3\right).
\end{align*}

By \Cref{lemma:term-2}, with probability at least $1-2\delta$, it follows that 
\begin{align*}
\text{Term (III)}=\cO\left(\left(H^{1.5}S\sqrt{AK}+H^3S^3A\right)\left(\ln\frac{HSAK}{\delta}\right)^4 \right).
\end{align*}

Moreover, it follows from \Cref{lemma:term-3} that
\begin{align*}
\text{Term (IV)}&=\cO\left(\left(\frac{H^{2.5}}{\bar{C}-\bar{C}_b}S\sqrt{AK}+\frac{H^4}{\bar{C}-\bar{C}_b}S^3A\right)\left(\ln\frac{HSAK}{\delta}\right)^3\right).
\end{align*}

Hence, by taking the union bound, 
$$\regret\left(\vec\pi\right)=\widetilde{\cO}\left(\frac{H}{\bar C - \bar C_b} \left(H^{1.5}S\sqrt{AK}+\frac{H^4S^3A}{\bar C-\bar C_b}\right)\right)$$
with probability at least $1-16\delta$.
\end{proof}

\section{Technical Lemmas}\label{sec:appendix:technical}

In this section, we provide technical lemmas that are crucial for our regret and constraint violation analysis. The following lemma is from \citep{ssp-adversarial-unknown} with a few modifications, and it is useful to bound the variance of 
$\langle \bm{n_k}, \bm{f_k}\rangle$.
\begin{lemma}{\rm \citep[Lemma 2]{ssp-adversarial-unknown}}\label{lemma2}
	Let $\pi_k$ be any policy for episode $k$, and let $q_k$ denote the occupancy measure $q^{P,\pi_k}$. Let $\ell:\cS\times \cA\times[H]\to[-B,B]$ be an arbitrary function, and let $P$ be an arbitrary transition kernel. Then
	$$\mathbb{E}\left[\langle \bm{n_k}, \bm{\ell}\rangle^2\mid \ell, \pi_k, P\right]\leq  2B \langle \bm{q_k},\bm{\vec h}\odot \bm{\ell}\rangle$$
    where $\bm{q_k}, \bm{n_k},\bm{\ell}$ are the vector representations of $q_k,n_k,\ell$.
\end{lemma}
\begin{proof}
For ease of notation, let $\mathbb{E}_k\left[\cdot\right]$ denotes $\mathbb{E}\left[\cdot\mid \ell, \pi_k, P\right]$, and let $s_h$ and $a_h$ denote $s_h^{P,\pi_k}$ and $a_h^{P,\pi_k}$, respectively for $h\in[H]$. Note that
\begin{align*}
\mathbb{E}_k\left[\langle \bm{n_k}, \bm{\ell}\rangle^2\right]
&=\mathbb{E}_k\left[\left(\sum_{h=1}^H\sum_{(s,a)\in\cS\times\cA} n_k(s,a,h)\ell(s,a,h)\right)^2\right]\\
&=\mathbb{E}_k\left[\left(\sum_{h=1}^H \ell(s_h,a_h,h)\right)^2\right]\\
&\leq 2\mathbb{E}_k\left[\sum_{h=1}^H \ell(s_h,a_h,h)\left(\sum_{m=h}^H \ell(s_m,a_m,m)\right)\right]\\
&=2\mathbb{E}_k\left[\sum_{h=1}^H\mathbb{E}_k\left[\ell(s_h,a_h,h)\left(\sum_{m=h}^H\ell(s_m,a_m,m)\right)\mid s_h, a_h\right]\right]\\
&=2\mathbb{E}_k\left[\sum_{h=1}^H\ell(s_h,a_h,h)\mathbb{E}_k\left[\sum_{m=h}^H\ell(s_m,a_m,m)\mid s_h, a_h\right]\right]\\
&=2\mathbb{E}_k\left[\sum_{h=1}^H\ell(s_h,a_h,h) Q^{\pi_k}_h(s_h,a_h;\ell,P) \right]\\
&=2\mathbb{E}_k\left[\sum_{h=1}^H\sum_{(s,a)\in\cS\times\cA}n_k(s,a,h)\ell(s,a,h)Q^{\pi_k}_h(s,a;\ell,P) \right]
\end{align*}
where the first inequality holds because $(\sum_{h=1}^H x_h)^2\leq 2\sum_{h=1}^H x_h (\sum_{m=h}^H x_m)$. Moreover,
\begin{align*}
\mathbb{E}_k\left[\sum_{h=1}^H\sum_{(s,a)\in\cS\times\cA}n_k(s,a,h)\ell(s,a,h)Q^{\pi_k}_h(s,a;\ell,P) \right]
&= \sum_{h=1}^H\sum_{(s,a)\in\cS\times\cA}\ell(s,a,h)Q^{\pi_k}_h(s,a;\ell,P)\mathbb{E}_k\left[n_k(s,a,h)\right]\\
&=\sum_{h=1}^H\sum_{(s,a)\in\cS\times\cA}\ell(s,a,h)Q^{\pi_k}_h(s,a;\ell,P)q_k(s,a,h)\\
&=\langle \bm{q_k}, \bm{\ell}\odot\bm{Q^{P,\pi_k,\ell}}\rangle.
\end{align*}
Therefore, it follows that
$$\mathbb{E}_k\left[\langle \bm{n_k}, \bm{\ell}\rangle^2\right]\leq 2\langle \bm{q_k}, \bm{\ell}\odot\bm{Q^{P,\pi_k,\ell}}\rangle.$$
Next, observe that
\begin{align*}
\langle \bm{q_k},\bm{\ell}\odot \bm{Q^{P,\pi_k,\ell}}\rangle&\leq B \sum_{h=1}^H\sum_{(s,a)\in\cS\times\cA}Q^{\pi_k}_h(s,a;\ell,P)q_k(s,a,h)\\
&=B\sum_{h=1}^H\sum_{(s,a)\in\cS\times\cA} \pi_k(a\mid s, h)Q^{\pi_k}_h(s,a;\ell,P)\left(\sum_{a'\in \cA}q_k(s,a',h)\right)\\
&=B\sum_{h=1}^H\sum_{s\in\cS} V^{\pi_k}_h(s;\ell,P)\left(\sum_{a'\in \cA}q_k(s,a',h)\right)\\
&=B\sum_{h=1}^H\sum_{s\in\cS} \left(\sum_{m=h}^H\sum_{(s'',a'')\in \cS\times\cA} q_k(s'',a'',m\mid s,h)\ell(s'',a'',m)\right)\left(\sum_{a'\in \cA}q_k(s,a',h)\right)\\
&=B\sum_{h=1}^H\sum_{m=h}^H\sum_{(s'',a'')\in \cS\times\cA}\sum_{s\in\cS}q_k(s'',a'',m\mid s,h)\left(\sum_{a'\in \cA}q_k(s,a',h)\right)\ell(s'',a'',m)\\
&=B\sum_{h=1}^H\sum_{m=h}^H\sum_{(s'',a'')\in \cS\times\cA}q_k(s'',a'',m)\ell(s'',a'',m)\\
&=B\sum_{h=1}^H \sum_{(s,a)\in \cS\times\cA}h\cdot q_k(s,a,h)\ell(s,a,h)\\
&=B\langle \bm{q_k},\bm{\vec h}\odot \bm{\ell}\rangle
\end{align*}
where the first inequality holds because $\ell(s,a,h)\leq B$ for any $(s,a,h)$, the first equality holds because
$$q_k(s,a,h)= \pi_k(a\mid s, h)\sum_{a'\in \cA}q_k(s,a',h),$$
the fifth equality follows from 
$$\sum_{s\in\cS}q_k(s'',a'',m\mid s,h)\left(\sum_{a'\in \cA}q_k(s,a',h)\right)=q_k(s'',a'',m).$$
Therefore, we get that $\langle \bm{q_k}, \bm{\ell}\odot \bm{Q^{P,\pi_k,\ell}}\rangle\leq B\langle \bm{q_k},\bm{\vec h}\odot \bm{\ell}\rangle$ as required.
\end{proof}

The following lemma is from the first statement of \cite[Lemma 7]{ssp-adversarial-unknown} with a few modifications to adapt the proof to our setting.
\begin{lemma}{\rm \citep[Lemma 7]{ssp-adversarial-unknown}}\label{lemma7}
Let $\pi$ be a policy, and let $\widetilde P,\widehat P$ be two different transition kernels. We denote by $\widetilde q$ the occupancy measure $q^{\widetilde P,\pi}$ associated with $\widetilde P$ and $\pi$, and we denote by $\widehat q$ the occupancy measure $q^{\widehat P,\pi}$ associated with $\widehat P$ and $\pi$. Then
\begin{align*}
\widehat q(s,a,h) - \widetilde q(s,a,h)
&=\sum_{(s',a',s'')\in\cS\times \cA\times \cS}\sum_{m=1}^{h-1}\widetilde q(s',a',m)\left(\widehat P(s''\mid s',a',m)- \widetilde P(s''\mid s',a',m)\right)\widehat q(s,a,h\mid s'',m+1).
\end{align*}

\end{lemma}
\begin{proof}
We prove the first statement by induction on $h$. When $h=1$, note that
$$\widehat q(s,a,h)= \widetilde q(s,a,h) = \pi(a\mid s,1)\cdot p(s).$$
Hence, both the left-hand side and right-hand side are equal to 0. Next, assume that the equality holds with $h-1\geq 1$. Then we consider $h$. By the definition of occupancy measures,
\begin{align*}
\widehat q(s,a,h) - \widetilde q(s,a,h)
&=\pi(a\mid s,h)\sum_{(s',a')\in\cS\times \cA}(\widehat P(s\mid s',a',h-1)\widehat q(s',a',h-1)-\widetilde P(s\mid s',a',h-1)\widetilde q(s',a',h-1))\\
&=\underbrace{\pi(a\mid s,h)\sum_{(s',a')\in\cS\times \cA}\widehat P(s\mid s',a',h-1)(\widehat q(s',a',h-1)-\widetilde q(s',a',h-1))}_{\text{Term 1}}\\
&\qquad + \underbrace{\pi(a\mid s,h)\sum_{(s',a')\in\cS\times \cA}\widetilde q(s',a',h-1)(\widehat P(s\mid s',a',h-1)-\widetilde P(s\mid s',a',h-1))}_{\text{Term 2}}.
\end{align*}
To provide an upper bound on Term 1, we use the induction hypothesis for $h-1$:
\begin{align*}
&\widehat q(s',a',h-1) - \widetilde q(s',a',h-1)\\
&=\sum_{(s'',a'',s''')\in\cS\times\cA\times\cS}\sum_{m=1}^{h-2}\widetilde q(s'',a'',m)\left((\widehat P- \widetilde P)(s'''\mid s'',a'',m)\right)\widehat q(s',a',h-1\mid s''',m+1)
\end{align*}
where 
$$(\widehat P- \widetilde P)(s'''\mid s'',a'',m)=\widehat P(s'''\mid s'',a'',m)- \widetilde P(s'''\mid s'',a'',m).$$
In addition, observe that
$$\pi(a\mid s,h)\sum_{(s',a')\in\cS\times \cA} \widehat P(s\mid s',a',h-1)\widehat q(s',a',h-1\mid s''',m+1)=\widehat q(s,a,h\mid s''',m+1).$$
Therefore, it follows that Term 1 is equal to
\begin{align*}
&\sum_{(s'',a'',s''')\in\cS\times\cA\times\cS}\sum_{m=1}^{h-2}\widetilde q(s'',a'',m)\left((\widehat P- \widetilde P)(s'''\mid s'',a'',m)\right)\widehat q(s,a,h\mid s''',m+1)\\
&=\sum_{(s',a',s'')\in\cS\times\cA\times\cS}\sum_{m=1}^{h-2}\widetilde q(s',a',m)\left(\widehat P(s''\mid s',a',m)- \widetilde P(s''\mid s',a',m)\right)\widehat q(s,a,h\mid s'',m+1).
\end{align*}

Next, we upper bound Term 2. Note that
$$\widehat q(s,a,h\mid s'',h) = \pi(a\mid s'',h)\cdot \mathbf{1}\left[s''=s\right].$$
Then it follows that
\begin{align*}
&\pi(a\mid s,h)(\widehat P(s\mid s',a',h-1)-\widetilde P(s\mid s',a',h-1))\\
&= \sum_{s''\in\cS} \mathbf{1}\left[s''=s\right]\cdot \pi(a\mid s'',h)(\widehat P(s''\mid s',a',h-1)-\widetilde P(s''\mid s',a',h-1))\\
&= \sum_{s''\in\cS}\widehat q(s,a,h\mid s'',h)(\widehat P(s''\mid s',a',h-1)-\widetilde P(s''\mid s',a',h-1)),
\end{align*}
implying in turn that Term 2 equals
$$\sum_{(s',a',s'')\in\cS\times \cA\times \cS}\widetilde q(s',a',h-1)(\widehat P(s''\mid s',a',h-1)-\widetilde P(s''\mid s',a',h-1))\widehat q(s,a,h\mid s'',h).$$
Adding the equivalent expression of Term 1 and that of Term 2 that we have obtained, we get the right-hand side of the statement.
\end{proof}

The following lemma is called value difference lemma~\citep{dann2017pac}. Based on~\Cref{lemma:confidence'} and~\Cref{lemma7}, we show the following lemma, which is a modification of \cite[Lemma 7, the second statement]{ssp-adversarial-unknown}.
\begin{lemma}\label{lemma:confidence''}
Let $\pi$ be a policy, and let $\widetilde P,\widehat P$ be two different transition kernels. We denote by $\widetilde q$ the occupancy measure $q^{\widetilde P,\pi}$ associated with $\widehat P$ and $\pi$, and we denote by $\widehat q$ the occupancy measure $q^{\widehat P,\pi}$ associated with $\widehat P$ and $\pi$. Let $\ell:\cS\times\cA\times[H]\rightarrow[-B, B]$ be an arbitrary function.
If $\widetilde P, \widehat P\in \cP_k$, then we have
\begin{align*}
\left|\langle\bm{\ell},\bm{\widehat q} - \bm{\widetilde q} \rangle\right|
&=\left|\sum_{(s,a,s',h)\in\cS\times\cA\times\cS\times[H]}\widetilde q(s,a,h)\left(\widehat P(s'\mid s,a,h)-\widetilde P(s'\mid s,a,h)\right)V^{\pi}_{h+1}(s';\ell, \widehat P)\right|\\
&\leq B H\sum_{(s,a,s',h)\in\cS\times \cA\times\cS\times[H]}\widetilde q(s,a,h)\epsilon_k^\star (s'\mid s,a,h)
\end{align*}
where $\bm{\widehat q},\bm{\widetilde q}, \bm{\ell}$ are the vector representations of $\widehat q, \widetilde q, \ell$.
\end{lemma}
\begin{proof}
First, observe that
\begin{align*}
\langle\bm{\ell},\bm{\widehat q} - \bm{\widetilde q} \rangle&=\sum_{(s,a,h)\in\cS\times\cA\times[H]}\left(\widehat q(s,a,h)-\widetilde q(s,a,h)\right)\ell(s,a,h).
\end{align*}
By \Cref{lemma7}, the right-hand side can be rewritten so that we obtain the following.
\begin{align*}
\langle\bm{\ell},\bm{\widehat q} - \bm{\widetilde q} \rangle
&=\sum_{(s,a,h)}
\sum_{(s',a',s'')}\sum_{m=1}^{h-1}\widetilde q(s',a',m)\left((\widehat P- \widetilde P)(s''\mid s',a',m)\right)\widehat q(s,a,h\mid s'',m+1)
\ell(s,a,h)\\
&=\sum_{m=1}^{H}\sum_{(s',a',s'')}\widetilde q(s',a',m)\left((\widehat P- \widetilde P)(s''\mid s',a',m)\right)\sum_{(s,a,h):h>m}\widehat q(s,a,h\mid s'',m+1)
\ell(s,a,h)\\
&=\sum_{m=1}^{H}\sum_{(s',a',s'')}\widetilde q(s',a',m)\left((\widehat P- \widetilde P)(s''\mid s',a',m)\right)V^{\pi}_{m+1}(s'';\ell, \widehat P)\\
&=\sum_{h=1}^{H}\sum_{(s',a',s'')}\widetilde q(s',a',h)\left(\widehat P(s''\mid s',a',h)-\widetilde P(s''\mid s',a',h)\right)V^{\pi}_{h+1}(s'';\ell, \widehat P).
\end{align*}
Since $\widetilde P, \widehat P\in\cP_k$, \Cref{lemma:confidence'} implies that
\begin{align*}
\left|\langle\bm{\ell},\bm{\widehat q} - \bm{\widetilde q} \rangle\right|
&\leq\sum_{h=1}^{H}\sum_{(s',a',s'')}\widetilde q(s',a',h)\left|\widehat P(s''\mid s',a',h)-\widetilde P(s''\mid s',a',h)\right|V^{\pi}_{h+1}(s'';\ell, \widehat P)\\
&\leq \sum_{h=1}^{H}\sum_{(s',a',s'')}\widetilde q(s',a',h)\left(2\epsilon_k(s''\mid s',a',h)
\right) V^{\pi}_{h+1}(s'';\ell, \widehat P)\\
&\leq \sum_{h=1}^{H}\sum_{(s',a',s'')}\widetilde q(s',a',h)\epsilon_k^\star(s''\mid s',a',h)V^{\pi}_{h+1}(s'';\ell, \widehat P)\\
&\leq B H\sum_{h=1}^{H}\sum_{(s',a',s'')}\widetilde q(s',a',h)\epsilon_k^\star(s''\mid s',a',h)\\
&=B H\sum_{(s,a,s',h)\in\cS\times\cA\times\cS\times[H]}\widetilde q(s,a,h)\epsilon_k^\star(s'\mid s,a,h)
\end{align*}
where the third inequality holds because $V^{\pi}_{h+1}(s'';\ell, \widehat P)\leq B H$, as required.
\end{proof}

\begin{lemma}\label{lemma:confidence'''}
Let $\pi$ be a policy, and let $\widetilde P,\widehat P$ be two different transition kernels. We denote by $\widetilde q$ the occupancy measure $q^{\widetilde P,\pi}$ associated with $\widetilde P$ and $\pi$, and we denote by $\widehat q$ the occupancy measure $q^{\widehat P,\pi}$ associated with $\widehat P$ and $\pi$. Let $(s,h)\in \cS\times [H]$, and consider $\widetilde q(\cdot \mid s,h), \widehat q(\cdot \mid s,h):\cS\times\cA\times\{h,\ldots, H\}$.
If $\widetilde P, \widehat P\in \cP_k$, then we have
$$\left|\langle\bm{\ell_{(h)}}, \bm{\widehat q_{(s,h)}} - \bm{\widetilde q_{(s,h)}}\rangle\right|\leq B H\sum_{(s',a',s'',m)\in\cS\times \cA\times\cS\times\{h,\ldots,H\}}\widetilde q(s',a',m\mid s,h)\epsilon_k^\star (s''\mid s',a',m)$$
where $\bm{\widetilde q_{(s,h)}},\bm{\widehat q_{(s,h)}}, \bm{\ell_{(h)}}$ are the vector representations of $\widehat q(\cdot\mid s,h), \widetilde q(\cdot\mid s,h):\cS\times\cA\times\{h,\ldots, H\}\to[0,1]$ and $\ell_{(h)}:\cS\times\cA\times[H]\rightarrow [-B, B].$
\end{lemma}
\begin{proof}
The proof follows the same argument used to prove Lemmas \ref{lemma7} and \ref{lemma:confidence''}.
\end{proof}

The following lemma is called a Bellman-type law of total variance lemma~\citep{azar2017, ssp-adversarial-unknown}. We follow the proof of \cite[Lemma 4]{ssp-adversarial-unknown} after some changes to adapt to our setting.
\begin{lemma}{\rm \citep[Lemma 4]{ssp-adversarial-unknown}}\label{lemma4}
Let $\pi_k$ be the policy for episode $k$, $P$ be an arbitrary transition kernel, and let $q_k$ denote the occupancy measure $q^{P,\pi_k}$. Let $\ell:\cS\times \cA\times[H]\to[-B,B]$ be an arbitrary reward function, and define $\mathbb{V}_k(s,a,h)=\var_{s'\sim P(\cdot\mid s,a,h)}\left[V^{\pi_k}_{h+1}(s';\ell,P)\right]$. Then
$$\langle \bm{q_k},\bm{\mathbb{V}_k}\rangle\leq \var\left[\langle \bm{n_k},\bm{\ell}\rangle\mid \ell,\pi_k, P\right]$$
where $\bm{q_k},\bm{\mathbb{V}_k}, \bm{n_k},\bm{\ell}$ are the vector representations of $q_k, \mathbb{V}_k,n_k,\ell.$ %
\end{lemma}
\begin{proof}
For ease of notation, let $s_h$ and $a_h$ denote $s_h^{P,\pi_k}$ and $a_h^{P,\pi_k}$, respectively for $h\in[H]$. Moreover, let $V(s,h)$ denote $V^{\pi}_h(s;\ell,P)$ for $(s,h)\in\cS\times[H]$.
Note that
\begin{align*}
\langle \bm{n_k},\bm{\ell}\rangle= \sum_{(s,a,h)\cS\times\cA\times[H]}\ell(s,a,h)n_k(s,a,h)=\sum_{h=1}^H \ell\left(s_h, a_h,h\right).
\end{align*}
For ease of notation, let $\mathbb{E}_k\left[\cdot\right]$ and $\var_k\left[\cdot\right]$ denote $\mathbb{E}\left[\cdot \mid \ell,\pi_k, P\right]$ and $\var\left[\cdot \mid \ell,\pi_k, P\right]$, respectively. 
Then
\begin{align*}
    \mathbb{E}_k\left[\langle \bm{n_k},\bm{\ell}\rangle\right] = \mathbb{E}_k\left[\sum_{h=1}^H \ell\left(s_h, a_h,h\right)\right]= \mathbb{E}_k\left[\mathbb{E}\left[\sum_{h=1}^H \ell\left(s_h, a_h,h\right)\mid \ell,\pi_k,P, s_1\right]\right]
     =\mathbb{E}_k\left[V(s_1,1)\right].
    \end{align*}
Moreover, 
\begin{align*}
\text{Var}_k\left[\langle \bm{n_k},\bm{\ell}\rangle\right]
&=\mathbb{E}_k\left[\left(\sum_{h=1}^H \ell\left(s_h, a_h,h\right)- \mathbb{E}_k\left[V(s_1,1)\right]\right)^2\right]\\
&=\mathbb{E}_k\left[\left(\sum_{h=1}^H \ell\left(s_h, a_h,h\right)- V(s_1,1)+V(s_1,1)-\mathbb{E}_k\left[V(s_1,1)\right] \right)^2\right]\\
&=\mathbb{E}_k\left[\left(\sum_{h=1}^H \ell\left(s_h, a_h,h\right)- V(s_1,1)\right)^2\right]+\mathbb{E}_k\left[\left(V(s_1,1)-\mathbb{E}_k\left[V(s_1,1)\right] \right)^2\right]\\
&\quad + 2\mathbb{E}_k\left[\left(\sum_{h=1}^H \ell\left(s_h, a_h,h\right)- V(s_1,1)\right)\left(V(s_1,1)-\mathbb{E}_k\left[V(s_1,1)\right] \right)\right]\\
&\geq\mathbb{E}_k\left[\left(\sum_{h=1}^H \ell\left(s_h, a_h,h\right)- V(s_1,1)\right)^2\right]
\end{align*}
where the inequality is by
$\mathbb{E}_k\left[V(s_1,1)-\mathbb{E}_k\left[V(s_1,1)\right]\mid s_1\right]=0$ and $\left(V(s_1,1)-\mathbb{E}_k\left[V(s_1,1)\right] \right)^2\geq 0$.
Therefore,
\begin{align*}
\text{Var}_k\left[\langle \bm{n_k},\bm{\ell}\rangle\right]
&\geq\mathbb{E}_k\left[\left(\sum_{h=2}^H \ell\left(s_h, a_h,h\right)- V(s_2,2)+\ell\left(s_1, a_1,1\right)+V(s_2,2)-V(s_1,1)\right)^2\right].
\end{align*}
Note that 
\begin{align}\label{CLlemma4:eq1}
\begin{aligned}
\mathbb{E}_k\left[\sum_{h=2}^H \ell\left(s_h, a_h,h\right)- V(s_2,2)\mid s_1,a_1,s_2\right]&=\mathbb{E}_k\left[\sum_{h=2}^H \ell\left(s_h, a_h,h\right)\mid s_2\right]- V(s_2,2)= 0.
\end{aligned}
\end{align}
Then 
\begin{align*}
\text{Var}_k\left[\langle \bm{n_k},\bm{\ell}\rangle\right]
&\geq\mathbb{E}_k\left[\left(\sum_{h=2}^H \ell\left(s_h, a_h,h\right)- V(s_2,2)\right)^2\right]+\mathbb{E}_k\left[\left(\ell\left(s_1,a_1,1\right)+V(s_2,2)-V(s_1,1)\right)^2\right]\\
&\quad + 2\mathbb{E}_k\left[\mathbb{E}_k\left[\left(\sum_{h=2}^H \ell\left(s_h, a_h,h\right)- V(s_2,2)\right)\left(\ell\left(s_1,a_1,1\right)+V(s_2,2)-V(s_1,1) \right)\mid s_1,a_1,s_2\right]\right]\\
&=\mathbb{E}_k\left[\left(\sum_{h=2}^H \ell\left(s_h, a_h,h\right)- V(s_2,2)\right)^2\right]+\mathbb{E}_k\left[\left(\ell\left(s_1,a_1,1\right)+V(s_2,2)-V(s_1,1) \right)^2\right]\\
&\quad + 2\mathbb{E}_k\left[\left(\ell\left(s_1,a_1,1\right)+V(s_2,2)-V(s_1,1) \right)\mathbb{E}_k\left[\sum_{h=2}^H \ell\left(s_h, a_h,h\right)- V(s_2,2)\mid s_1,a_1,s_2\right]\right]\\
&=\mathbb{E}_k\left[\left(\sum_{h=2}^H \ell\left(s_h, a_h,h\right)- V(s_2,2)\right)^2\right]+\mathbb{E}_k\left[\left(\ell\left(s_1,a_1,1\right)+V(s_2,2)-V(s_1,1)\right)^2\right]
\end{align*}
where the last equality follows from~\eqref{CLlemma4:eq1}.
Here, the second term from the right-most side can be bounded from below as follows.
\begin{align*}
&\mathbb{E}_k\left[\left(\ell\left(s_1,a_1,1\right)+V(s_2,2)-V(s_1,1) \right)^2\right]\\
&=\mathbb{E}_k\left[\left(\ell\left(s_1,a_1,1\right)+\sum_{s'\in \cS}P(s'\mid s_1, a_1,1)V(s',2)-V(s_1,1) +V(s_2,2)-\sum_{s'\in \cS}P(s'\mid s_1, a_1,1)V(s',2)\right)^2\right]\\
&=\mathbb{E}_k\left[\left(\ell\left(s_1,a_1,1\right)+\sum_{s'\in \cS}P(s'\mid s_1, a_1,1)V(s',2)-V(s_1,1) \right)^2\right]\\
&\quad + \mathbb{E}_k\left[\left(V(s_2,2)-\sum_{s'\in \cS}P(s'\mid s_1, a_1,1)V(s',2)\right)^2\right]\\
&\quad + 2\mathbb{E}_k\left[\left(\ell\left(s_1,a_1,1\right)+\sum_{s'\in \cS}P(s'\mid s_1, a_1,1)V(s',2)-V(s_1,1) \right)\left(V(s_2,2)-\sum_{s'\in \cS}P(s'\mid s_1, a_1,1)V(s',2)\right)\right]\\
&=\mathbb{E}_k\left[\left(\ell\left(s_1,a_1,1\right)+\sum_{s'\in \cS}P(s'\mid s_1, a_1,1)V(s',2)-V(s_1,1) \right)^2\right]\\
&\quad + \mathbb{E}_k\left[\left(V(s_2,2)-\sum_{s'\in \cS}P(s'\mid s_1, a_1,1)V(s',2)\right)^2\right]\\
&\geq \mathbb{E}_k\left[\mathbb{V}_k(s_1,a_1,1)\right]
\end{align*}
where third equality holds because
\begin{align*}
&\mathbb{E}_k\left[\left(\ell\left(s_1,a_1,1\right)+\sum_{s'\in \cS}P(s'\mid s_1, a_1,1)V(s',2)-V(s_1,1) \right)\left(V(s_2,2)-\sum_{s'\in \cS}P(s'\mid s_1, a_1,1)V(s',2)\right)\mid s_1,a_1\right]\\
&=\left(\ell\left(s_1,a_1,1\right)+\sum_{s'\in \cS}P(s'\mid s_1, a_1,1)V(s',2)-V(s_1,1)  \right)\mathbb{E}_k\left[V(s_2,2)-\sum_{s'\in \cS}P(s'\mid s_1, a_1,1)V(s',2)\mid s_1,a_1\right]\\
&=\left(\ell\left(s_1,a_1,1\right)+\sum_{s'\in \cS}P(s'\mid s_1, a_1,1)V(s',2)-V(s_1,1)  \right)\times 0
\end{align*}
and the last inequality holds because
\begin{align*}
\mathbb{E}_k\left[\left(V(s_2,2) -\sum_{s'\in \cS}P(s'\mid s_1, a_1,1)V(s',2)\right)^2\right]= \mathbb{E}_k\left[\mathbb{V}_k(s_1,a_1,1)\right].
\end{align*}
Then it follows that
\begin{align*}
\text{Var}_k\left[\langle \bm{n_k},\bm{\ell}\rangle\right]&\geq\mathbb{E}_k\left[\left(\sum_{h=1}^H \ell\left(s_h, a_h,h\right)- V(s_1,1)\right)^2\right]\geq \mathbb{E}_k\left[\left(\sum_{h=2}^H \ell\left(s_h, a_h,h\right)- V(s_2,2)\right)^2\right]+ \mathbb{E}_k\left[\mathbb{V}_k(s_1,a_1,1)\right].
\end{align*}
Repeating the same argument, we deduce that
\begin{align*}
\text{Var}_k\left[\langle \bm{n_k},\bm{\ell}\rangle\right]&\geq\sum_{h=1}^H \mathbb{E}_k\left[\mathbb{V}_k(s_h,a_h,h)\right]= \sum_{(s,a,h)\in\cS\times \cA\times [H]}q_k(s,a,h)\mathbb{V}_k(s,a,h)=\langle \bm{q_k},\bm{\mathbb{V}_k}\rangle,
\end{align*}
as required.
\end{proof}

The following lemma is useful when we prove \Cref{theorem:U_k}. The proof is inspired by \cite[Lemma 10]{ssp-adversarial-unknown} with a few modifications.
\begin{lemma}\label{lemma10}
Assume that the good event $\cE$ holds. Let $\pi_k$ be any policy for episode $k\in[K]$, and let $q_k$ denote the occupancy measure $q^{P,\pi_k}:\cS\times\cA\times[H]\to [0,1]$. Let $\ell:\cS\times\cA\times[H]\to[-B, B]$ be an arbitrary reward function. Then
\begin{align*}
&\left|\sum_{(s,a,s',h)}q_k(s,a,h)\left(P-P_k\right)(s'\mid s,a,h)\left(V^{\pi_k}_{h+1}(s';\ell,P_k)-V^{\pi_k}_{h+1}(s';\ell,P)\right)\right|\\
&\leq 10^4 BH^2S^2 \left(\ln \frac{HSAK}{\delta}\right)^2 \sum_{(s,a,h)} \frac{q_k(s,a,h)}{\max\{1, N_k(s,a,h)\}}
\end{align*}
for any $P_k\in \cP_k$ where $\left(P-P_k\right)(s'\mid s,a,h)=P(s'\mid s,a,h)-P_k(s'\mid s,a,h)$.

\end{lemma}
\begin{proof}
Let $\bm{q^{P_k,\pi_k}_{(s',h+1)}},\bm{q^{P,\pi_k}_{(s',h+1)}},\bm{\ell}$ be the vector representations of $q^{P_k,\pi_k}(\cdot \mid s',h+1),q^{P,\pi_k}(\cdot \mid s',h+1):\cS\times\cA\times\{h+1,\ldots, H\}\to[0,1]$, and $\ell_{(h+1)}:\cS\times\cA\times\{h+1,\ldots, H\}\to[-B, B]$ respectively. Note that 
\begin{align*}
&\left|\sum_{(s,a,s',h)}q_k(s,a,h)\left(\left(P-P_k\right)(s'\mid s,a,h)\right)\left(V^{\pi_k}_{h+1}(s';\ell,P_k)-V^{\pi_k}_{h+1}(s';\ell,P)\right)\right|\\
&\leq \sum_{(s,a,s',h)}q_k(s,a,h)\epsilon_k^\star(s'\mid s,a,h)\left|\left(V^{\pi_k}_{h+1}(s';\ell,P_k)-V^{\pi_k}_{h+1}(s';\ell,P)\right)\right|\\
&= \sum_{(s,a,s',h)}q_k(s,a,h)\epsilon_k^\star(s'\mid s,a,h)\left|\langle\bm{q^{P_k,\pi_k}_{(s',h+1)}}-\bm{q^{P,\pi_k}_{(s',h+1)}}, \bm{\ell_{(h+1)}}\rangle\right|\\
&\leq B H\sum_{(s,a,s',h)}q_k(s,a,h)\epsilon_k^\star(s'\mid s,a,h) \sum_{(s'',a'',s'''), m\geq h+1}q_k(s'',a'',m\mid s',h+1)\epsilon_k^\star(s'''\mid s'',a'',m)\\
\end{align*}
where the first inequality is from~\Cref{lemma:confidence'}, the first equality holds because $V^{\pi_k}_{h+1}(s';\ell,P_k)=\langle\bm{q^{P_k,\pi_k}_{(s',h+1)}}, \bm{\ell_{(h+1)}}\rangle$ and $V^{\pi_k}_{h+1}(s';\ell,P)=\langle\bm{q^{P,\pi_k}_{(s',h+1)}}, \bm{\ell_{(h+1)}}\rangle$, the second inequality is due to \Cref{lemma:confidence'''}. Remember that the definition of $\epsilon^\star_k$ is given by
$$\epsilon^\star_k(s'\mid s,a,h) = 6\sqrt{\frac{P(s'\mid s,a,h) \ln(HSAK/\delta)}{\max\{1,N_k(s,a,h)\}}}+94\frac{\ln(HSAK/\delta)}{\max\{1,N_k(s,a,h)\}}.$$

Then it follows that
{\allowdisplaybreaks\begin{align*}
&\left(\ln\frac{HSAK}{\delta}\right)^{-2}\sum_{(s,a,s',h)} q_k(s,a,h)\epsilon_k^\star(s'\mid s,a,h) \sum_{(s'', a'', s'''),m\geq h+1} q_k(s'',a'',m\mid s', h+1)\epsilon_k^\star(s'''\mid s'',a'',m)\\
&\leq36\underbrace{\sum_{\substack{(s,a,s',h),\\(s'',a'',s'''),\\ m\geq h+1}} \sqrt{\frac{q_k(s,a,h)^2 P(s'\mid s,a,h)}{\max\{1, N_k(s,a,h)\}}} \sqrt{\frac{q_k(s'',a'',m\mid s',h+1)^2 P(s'''\mid s'',a'',m)}{\max\{1,N_k(s'',a'',m)\}}}}_{\text{Term 1}}\\
&\quad+564\underbrace{\sum_{\substack{(s,a,s',h),\\(s'',a'',s'''),\\ m\geq h+1}}\sqrt{\frac{q_k(s,a,h)^2 P(s'\mid s,a,h)}{\max\{1, N_k(s,a,h)\}}}\frac{q_k(s'',a'',m \mid s',h+1)}{\max\{1,N_k(s'',a'',m)\}}}_{\text{Term 2}}\\
&\quad+564\underbrace{\sum_{\substack{(s,a,s',h),\\(s'',a'',s'''),\\ m\geq h+1}}\frac{q_k(s,a,h)}{\max\{1, N_k(s,a,h)\}}\sqrt{\frac{q_k(s'',a'',m\mid s',h+1)^2 P(s'''\mid s'',a'',m)}{\max\{1,N_k(s'',a'',m)\}}}}_{\text{Term 3}}\\
&\quad+8836\underbrace{\sum_{\substack{(s,a,s',h),\\(s'',a'',s'''),\\ m\geq h+1}}\frac{q_k(s,a,h)}{\max\{1, N_k(s,a,h)\}}\frac{q_k(s'',a'',m \mid s',h+1)}{\max\{1,N_k(s'',a'',m)\}}}_{\text{Term 4}}.
\end{align*}}

Term 1 can be bounded as follows.
\begin{align*}
\text{Term 1}\leq&\sqrt{\sum_{\substack{(s,a,s',h),\\(s'',a'',s'''),\\ m\geq h+1}} {\scriptstyle \frac{q_k(s,a,h) P(s'''\mid s'',a'',m)q_k(s'',a'',m\mid s',h+1)}{\max\{1, N_k(s,a,h)\}}}} \sqrt{\sum_{\substack{(s,a,s',h),\\(s'',a'',s'''),\\ m\geq h+1}} {\scriptstyle\frac{q_k(s'',a'',m\mid s',h+1) P(s'\mid s,a,h)q_k(s,a,h)}{\max\{1,N_k(s'',a'',m)\}}}}\\
\leq&\sqrt{HS\sum_{(s,a,h)} \frac{q_k(s,a,h)}{\max\{1, N_k(s,a,h)\}}} \sqrt{HS\sum_{(s'',a'',m)} \frac{q_k(s'',a'',m)}{\max\{1,N_k(s'',a'',m)\}}}\\
=&HS\sum_{(s,a,h)} \frac{q_k(s,a,h)}{\max\{1,N_k(s,a,h)\}}
\end{align*}
where the first inequality is from the Cauchy-Schwarz inequality.

We can bound Term 2 as the following argument.
\begin{align*}
\text{Term 2}\leq&\sqrt{\sum_{\substack{(s,a,s',h),\\(s'',a'',s'''),\\ m\geq h+1}}{\scriptstyle\frac{q_k(s,a,h)q_k(s'',a'',m\mid s',h+1)}{\max\{1,N_k(s,a,h)\}\max\{1,N_k(s'',a'',m)\}}}}\sqrt{\sum_{\substack{(s,a,s',h),\\(s'',a'',s'''),\\ m\geq h+1}}{\scriptstyle\frac{q_k(s'',a'',m\mid s',h+1)P(s'\mid s,a,h)q_k(s,a,h)}{\max\{1,N_k(s'',a'',m)\}}}}\\
\leq&\sqrt{HS^2\sum_{(s,a,h)}\frac{q_k(s,a,h)}{\max\{1,N_k(s,a,h)\}}}\sqrt{HS\sum_{(s'',a'',m)} \frac{q_k(s'',a'',m)}{\max\{1,N_k(s'',a'',m)\}}}\\
=& HS^{1.5} \sum_{(s,a,h)} \frac{q_k(s,a,h)}{\max\{1, N_k(s,a,h)\}}.
\end{align*}
Similar to Term 2, we have an upper bound on Term 3 as follows.
$$\text{Term 3} = HS^{1.5} \sum_{(s,a,h)} \frac{q_k(s,a,h)}{\max\{1, N_k(s,a,h)\}}.$$

Since $1/{\max\{1,N_k(s,a,h)\}} \leq 1$, we bound Term 4 in the following way.
\begin{align*}
\text{Term 4} \leq& HS^2\sum_{(s,a,h)} \frac{q_k(s,a,h)}{\max\{1,N_k(s,a,h)\}}.
\end{align*}

Finally, we deduce that
\begin{align*}
&\left|\sum_{(s,a,s',h)}q_k(s,a,h)\left(P-P_k\right)(s'\mid s,a,h)\left(V^{\pi_k}_{h+1}(s';\ell,P_k)-V^{\pi_k}_{h+1}(s';\ell,P)\right)\right|\\
&\leq 10^4 BH^2S^2 \left(\ln \frac{HSAK}{\delta}\right)^2 \sum_{(s,a,h)} \frac{q_k(s,a,h)}{\max\{1, N_k(s,a,h)\}}
\end{align*} 
as desired.
\end{proof}

Next, we provide \Cref{lemma9}, which is a modification of \citep[Lemma 9]{ssp-adversarial-unknown} to our finite-horizon MDP setting.

\begin{lemma}\label{lemma9}
	Assume that the good event $\cE$ holds. Let $\pi_k$ be any policy for episode $k$, let $P_k$ be any transition kernel from $\cP_k$ for episode $k$, and let $P$ be the true transition kernel. Let $q_k,\widehat q_k$ denote the occupancy measures $q^{P,\pi_k},q^{P_k,\pi_k}$, respectively. Let $\ell_k:\cS\times \cA\times[H]\to[-B,B]$ be an arbitrary reward function for episode $k$. With probability at least $1-2\delta$, 
	\begin{align*}
		\sum_{k=1}^K \left|\langle  \bm{\ell_k}, \bm{q_k}-\bm{\widehat q_k}\rangle\right|=\cO\left(B\left(H^{1.5}S\sqrt{AK}+H^3S^3A\right)\left(\ln\frac{HSAK}{\delta}\right)^3 \right).
	\end{align*}
    where $\bm{q_k},\bm{\widehat q_k},\bm{\ell_k}$ are the vector representations of $q_k, \widehat q_k, \ell_k$.
\end{lemma}
\begin{proof}
We define $\xi_1$ as 
$\xi_1=\left\{\ell_1,\pi_1\right\}$
and for $k\geq 2$, we define $\xi_k$ as 
$$\left\{s_1^{P,\pi_{k-1}},a_1^{P,\pi_{k-1}},\ldots, s_h^{P,\pi_{k-1}}, a_h^{P,\pi_{k-1}}, \ell_k,\pi_k\right\}$$
where $\pi_{k-1}$ and $\pi_k$ denote the policies for episode $k-1$ and episode $k$, respectively, and $$\left(s_1^{P,\pi_{k-1}},a_1^{P,\pi_{k-1}},\ldots, s_h^{P,\pi_{k-1}}, a_h^{P,\pi_{k-1}}\right)$$
is the trajectory generated under policy $\pi_{k-1}$ and transition kernel $P$. Then for $k\in[K]$, let $\cH_k$ be defined as the $\sigma$-algebra generated by the random variables in $\xi_1\cup\cdots\cup \xi_{k}$. Then it follows that $\cH_1,\ldots, \cH_k$ give rise to a filtration.

Let us define $$\mu_k(s,a,h)=\mathbb{E}_{s'\sim P(\cdot\mid s,a,h)}\left[V^{\pi_k}_{h+1}(s';\ell_k, P)\right].$$

Note that
\begin{align*}
\sum_{k=1}^K \left|\langle \bm{\ell_k}, \bm{q_k}-\bm{\widehat q_k}\rangle\right|
&=\sum_{k=1}^K\left|\sum_{(s,a,s',h)\in\cS\times\cA\times\cS\times[H]}q_k(s,a,h)\left(P(s'\mid s,a,h)-P_k(s'\mid s,a,h)\right)V^{\pi_k}_{h+1}(s';\ell_k, P_k)\right|\\
&\leq \sum_{k=1}^K\left|\sum_{(s,a,s',h)\in\cS\times\cA\times\cS\times[H]}q_k(s,a,h)\left(P(s'\mid s,a,h)-P_k(s'\mid s,a,h)\right)V^{\pi_k}_{h+1}(s';\ell_k, P)\right|\\
&\quad + \cO\left(B H^3S^3A\left(\ln(HSAK/\delta)\right)^3 \right)
\end{align*}
where the equality is due to \Cref{lemma:confidence''} and the inequality is due to Lemmas \ref{lemma10} and \ref{lemma8}.

Moreover,
\begin{align*}
&\sum_{k=1}^K\left|\sum_{(s,a,s',h)\in\cS\times\cA\times\cS\times[H]}q_k(s,a,h)\left(P(s'\mid s,a,h)-P_k(s'\mid s,a,h)\right)V^{\pi_k}_{h+1}(s';\ell_k,P)\right|\\
&=\sum_{k=1}^K\left|\sum_{(s,a,s',h)\in\cS\times\cA\times\cS\times[H]}q_k(s,a,h)\left(\left(P-P_k\right)(s'\mid s,a,h)\right)\left(V^{\pi_k}_{h+1}(s';\ell_k,P)-\mu_k(s,a,h)\right)\right|\\
&\leq \sum_{k=1}^K\sum_{(s,a,s',h)\in\cS\times\cA\times\cS\times[H]}q_k(s,a,h)\epsilon_k^\star(s'\mid s,a,h)\left|V^{\pi_k}_{h+1}(s';\ell_k,P)-\mu_k(s,a,h)\right|\\
&\leq \cO\left(\sum_{k=1}^K\sum_{\substack{(s,a,s',h)\in\\ \cS\times\cA\times\cS\times[H]}}q_k(s,a,h)\sqrt{\frac{P(s'\mid s,a,h)\ln(HSAK/\delta)}{\max\{1,N_k(s,a,h)\}}\left(V^{\pi_k}_{h+1}(s';\ell_k,P)-\mu_k(s,a,h)\right)^2}\right)\\
&\quad + \cO\left(B HS\sum_{k=1}^K\sum_{(s,a,h)\in\cS\times\cA\times[H]}\frac{q_k(s,a,h)\ln(HSAK/\delta)}{\max\{1,N_k(s,a,h)\}}\right)\\
&\leq \cO\left(\sum_{k=1}^K\sum_{\substack{(s,a,s',h)\in\\ \cS\times\cA\times\cS\times[H]}}q_k(s,a,h)\sqrt{\frac{P(s'\mid s,a,h)\ln(HSAK/\delta)}{\max\{1,N_k(s,a,h)\}}\left(V^{\pi_k}_{h+1}(s';\ell_k,P)-\mu_k(s,a,h)\right)^2}\right)\\
&\quad +  \cO\left(B H^2S^2A\left(\ln(HSAK/\delta)\right)^2 \right)
\end{align*}
where the first equality holds because $\sum_{s'\in\cS}\left(P-P_k\right)(s'\mid s,a,h)=0$ and $\mu_k(s,a,h)$ is independent of $s'$, the first inequality is due to~\Cref{lemma:confidence'},  the second inequality is from $\left|V^{\pi_k}_{h+1}(s';\ell_k,P)-\mu_k(s,a,h)\right|\leq 2B H$, and the last inequality is from \Cref{lemma8}. Recall that
$q_k(s,a,h)=\mathbb{E}\left[n_k(s,a,h)\mid  \pi_k, P\right]$, which implies that
\begin{align*}
\sum_{k=1}^K\sum_{\substack{(s,a,s',h)\in\\ \cS\times\cA\times\cS\times[H]}}q_k(s,a,h)\sqrt{\frac{P(s'\mid s,a,h)\ln(HSAK/\delta)}{\max\{1,N_k(s,a,h)\}}\left(V^{\pi_k}_{h+1}(s';\ell_k,P)-\mu_k(s,a,h)\right)^2}=\sum_{k=1}^K \mathbb{E}\left[X_k\mid \cH_k, P\right]
\end{align*}
where 
$$X_k= \sum_{\substack{(s,a,s',h)\in\\ \cS\times\cA\times\cS\times[H]}}n_k(s,a,h)\sqrt{\frac{P(s'\mid s,a,h)\ln(HSAK/\delta)}{\max\{1,N_k(s,a,h)\}}\left(V^{\pi_k}_{h+1}(s';\ell_k,P)-\mu_k(s,a,h)\right)^2}.$$
Here, we have
$$0\leq X_k \leq \cO\left(B HS\sum_{(s,a,h)\in \cS\times \cA\times [H]}n_k(s,a,h) \sqrt{\ln(HSAK/\delta)}\right)= \cO(B H^2S\sqrt{\ln(HSAK/\delta)}).$$
Then it follows from \Cref{cohen-concentration} that with probability at least $1-\delta$, 
\begin{align*}
&\sum_{k=1}^K \mathbb{E}\left[X_k\mid \cH_k,  P\right]\\
&\leq 2\sum_{k=1}^K \sum_{\substack{(s,a,s',h)\in\\ \cS\times\cA\times\cS\times[H]}}n_k(s,a,h)\sqrt{\frac{P(s'\mid s,a,h)\ln(HSAK/\delta)}{\max\{1,N_k(s,a,h)\}}\left(V^{\pi_k}_{h+1}(s';\ell_k,P)-\mu_k(s,a,h)\right)^2}\\
&\quad + \cO\left(B H^2S \left(\ln(HSAK/\delta)\right)^{1.5}\right).
\end{align*}
Note that
\begin{align*}
&\sum_{k=1}^K \sum_{\substack{(s,a,s',h)\in\\ \cS\times\cA\times\cS\times[H]}}n_k(s,a,h)\sqrt{\frac{P(s'\mid s,a,h)\ln(HSAK/\delta)}{\max\{1,N_k(s,a,h)\}}\left(V^{\pi_k}_{h+1}(s';\ell_k,P)-\mu_k(s,a,h)\right)^2}\\
&\leq \sum_{k=1}^K \sum_{\substack{(s,a,s',h)\in\\ \cS\times\cA\times\cS\times[H]}}n_k(s,a,h)\sqrt{\frac{P(s'\mid s,a,h)\ln(HSAK/\delta)}{\max\{1,N_{k+1}(s,a,h)\}}\left(V^{\pi_k}_{h+1}(s';\ell_k,P)-\mu_k(s,a,h)\right)^2}\\
&\quad + B H\sum_{k=1}^K \sum_{\substack{(s,a,s',h)\in\\ \cS\times\cA\times\cS\times[H]}}n_k(s,a,h)\left(\sqrt{\frac{P(s'\mid s,a,h)\ln(HSAK/\delta)}{\max\{1,N_{k}(s,a,h)\}}}-\sqrt{\frac{P(s'\mid s,a,h)\ln(HSAK/\delta)}{\max\{1,N_{k+1}(s,a,h)\}}}\right)\\
&\leq \sum_{k=1}^K \sum_{\substack{(s,a,s',h)\in\\ \cS\times\cA\times\cS\times[H]}}n_k(s,a,h)\sqrt{\frac{P(s'\mid s,a,h)\ln(HSAK/\delta)}{\max\{1,N_{k+1}(s,a,h)\}}\left(V^{\pi_k}_{h+1}(s';\ell_k,P)-\mu_k(s,a,h)\right)^2}\\
&\quad +B H\sqrt{S}\sum_{k=1}^K \sum_{{(s,a,h)\in \cS\times\cA\times[H]}}\left(\sqrt{\frac{\ln(HSAK/\delta)}{\max\{1,N_{k}(s,a,h)\}}}-\sqrt{\frac{\ln(HSAK/\delta)}{\max\{1,N_{k+1}(s,a,h)\}}}\right)\\
&\leq \sum_{k=1}^K \sum_{\substack{(s,a,s',h)\in\\ \cS\times\cA\times\cS\times[H]}}n_k(s,a,h)\sqrt{\frac{P(s'\mid s,a,h)\ln(HSAK/\delta)}{\max\{1,N_{k+1}(s,a,h)\}}\left(V^{\pi_k}_{h+1}(s';\ell_k,P)-\mu_k(s,a,h)\right)^2}\\
&\quad +  \cO\left(B H^2S^{1.5}A \sqrt{\ln(HSAK/\delta)}\right).
\end{align*}
where the first inequality holds because $\left|V^{\pi_k}_{h+1}(s';\ell_k,P)-\mu_k(s,a,h)\right|\leq B H$, the second inequality holds because $n_k(s,a,h)\leq 1$ and the Cauchy-Schwarz inequality implies that
$$\sum_{s'\in \cS}\sqrt{P(s'\mid s,a,h)}\leq \sqrt{S\sum_{s'\in \cS}P(s'\mid s,a,h)}=\sqrt{S},$$
and the third inequality follows from 
$$\sum_{k=1}^K \left(\sqrt{\frac{1}{\max\{1,N_{k}(s,a,h)\}}}-\sqrt{\frac{1}{\max\{1,N_{k+1}(s,a,h)\}}}\right)\leq \sqrt{\frac{1}{\max\{1,N_{1}(s,a,h)\}}}=1.$$
Next, the Cauchy-Schwarz inequality implies the following. 
\begin{align*}
& \sum_{k=1}^K \sum_{\substack{(s,a,s',h)\in\\ \cS\times\cA\times\cS\times[H]}}n_k(s,a,h)\sqrt{\frac{P(s'\mid s,a,h)\ln(HSAK/\delta)}{\max\{1,N_{k+1}(s,a,h)\}}\left(V^{\pi_k}_{h+1}(s';\ell_k,P)-\mu_k(s,a,h)\right)^2}\\
&\leq  \sqrt{\sum_{k=1}^K \sum_{\substack{(s,a,s',h)\in\\ \cS\times\cA\times\cS\times[H]}}n_k(s,a,h){P(s'\mid s,a,h)}\left(V^{\pi_k}_{h+1}(s';\ell_k,P)-\mu_k(s,a,h)\right)^2}\\
&\quad \times\sqrt{ \sum_{k=1}^K \sum_{\substack{(s,a,s',h)\in\\ \cS\times\cA\times\cS\times[H]}}n_k(s,a,h)\frac{\ln(HSAK/\delta)}{\max\{1,N_{k+1}(s,a,h)\}}}
\end{align*}
Here, the second term can be bounded as follows.
\begin{align*}
\sum_{k=1}^K \sum_{\substack{(s,a,s',h)}}n_k(s,a,h)\frac{\ln(HSAK/\delta)}{\max\{1,N_{k+1}(s,a,h)\}}&=S\ln\left(HSAK/\delta\right)\sum_{k=1}^K \sum_{\substack{(s,a,h)}}\frac{n_k(s,a,h)}{\max\{1,N_{k+1}(s,a,h)\}}\\
&=S\ln\left(HSAK/\delta\right) \sum_{\substack{(s,a,h)}}\sum_{k=1}^K\frac{n_k(s,a,h)}{\max\{1,N_{k+1}(s,a,h)\}}\\
&=\cO\left(HS^2 A\left(\ln\left({HSAK}/{\delta}\right)\right)^2\right).
\end{align*}
For $(s,a,h)\in \cS\times \cA\times [H]$, we define 
$$\mathbb{V}_k(s,a,h)=\var_{s'\sim P(\cdot\mid s,a,h)}\left[V^{\pi_k}_{h+1}(s';\ell_k,P)\right].$$
Then 
\begin{align*}
\mathbb{V}_k(s,a,h)&= \mathbb{E}_{s'\sim P(\cdot\mid s,a,h)}\left[\left(V^{\pi_k}_{h+1}(s';\ell_k,P)-\mu_k(s,a,h)\right)^2\right]\\
&= \sum_{s'\in\cS} P(s'\mid s,a,h)\left(V^{\pi_k}_{h+1}(s';\ell_k,P)-\mu_k(s,a,h)\right)^2
\end{align*}
Furthermore, with probability at least $1-\delta$,
\begin{align*}
&\sum_{k=1}^K \sum_{\substack{(s,a,s',h)\in\\ \cS\times\cA\times\cS\times[H]}}n_k(s,a,h){P(s'\mid s,a,h)}\left(V^{\pi_k}_{h+1}(s';\ell_k,P)-\mu_k(s,a,h)\right)^2\\
&=\sum_{k=1}^K \sum_{\substack{(s,a,h)\in \cS\times\cA\times[H]}}n_k(s,a,h) \mathbb{V}_k(s,a,h)\\
&=\sum_{k=1}^K \langle\bm{q_k},\bm{\mathbb{V}_k}\rangle+\sum_{k=1}^K \sum_{\substack{(s,a,h)\in \cS\times\cA\times[H]}}(n_k(s,a,h)-q_k(s,a,h)) \mathbb{V}_k(s,a,h)\\
&\leq \sum_{k=1}^K\var\left[\langle n_k,\ell_k\rangle\mid \ell_k,\pi_k,P\right]+\cO\left(B^2 H^3\sqrt{K\ln(1/\delta)}\right)
\end{align*}
where $\bm{\mathbb{V}_k}\in\mathbb{R}^{SAH}$ is the vector representation of $\mathbb{V}_k$ and the  inequality follows from \Cref{lemma4}, $ \mathbb{V}_k(s,a,h)\leq B^2 H^2$,
\begin{align*}\sum_{\substack{(s,a,h)\in \cS\times\cA\times[H]}}(n_k(s,a,h)-q_k(s,a,h)) \mathbb{V}_k(s,a,h)&\leq \sum_{\substack{(s,a,h)\in \cS\times\cA\times[H]}}(n_k(s,a,h)+q_k(s,a,h)) B^2 H^2\\
&\leq 2B^2H^3,
\end{align*}
and \Cref{azuma}.
Therefore, we finally have proved that
\begin{align*}
\sum_{k=1}^K \left|\langle \bm{\ell_k}, \bm{q_k}-\bm{\widehat q_k}\rangle\right|
&=\cO\left(\sqrt{HS^2A\left(\ln\frac{HSAK}{\delta}\right)^2\left(\sum_{k=1}^K\var\left[\langle n_k,\ell_k\rangle\mid \ell_k,\pi_k,P\right]+B^2 H^3\sqrt{K\ln\frac{1}{\delta}}\right)}\right)\\
&\quad + \cO\left(B H^3S^3A\left(\ln\frac{HSAK}{\delta}\right)^3 \right).
\end{align*}
Moreover, we know from~\Cref{lemma2} that
$$\var\left[\langle \bm{n_k}, \bm{\ell_k}\rangle\mid  \ell_k, \pi_k, P\right]\leq \mathbb{E}\left[\langle \bm{n_k}, \bm{\ell_k}\rangle^2\mid \ell_k, \pi_k, P\right]\leq 2 B\langle \bm{q_k},\bm{\vec h}\odot \bm{\ell_k}\rangle,$$
and therefore, it follows that
\begin{align*}
\sum_{k=1}^K \left|\langle \bm{\ell_k}, \bm{q_k}-\bm{\widehat q_k}\rangle\right|
&=\cO\left(\left(\sqrt{HS^2A\left(B\sum_{k=1}^K\langle \bm{q_k},\bm{\vec h}\odot \bm{\ell_k}\rangle+B^2 H^3\sqrt{K}\right)}+B H^3S^3A\right)\left(\ln\frac{HSAK}{\delta}\right)^3 \right)\\
&=\cO\left(\left(\sqrt{B^2 H^3S^2AK+ B^2 H^4S^2A\sqrt{K}}+B H^3S^3A\right)\left(\ln\frac{HSAK}{\delta}\right)^3 \right)\\
&=\cO\left(\left(\sqrt{B^2 H^3S^2AK+B^2 H^3S^2AK + B^2 H^5S^2A}+B H^3S^3A\right)\left(\ln\frac{HSAK}{\delta}\right)^3 \right)\\
&=\cO\left(B\left(H^{1.5}S\sqrt{AK} + H^3S^3A\right)\left(\ln\frac{HSAK}{\delta}\right)^3 \right)
\end{align*}
where the second equality holds because $\langle \bm{q_k},\bm{\vec h}\odot \bm{\ell_k} \rangle=\cO(B H^2)$ and the third equality holds because $B^2 H^4S^2A\sqrt{K}=\cO\left(B^2\left( H^3S^2AK + H^5S^2A\right)\right)$.
\end{proof}

\section{Concentration Inequalities}

\begin{lemma}{\rm (Hoeffding's inequality)}\label{hoeffding} 
For i.i.d. random variables $Z_1, \ldots, Z_n$ following $1/2$-sub-Gaussian with zero mean,
\begin{align*}
    &\mathbb{P}\left(\frac{1}{n}\sum_{j=1}^n Z_j \geq \epsilon\right) \leq \exp\left(-n\epsilon^2\right), \\
    &\mathbb{P}\left(\frac{1}{n}\sum_{j=1}^n Z_j \leq -\epsilon\right) \leq \exp\left(-n\epsilon^2\right).
\end{align*}
\end{lemma}

\begin{lemma}{\rm \citep[Theorem 4]{Maurer-bernstein}}\label{bernstein} Let $Z_1,\ldots, Z_n\in[0,1]$ be i.i.d. random variables with mean $z$, and let $\delta>0$. Then with probability at least $1-\delta$,  
$$z- \frac{1}{n}\sum_{j=1}^n Z_j \leq \sqrt{\frac{2 V_n\ln(2/\delta)}{n}} + \frac{7\ln (2/\delta)}{3(n-1)}$$
where $V_n$ is the sample variance given by
$$ V_n=\frac{1}{n(n-1)}\sum_{1\leq j<k\leq n} (Z_j-Z_k)^2.$$
\end{lemma}

Next, we need the following Bernstein-type concentration inequality for martingales due to \cite{beygelzimer11a}. We take the version used in~\cite[Lemma 9]{Jin2020}.
\begin{lemma}{\rm \citep[Theorem 1]{beygelzimer11a}}\label{bernstein2}
Let $Y_1,\ldots, Y_n$ be a martingale difference sequence with respect to a filtration $\cF_1,\ldots, \cF_n$. Assume that $Y_j\leq R$ almost surely for all $j\in [n]$. Then for any $\delta\in(0,1)$ and $\lambda\in(0,1/R]$, with probability at least $1-\delta$, we have
$$\sum_{j=1}^n Y_j\leq \lambda \sum_{j=1}^n\mathbb{E}\left[Y_j^2\mid \cF_j\right]+ \frac{\ln(1/\delta)}{\lambda}.$$
\end{lemma}

\begin{lemma}[Azuma's inequality]\label{azuma}
Let $Y_1,\ldots, Y_n$ be a martingale difference sequence with respect to a filtration $\cF_1,\ldots, \cF_n$. Assume that $|Y_j|\leq B$ for $j\in[n]$. Then with probability at least $1-\delta$, we have
$$\left|\sum_{j=1}^n Y_j\right|\leq B\sqrt{2n\ln(2/\delta)}.$$
\end{lemma}

Next, we need the following concentration inequalities due to \cite{cohen2020}.

\begin{lemma}{\rm \citep[Theorem D.3]{cohen2020}}\label{cohen-concentration0}
	Let $\{X_n\}_{n=1}^\infty$ be a sequence of i.i.d. random variables with expectation $\mu$. Suppose that $0\leq X_n\leq B$ holds almost surely for all $n$. Then with probability at least $1-\delta$, the following holds for all $n\geq 1$ simultaneously:
	\begin{align*}
		\left|\sum_{i=1}^n(X_i-\mu)\right|&\leq 2\sqrt{B\mu n\ln \frac{2n}{\delta}} + B\ln\frac{2n}{\delta},\\
	\left|\sum_{i=1}^n(X_i-\mu)\right|&\leq 2\sqrt{B\sum_{i=1}^n X_i \ln \frac{2n}{\delta}} + 7B\ln\frac{2n}{\delta}.
\end{align*}
\end{lemma}

\begin{lemma}{\rm \citep[Lemma D.4]{cohen2020}}\label{cohen-concentration}
Let $\{X_n\}_{n=1}^\infty$ be a sequence of random variables adapted to the filtration $\{\cF_n\}_{n=1}^\infty$. Suppose that $0\leq X_n\leq B$ holds almost surely for all $n$. Then with probability at least $1-\delta$, the following holds for all $n\geq 1$ simultaneously:
$$\sum_{i=1}^n \mathbb{E}\left[ X_i\mid \cF_i\right]\leq 2 \sum_{i=1}^n X_i + 4B\ln\left(2n/\delta\right).$$
\end{lemma}

\section{Experimental Setup Detailes}\label{sec:appendix:numerical}

We evaluate DOPE+ via the following numerical experiment. We first explain the details of our CMDP setting, which is a modification of the three-state CMDP instances of~\cite{zheng2020constrained, simao2021always, bura2022dope}. We define the state space $\{s_1,s_2,s_3\}$ and the action space $\{a_1,a_2\}$. In \Cref{fig:numerical:transition}, we illustrate the transition probability. For taking $a_1$ at $s_1$, the agent remains in $s_1$ with probability $0.8$, and moves to $s_2$ with probability $0.2$. For taking $a_2$ at $s_1$, the agent moves to $s_2$ with probability $0.8$, and remains in $s_2$ with probability $0.2$. Furthermore, the same transition rule is applied to $s_2$ and $s_3$. 
\begin{figure}[ht]
    \centering
        \begin{tikzpicture}[roundnode/.style={circle, draw=black!60, fill=white!5, thick, minimum size = 10mm}]    
            \node[roundnode] (s1) at (2,2) {\small$s_1$};
            \node[roundnode] (s2) at (4,0) {\small$s_2$};
            \node[roundnode] (s3) at (0,0) {\small$s_3$};
        
            \draw[->] (s1) edge node[right] {$0.2$} (s2);
            \draw[->] (s2) edge node[above] {$0.2$} (s3);
            \draw[->] (s3) edge node[left] {$0.2$} (s1);
        
            \draw[->, loop above] (s1) edge node[above] {$0.8$} (s1);
            \draw[->, loop right] (s2) edge node[right] {$0.8$}(s2);
            \draw[->, loop left] (s3) edge node[left] {$0.8$} (s3);
        \end{tikzpicture} 
    
        \begin{tikzpicture}[roundnode/.style={circle, draw=black!60, fill=white!5, thick, minimum size = 10mm}]    
            \node[roundnode] (s1) at (2,2) {\small$s_1$};
            \node[roundnode] (s2) at (4,0) {\small$s_2$};
            \node[roundnode] (s3) at (0,0) {\small$s_2$};
        
            \draw[->] (s1) edge node[right] {$0.8$} (s2);
            \draw[->] (s2) edge node[above] {$0.8$} (s3);
            \draw[->] (s3) edge node[left] {$0.8$} (s1);
        
            \draw[->, loop above] (s1) edge node[above] {$0.2$} (s1);
            \draw[->, loop right] (s2) edge node[right] {$0.2$}(s2);
            \draw[->, loop left] (s3) edge node[left] {$0.2$} (s3);
        \end{tikzpicture} 
     
    \caption{Transition Probability for Taking $a_1$ and $a_2$ at Each State: Taking $a_1$ (Left) and Taking $a_2$ (Right)}
    \label{fig:numerical:transition}
\end{figure}
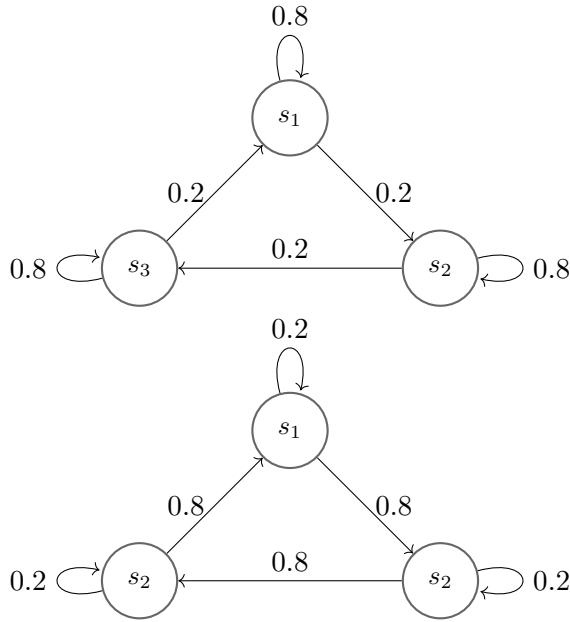

Next, we present the reward function $f$ and the cost function $g$. When the agent takes $a_1$, no reward or cost occurs. Then it can be written as $f(s,a_1)=g(s,a_1)=0$ for $s=s_1,s_2,s_3$. When $a_2$ is taken, the reward occurs depending on the current state. Specifically, we set $f(s_1,a_2)={1}/{3}, \ f(s_2,a_2)={2}/{3}$, and $f(s_3,a_2)=1$. On the other hand, for any state, the same amount of cost is incurred for $a_2$, i.e,  $g(s_1,a_2)=g(s_2,a_2)=g(s_3,a_2)=1$. Hence, $a_2$ is an action with a high reward and a high cost while $a_1$ is an action with zero reward and zero cost. Furthermore, for taking action $a$ at state $s$, the agent can observe the noisy reward $f(s,a)+\zeta_1$ and the noisy cost $g(s,a)+\zeta_2$, where $\zeta_1,\zeta_2$ are independently drawn from a zero-mean $1/2$-sub-Gaussian distribution.

In \Cref{fig:numerical:plot}, we compare regret and constraint violation under DOPE+ and DOPE for $200,000$ episodes when $H=30$. We consider DOPE as a benchmark algorithm because it provides the best-known regret bound among the existing algorithms while ensuring zero hard constraint violation. For the parameters of the experiment, we use $H=30$, $K=200,000$, $\bar C=18$, $\bar C_b=15$, $\delta=0.01$, and the uniform initial distribution of states. To obtain safe baseline policies, we sample a random policy whose expected cost is less than $\bar C_b$. Furthermore, we run the safe baseline policies until the LP becomes feasible for both DOPE+ and DOPE. 
In \Cref{fig:numerical:plot}, to observe the learning process easily, we consider the regret and constraint violations incurred after each LP becomes feasible.
Our results are averaged across 5 runs with different random seeds, and we display the $95\%$ confidence interval with shaded regions. The experiment was conducted on an {Apple M2 Pro}.%

\end{document}